\documentclass[sigconf]{acmart}

\AtBeginDocument{%
  }

\setcopyright{acmlicensed}
\copyrightyear{2018}
\acmYear{2018}
\acmDOI{XXXXXXX.XXXXXXX}
\usepackage[ruled, vlined, linesnumbered]{algorithm2e}
\usepackage{bm}
\usepackage{amsmath}

\DeclareMathOperator*{\argmin}{argmin}

\DeclareFontFamily{U}{mathx}{\hyphenchar\font45}
\DeclareFontShape{U}{mathx}{m}{n}{<-> mathx10}{}
\DeclareSymbolFont{mathx}{U}{mathx}{m}{n}
\usepackage{enumitem}
\usepackage{multirow}
\usepackage{subfigure}
\usepackage{threeparttable}
\newcommand{\tabincell}[2]{\begin{tabular}{@{}#1@{}}#2\end{tabular}}
\definecolor{green}{rgb}{0.0,0.55,0.13}
\acmConference[Conference acronym 'XX]{Make sure to enter the correct
  conference title from your rights confirmation emai}{June 03--05,
  2018}{Woodstock, NY}
\acmISBN{978-1-4503-XXXX-X/18/06}

\begin{document}

\title{Computing Approximate Graph Edit Distance \\ via Optimal Transport}

\author{Qihao Cheng}
\authornote{The first three authors contributed equally to this work.}
\affiliation{%
  \institution{Tsinghua University}
  \country{} 
}
\email{cqh22@mails.tsinghua.edu.cn}

\author{Da Yan}
\authornotemark[1]
\affiliation{%
  \institution{Indiana University Bloomington}
  \country{} 
}
\email{yanda@iu.edu}

\author{Tianhao Wu}
\authornotemark[1]
\affiliation{%
  \institution{Tsinghua University}
  \country{} 
}
\email{wuth20@mails.tsinghua.edu.cn}

\author{Zhongyi Huang}
\affiliation{%
  \institution{Tsinghua University}
  \country{} 
}
\email{zhongyih@tsinghua.edu.cn}

\author{Qin Zhang}
\affiliation{%
  \institution{Indiana University Bloomington}
  \country{} 
}
\email{qzhangcs@iu.edu}


\begin{abstract}
Given a graph pair $(G^1, G^2)$, graph edit distance (GED) is defined as the minimum number of edit operations converting $G^1$ to $G^2$. GED is a fundamental operation widely used in many applications, but its exact computation is NP-hard, so the approximation of GED has gained a lot of attention. Data-driven learning-based methods have been found to provide superior results compared to classical approximate algorithms, but they directly fit the coupling relationship between a pair of vertices from their vertex features. 
We argue that while pairwise vertex features can capture the coupling cost (discrepancy) of a pair of vertices, the vertex coupling matrix should be derived from the vertex-pair cost matrix through a more well-established method that is aware of the global context of the graph pair, such as optimal transport. 
In this paper, we propose an ensemble approach that integrates a supervised learning-based method and an unsupervised method, both based on optimal transport. Our learning method, GEDIOT, is based on inverse optimal transport that leverages a learnable Sinkhorn algorithm to generate the coupling matrix. Our unsupervised method, GEDGW, models  GED computation as a linear combination of optimal transport and its variant, Gromov-Wasserstein discrepancy, for node and edge operations, respectively, which can be solved efficiently without needing the ground truth. 
Our ensemble method, GEDHOT, combines GEDIOT and GEDGW to further boost the performance. Extensive experiments demonstrate that our methods significantly outperform the existing methods in terms of the performance of GED computation, edit path generation, and model generalizability.
\end{abstract}




\received{20 February 2007}
\received[revised]{12 March 2009}
\received[accepted]{5 June 2009}

\maketitle

\begin{figure}[t]
  \centering
  \includegraphics[width=0.8\columnwidth]{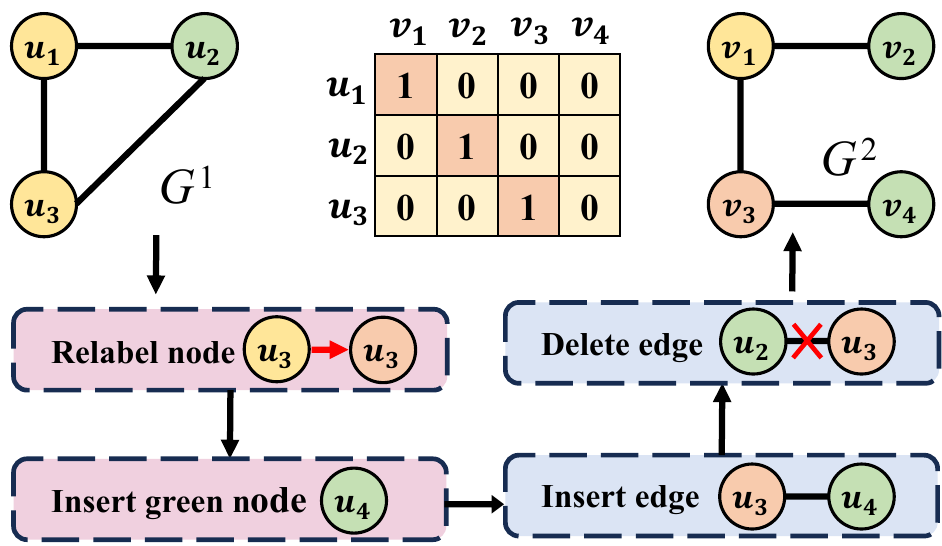}
  \vspace{-2mm}
  \caption{A toy example of graph pair $(G^1,G^2)$}
  \label{fig:matching}
  \vspace{-4mm}
\end{figure}

\begin{figure*}[t]
    \centering
    \includegraphics[width=1\linewidth]{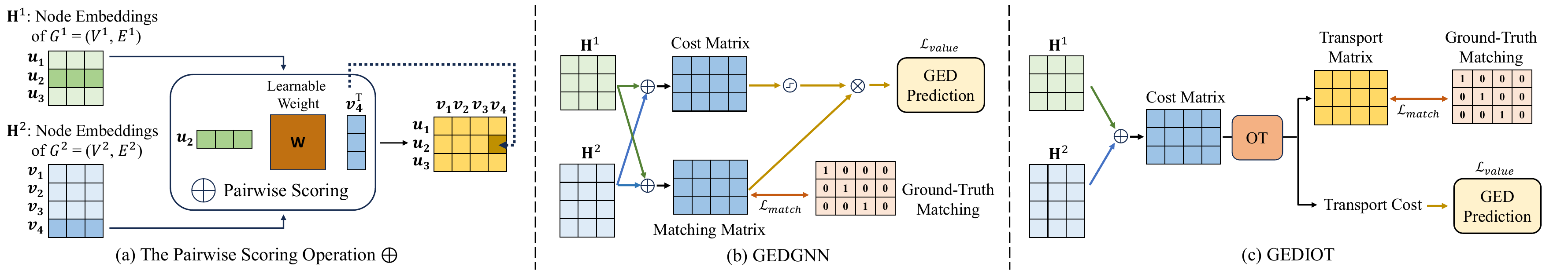}
    \vspace{-7mm}
    \caption{OT Motivation and Learning-based Model Comparison}
    \vspace{-4mm}
    \label{fig:difference}
\end{figure*}

\section{Introduction}
Graph edit distance (GED) is one of the most widely used graph similarity metrics, which is defined as the minimum number of edit operations that transform one graph to the other. {GED has wide applications, such as graph similarity search~\cite{DBLP:conf/icde/LiangZ17,DBLP:conf/icde/WangDTYJ12,zeng2009comparing,DBLP:journals/pvldb/ZhaoXLLZ13,DBLP:conf/cikm/ZhengZLWZ13}, graph classification~\cite{DBLP:conf/sspr/RiesenB08,riesen2009Hungarian}, and inexact graph matching~\cite{bunke1983inexact}. Scenarios include handwriting recognition~\cite{DBLP:conf/gbrpr/FischerSFRB13}, image indexing~\cite{DBLP:journals/imst/XiaoGTL08}, semantic image matching~\cite{wang2021combinatorial}, and investigations of antiviral drugs~\cite{zhao2012efficient}, etc. }
The lower part of Figure~\ref{fig:matching} illustrates the edit path (i.e., sequence of edit operations) of a graph pair $(G^1, G^2)$ with GED $=4$. The edit path with the minimum length is called Graph Edit Path (GEP), so the length of a GEP is exactly the GED. 

Existing methods for GED computation can be categorized into the following three types. {\bf (1) Exact Algorithms.} GED can be computed exactly by the A* algorithm~\cite{riesen2013novel}, but due to being NP-hard~\cite{zeng2009comparing}, it is time-consuming even for a pair of 6-node graphs~\cite{blumenthal2020exact}. {\bf (2) Approximate Algorithms.} To make computation tractable, approximate algorithms are proposed based on discrete optimization or combinatorial search such as A*-Beam~\cite{neuhaus2006A*beam}, Hungarian~\cite{riesen2009Hungarian} and VJ~\cite{fankhauser2011VJ}. {A*-beam restricts the search space of A* algorithm, which is still an exponential-time algorithm. Hungarian and VJ convert the GED computation to a linear sum assignment problem and find the optimal node matching between two graphs, which takes $O(n^3)$ time.} Moreover, these heuristic methods lack a theoretical guarantee and generate results of inferior quality. {\bf (3) Learning-based Methods.} Recent studies turn to data-driven methods based on graph neural networks (GNN) to achieve better performance~\cite{bai2019simgnn,yang2021noah,piao2023gedgnn,bai2021tagsim,ranjan2022greed}. {Differing from the approximate algorithms, learning-based methods extract intra-graph and inter-graph information by generating node and graph embeddings, which are then used to predict GEDs with smaller errors within $O(n^2)$ time in the worst case.} The two most recent works, Noah~\cite{yang2021noah} and GEDGNN~\cite{piao2023gedgnn}, further support generating the edit path based on A*-beam search and $k$-best matching, respectively, to ensure the feasibility of the predicted GED.


However, a key issue remains with these learning-based methods. Specifically, they compute a pairwise vertex discrepancy matrix $\mathbf{A}$ where each element $\mathbf{A}_{i,j}$ corresponds to the coupling cost (discrepancy) of matching vertex~$i$ in $G^1$ to vertex~$j$ in $G^2$, and $\mathbf{A}_{i,j}$ is computed only from their vertex features. {As Figure~\ref{fig:difference}(a) shows, a shared operation of all existing learning-based methods (including our GEDIOT) is pairwise scoring, which given two node embedding matrices obtained from $G^1$ and $G^2$ (via a graph neural network), returns a matrix $\mathbf{A}$ where element $\mathbf{A}_{i,j}$ is the pairwise score computed from the embeddings of node~$u_i$ in $G^1$ and node~$v_j$ in $G^2$. 
Here, we use $\oplus$ to denote the pairwise scoring operation.} Existing learning-based models directly treat $\mathbf{A}$ as the vertex coupling matrix to fit the ground-truth vertex coupling relationship, but we argue that the coupling matrix should be derived from the pairwise discrepancy matrix $\mathbf{A}$ through a more well-established method that is aware of the global context of the graph pair, such as optimal transport~\cite{kolouri2017optimal}. 
{As illustrated by the bottom branch of Figure~\ref{fig:difference}(b) for GEDGNN~\cite{piao2023gedgnn}, they fit $\mathbf{A}$ directly to the 0-1 ground-truth node-matching matrix for GED. Note that the optimal node matching is a global decision: node~$u_i$ in $G^1$ is matched to node~$v_j$ in $G^2$ in the GED solution not only because they have similar labels and neighborhood structures, but also because, for example, node~$u_i$ in $G^1$ is not as similar to the other nodes (e.g., node~$v_k$, $k \neq j$) in $G^2$. However, $\mathbf{A}_{i,j}$ is computed only based on the embeddings of nodes~$u_i$ and $v_j$. }

To fundamentally address this drawback, we propose solutions based on the foundation of the Optimal Transport~(OT) theory. OT is a mathematical framework that focuses on finding the most efficient way to move and transform one distribution of mass into another, which has been successfully applied in various fields~\cite{kolouri2017optimal,xu2021vocabulary,courty2016optimal}. Laid upon rigid mathematical theory~\cite{villani2009optimal,peyre2019computational}, OT provides strong theoretical guarantees and well-understood properties. 
With the development of numerical algorithms, such as the Sinkhorn algorithm~\cite{cuturi2013sinkhorn}, it is particularly effective and efficient when embedding sentences or graph vertices as probabilistic distributions in the Wasserstein space derived from optimal transport~\cite{xu2021vocabulary, DBLP:conf/nips/XuLC19}. 

In this paper, we propose an ensemble approach that integrates a supervised learning-based method and an unsupervised method, both based on OT. Our learning-based method, GEDIOT, is based on inverse optimal transport (IOT)~\cite{stuart2020inverse,chiu2022discrete} that leverages a learnable Sinkhorn algorithm to generate the coupling matrix. {As Figure~\ref{fig:difference}(c) shows, our GEDIOT model takes the cost matrix computed by pairwise scoring, and passes it through an OT module to minimize the cost of transporting masses from nodes of $G^1$ to nodes of $G^2$, which returns the learned transport matrix that considers the global cost matrix when fitting the ground-truth node-matching matrix for GED. As our experiments have shown, adding OT after the pairwise-scoring-induced cost matrix brings significant performance improvement in both GED and GEP predictions. }

Based on optimal transport, we also propose an unsupervised method, GEDGW, that models GED computation as a linear combination of optimal transport and its variant, Gromov-Wasserstein~(GW) discrepancy, for node and edge operations, respectively, which can be solved efficiently without the ground truth. 
Our ensemble method, GEDHOT, combines GEDIOT and GEDGW to further boost the performance. 
Our contributions are listed as follows:
\begin{itemize}[leftmargin=*]
\item We propose a neural network architecture based on inverse optimal transport (where the cost matrix is learnable) that formulates the GED learning task as a bi-level optimization problem, named GEDIOT (GED with IOT), {which introduces the OT component to capture the global context effectively. } 
\item {To make OT applicable, GEDIOT extends the learned cost matrix with a dummy row and utilizes the Sinkhorn algorithm with a learnable regularization coefficient to integrate OT with neural networks for GED computation, improving the model performance and stability.} Since the coupling matrix can represent the confidence of node matching, we can also generate the edit path from it using the $k$-best matching algorithm of~\cite{piao2023gedgnn}.
\item We separate the edit operations into two types: vertex edit operations and edge edit operations. We then model the GED computation as an optimization problem combining optimal transport (for vertex edits) and its variant Gromov-Wasserstein discrepancy (for edge edits), leading to our unsupervised solution named GEDGW (Graph Edit Distance with Gromov-Wasserstein discrepancy). 
\item We combine GEDIOT and GEDGW into an ensemble method named GEDHOT (Graph Edit Distance with Hybrid Optimal Transport) for more accurate GED computation.
\item Extensive experiments show the superior performance of proposed methods. Compared with the state-of-the-art existing method GEDGNN~\cite{piao2023gedgnn}, the Mean Absolute Error~(MAE) on GED computation decreases by $20.5 \%$--$63.8\%$ with GEDIOT. Furthermore, the hybrid method GEDHOT achieves the best performance, where the MAE decreases by $31.2\%$--$72.3\%$ compared with GEDGNN. We also conduct experiments to verify the high-quality edit path generation and superior generalizability of our methods. 
\end{itemize}

The rest of this paper is organized as follows. Section~\ref{sec:related} reviews the related work, and Section~\ref{sec:preliminaries} defines our problem and presents the background of OT. Then, Section~\ref{sec:gediot} describes the proposed learning-based method GEDIOT, and Section~\ref{sec:gedgw} further proposes the unsupervised method GEDGW and the ensemble method GEDHOT, and analyzes the time complexity of our methods. Finally, Section~\ref{sec:experiment} reports our experiments, and Section~\ref{sec:conclusion} concludes this paper.

\vspace{-1mm}
\section{Related Work}
\label{sec:related}

\vspace{1mm}
\noindent\textbf{GED Computation. } Classical exact algorithms~\cite{blumenthal2020exact, chang2020speeding} seek the exact graph edit distance for each graph pair. Due to the NP-hardness of GED computation, they fail to generate solutions in a limited time when the graph size increases. {To make computation tractable, plenty of heuristic algorithms are proposed, including A*-Beam~\cite{neuhaus2006A*beam}, Hungarian~\cite{riesen2009Hungarian} and VJ~\cite{fankhauser2011VJ}, all of which provide an approximate GED in polynomial time. }
Recently, graph neural networks (GNN) have become popular since the extracted node and graph embeddings can greatly help the performance in various tasks~\cite{xiao2022graph,zhou2020graph,zhang2018link,zhang2022graph,zhao2021learned,wang2024neural,li2023coclep}. 
A number of GNN-based methods, such as SimGNN~\cite{bai2019simgnn}, TaGSim~\cite{bai2021tagsim}, Noah~\cite{yang2021noah}, MATA*~\cite{liu2023mata} and GEDGNN~\cite{piao2023gedgnn}, have also been proposed to generate embeddings for GED computation with adequate training data, which achieve the best performance in approximate GED computation. For a more detailed review of heuristic and GNN-based methods, please see Appendix~\ref{app:review}. 

\vspace{1mm} 
{\noindent\textbf{Graph Similarity Search. }Given a query graph and a threshold, graph similarity search retrieves all graphs from a database with GED to the query graph within the given threshold. An important step in this task is to verify whether the GEDs of graph pairs are smaller than the threshold. A series of works~\cite{kim2019inves,kim2021boosting,zhao2018efficient,liang2017similarity,gouda2016csi_ged,chang2022accelerating,chang2020speeding} are proposed to speed up the GED verification process between the database and the query graph. It is related to, but also distinct from, GED computation. They focus on the filtering technique of search space based on the threshold, while GED computation seeks the difference between a pair of graphs and has no threshold available for filtering. However, when setting the similarity threshold to infinity, the verification step can also be extended for GED computation~\cite{chang2020speeding,chang2022accelerating}.} 

\vspace{1mm}
\noindent\textbf{Optimal Transport. } The goal of optimal transport (OT)~\cite{peyre2019computational} is to minimize the cost of transporting mass from one distribution to another. It has been applied in various fields, including image and signal processing~\cite{kolouri2017optimal}, natural language processing~\cite{xu2021vocabulary}, and domain adaptation~\cite{courty2016optimal}. 
Inverse optimal transport (IOT)~\cite{stuart2020inverse,chiu2022discrete} is an inverse process to the classical optimal transport, which calculates the cost matrix from the coupling matrix. 
Recent studies~\cite{shi2023understanding, shiot} interpret classical contrastive learning as inverse optimal transport. DB-OT~\cite{shi2024double} applies inverse optimal transport to long-tailed classification. Legal case matching algorithms are proposed in~\cite{yu2022explainable} via inverse optimal transport. They all use the general inverse optimal transport with the cross-entropy loss to build an OT-assisted neural network model, and the relation between inverse optimal transport and graphs remains rarely studied as it requires careful design for different graph problems. While our proposed GEDIOT model is also based on IOT, as Section~\ref{sec:learnot} will describe, in order for the Sinkhorn algorithm to be applicable to GED prediction, we need to modify the OT constraints by incorporating a dummy supernode.


A few works have also applied OT and its variants to other graph problems (but not GED)~\cite{petric2019got,dong2020copt,wang2022neural,vincent2021semi,chapel2020partial}. 
One of the most important variants is Gromov-Wasserstein discrepancy~(GW)~\cite{DBLP:journals/focm/Memoli11,peyre2016gromov}, a measure used to compare two metric spaces, capturing the differences in their intrinsic geometric structures. GW has been applied for graph partitioning and graph matching~\cite{DBLP:conf/nips/XuLC19}. 
Fused GW~\cite{vayer2020fused} is a combination of optimal transport and GW, which has been successfully applied in graph classification and clustering. 
However, the optimization objective of Fused GW does not consider the edit costs of unmatched vertices in GED computation, but the size of $G^1$ and $G^2$ may not match for a given graph pair $(G^1, G^2)$, so as Section~\ref{sec:gedgw} will describe, our proposed GEDGW model first needs to add dummy nodes to incorporate such costs into the objective.

\vspace{-1mm}
\section{Preliminaries}
\label{sec:preliminaries}
This section introduces Graph Edit Distance~(GED), Graph Edit Path~(GEP), and the fundamental concepts of Optimal Transport~(OT) on graphs. All vectors default to column vectors unless otherwise specified. Table~\ref{Table:notations} summarizes important notations for quick lookup.

\begin{table}[t]
    \begin{center}
    \caption{Notations}
    \vspace{-3mm}
    \label{Table:notations}
    \resizebox{0.84\columnwidth}{!}{
        \begin{tabular}{c|c}
    	\hline
    	\textbf{Notation} & \textbf{Description} \\
    	\hline
            $G$ &a labeled undirected graph\\ \hline
            $V, \ E, \ L$ & the node, edge and label sets of $G$\\ \hline
            $(G^1,G^2)$ & the graph pair for GED computation\\ \hline
            $\mathbf{M}$ & node label matching matrix of $(G^1,G^2)$\\ \hline
            $\mathbf{A}^1, \  \mathbf{A}^2$ & adjacency matrices of $G^1$ and $G^2$  \\ \hline
            $\mathbf{H}^1, \  \mathbf{H}^2$ & final node embeddings of $G^1$ and $G^2$  \\ \hline            
            $GED(G^1,G^2)$ & the GED of graph pair $(G^1,G^2)$\\ \hline  
            $GEP(G^1,G^2)$ & the GEP of graph pair $(G^1,G^2)$\\ \hline
    	$\bm{\pi}$ &the coupling matrix between $G^1$ and $G^2$\\ \hline
            $\bm{\pi}^{*},\ GED^{*}(G^1,G^2)$ & ground truths of the graph pair $(G^1,G^2)$\\ \hline
            $\bm{1}_n,\ \bm{0}_n$ & the $n$-dimensional vectors full of $1$ and $0$\\ \hline
            $\cdot\|\cdot$ & the concatenation operator \\ \hline
            $\cdot \oslash \cdot$ & the element-wise division \\ \hline
            $\left<\mathbf{P}~,\mathbf{Q}\right>$ &the Frobenius dot-product $\sum_{i}\sum_{j}(P_{i,j}Q_{i,j})$\\ \hline
            $\mathcal{L}(\mathbf{C}^1,\mathbf{C}^2)$ &the 4-th order tensor $\left((\mathbf{C}^1_{i,j}-\mathbf{C}^2_{k,l})^2\right)_{i,j,k,l}$ \\ \hline
            $\mathcal{L}\otimes \mathbf{B}$ &the matrix $\left(\sum_{j,l}\mathcal{L}_{i,j,k,l}\mathbf{B}_{j,l}\right)_{i,k}$\\ \hline
        \end{tabular}}
    \end{center}
    \vspace{-3mm}
\end{table}

\subsection{Problem Statement}\label{ssec:def}
We consider two tasks: GED computation and GEP generation between two node-labeled undirected graphs $G^1=(V^1, E^1, L^1)$ and $G^2=(V^2, E^2, L^2)$. {We discuss GED computation of edge-labeled graphs in Appendix~\ref{app:edge-labeled}.} We denote $|V^1|=n_1, |E^1|=m_1$ and $|V^2|=n_2, |E^2|=m_2$. We assume that $n_1\le n_2$ as otherwise, we can swap $G^1$ and $G^2$. 

\vspace{1mm}
\noindent\textbf{Graph Edit Distance (GED).}
    Given the graph pair $(G^1, G^2)$, graph edit distance $GED(G^1, G^2)$ is the minimum number of edit operations that transform $G^1$ to $G^2$. An edit operation is an insertion/deletion of a node/edge or the relabeling of a node.

\vspace{1mm}
\noindent\textbf{Graph Edit Path (GEP).}
    The edit path of the graph pair $(G^1, G^2)$ is a sequence of edit operations that transform $G^1$ to $G^2$. The graph edit path $GEP(G^1, G^2)$ is the shortest one with length  $GED(G^1, G^2)$. 

Figure~\ref{fig:matching} shows a GEP of the graph pair $(G^1, G^2)$, where different colors denote different vertex labels and $GED(G^1,G^2)=4$.

\vspace{1mm}
\noindent {\textbf{Node Matching. }}
The node matching (hereafter we use the terms ``node'' and ``vertex'' interchangeably) of $(G^1,G^2)$ is an $n_1\times n_2$ binary matrix, denoted by  $\bm{\pi}\in\{0,1\}^{n_1\times n_2}$, where $\bm{\pi}_{i,k}=1$ if the node $u_i\in V^1$ matches $v_k\in V^2$, and $\bm{\pi}_{i,k}=0$ otherwise. Since we assume that $n_1\le n_2$, $\bm{\pi}$ satisfies the following constraints
\begin{equation}
\label{eq:constraint}
        \bm{\pi}\bm{1}_{n_2}=\bm{1}_{n_1}, \ \
        \bm{\pi}^\top\bm{1}_{n_1}\le\bm{1}_{n_2}, \ \
        \bm{1}_{n_1}^\top\bm{\pi}\bm{1}_{n_2} = n_1, 
\end{equation}
where $\bm{1}_{n}$ is the $n$-dimensional vector full of $1$, and $\mathbf{a} \le \mathbf{b}$ denotes that $\mathbf{a}_i\leq\mathbf{b}_i$ for the $i^\textsuperscript{th}$ elements of $\mathbf{a}$ and $\mathbf{b}$ for all $i$. 
Intuitively, the constraints ensure that each of the $n_1$ vertices of $G^1$ is matched to a vertex in $G^2$ {since Eq.~\eqref{eq:constraint} allows exactly one $``1"$ in each row and at most one $``1"$ in each column. As illustrated in the 0-1 matrix in Figure~\ref{fig:matching}, nodes $u_1$, $u_2$ and $u_3$ in $G^1$ are matched to $v_1$, $v_2$ and $v_3$ in $G^2$, respectively, and $v_4$ in $G^2$ is not matched.}

With a given node matching between $G^1$ and $G^2$, the edit path can be generated by traversing and comparing the differences between the labels and edges of all matching nodes in $G^1$ and $G^2$. {Specifically, (i)~we first check each node in $G^2$ to see if it is matched and if it has the same label as that of its matched node in $G^1$, if not, we add one to the Edit Distance~(ED). This takes $O(n_2)$ time. In Figure~\ref{fig:matching}, since $v_3$ and $u_3$ have different labels (hints a node relabeling), and $v_4$ (hints a node insertion) is not matched, we add 2 to ED. (ii)~We then check each edge in $G^1$ to see if the corresponding edge (decided by the two matched end-nodes) exists in $G^2$, and vice versa; if not, we add one to ED. This takes $O(m_1+m_2)$ time. In Figure~\ref{fig:matching} edge $(u_2, u_3)$ exists in $G^1$ but the corresponding $(v_2, v_3)$ based on node matching does not exist in $G^2$ (hints an edge deletion), and edge $(v_3, v_4)$ exists in $G^2$ but there is no corresponding edge in $G^1$ (hints an edge insertion), so we add 2 to ED. Overall, the number of edit operations is $4$. Note that the time complexity is linear (i.e., $O(n_2 + m_1 + m_2)$). The pseudo-code is shown in Algorithm~\ref{algo:path} in Appendix~\ref{app:kbest}. }

By relaxing the binary constraints of $\bm{\pi}\in\{0,1\}^{n_1\times n_2}$ to $\bm{\pi}\in[0,1]^{n_1\times n_2}$,  node matching can be connected with the optimal transport theory to be introduced as follows. 

\begin{figure}[t]
    \centering
    \includegraphics[width=1.0\linewidth]{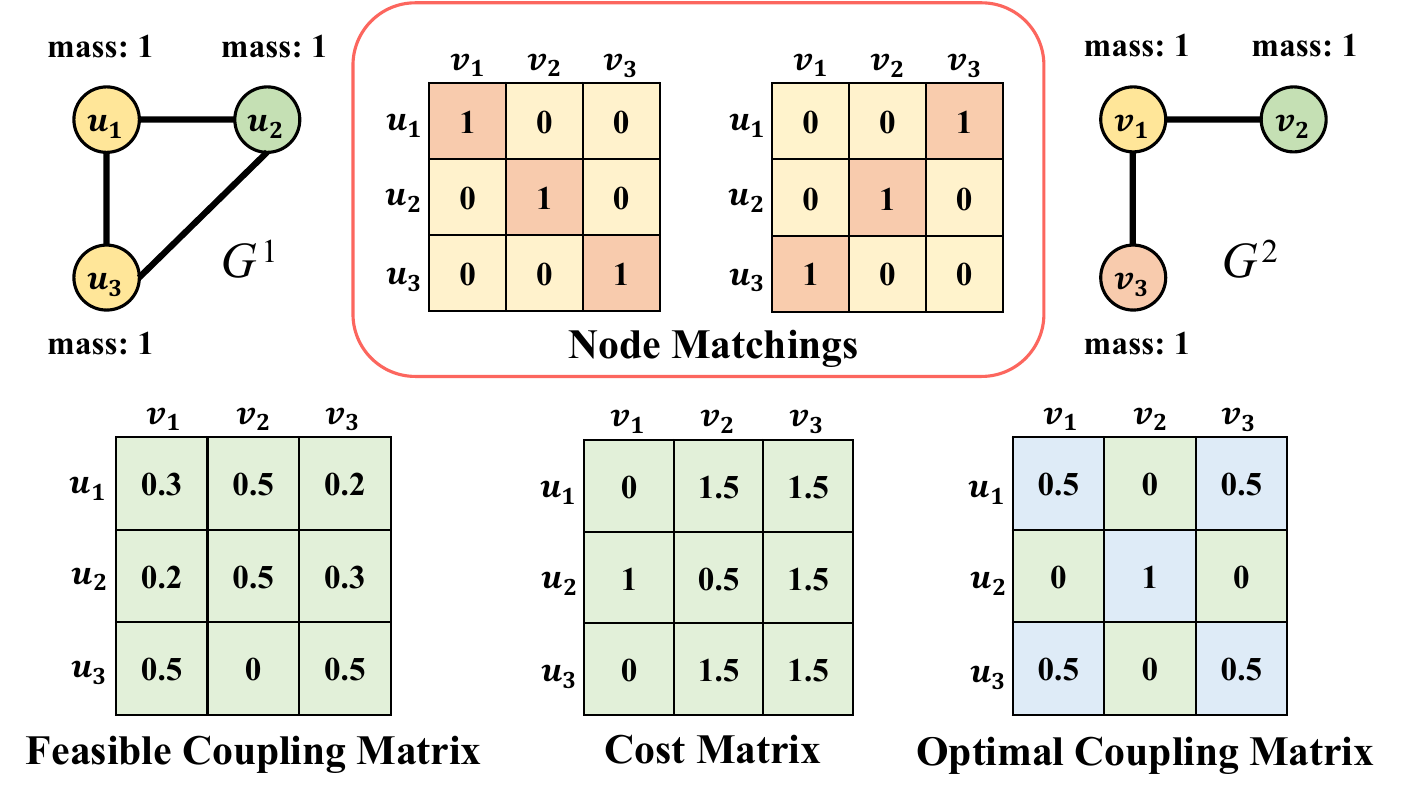}
    \vspace{-8mm}
    \caption{Example of Cost Matrix and Coupling Matrices}
    \vspace{-4mm}
    \label{fig:cost-coupling}
\end{figure}

\subsection{Background of Optimal Transport}

\vspace{1mm}
\noindent\textbf{Optimal Transport (OT). } The optimal transport problem seeks the most efficient way of transporting one distribution of mass into another. Given a graph pair $(G^1,G^2)$, where $G^1=(V^1,E^1,L^1)$ and $G^2=(V^2,E^2,L^2)$, we assume there are two pre-defined mass distributions $\bm\mu=\{\bm{\mu}_i\}_{i=1}^{n_1}$ and $\bm\nu=\{\bm{\nu}_j\}_{j=1}^{n_2}$ on nodes of $G^1$ and $G^2$, respectively. 
{For instance, when $n_1 = n_2$,  for all $u_i$ in $G^1$ and $v_j$ in $G^2$, we can set their masses as $\bm{\mu}_i = 1$ and $\bm{\nu}_j = 1$, which puts the same importance weight on every node. Figure~\ref{fig:cost-coupling} shows our mass distributions on $G^1$ and $G^2$ where every node has mass 1. } 
{Coupling matrix $\bm{\pi}\in\mathbb{R}^{n_1\times n_2}$ is a node-to-node mass transport matrix between $G^1$ and $G^2$, where each element $\bm{\pi}_{i,k}$ denotes the amount of mass transported from node $u_i\in V^1$ to $v_k\in V^2$. In our work, $\bm{\pi}_{i,k}$ is in the range $[0,1]$ reflecting the confidence that $u_i$ matches $v_k$. }
The feasible set of coupling matrices of $(G^1,G^2)$ is denoted by:
$$
\Pi(\bm\mu,\bm\nu)=\left\{\bm{\pi}\in\mathbb{R}^{n_1\times n_2}\,|\, \bm{\pi}\bm{1}_{n_2}=\bm\mu,~\bm{\pi}^\top\bm{1}_{n_1}=\bm\nu,~\bm{\pi}\ge 0\right\}. 
$$
{The lower-left corner of Figure~\ref{fig:cost-coupling} shows an example of a feasible $\bm\pi$. The feasible set of $\bm\pi$ basically relaxes Eq.~\eqref{eq:constraint} to allow values in $[0, 1]$, but still requires that elements in a row (resp.\ column) sum up to 1 (if $n_1\neq n_2$, dummy nodes need to be added as we will describe in Section~\ref{sec:learnot}. So we are basically generalizing Eq.~\eqref{eq:constraint} for the case when $G^1$ and $G^2$ have the same size, where $\bm{\pi}^\top\bm{1}_{n_1}=\bm{1}_{n_2}$). } 

With a given inter-graph node-to-node cost matrix $\mathbf{C}\in\mathbb{R}^{n_1\times n_2}$, where {$\mathbf{C}_{i,j}$ denotes the cost of transporting a unit of mass from $u_i\in V^1$ to $v_j\in V^2$, } OT finds 
{the optimal coupling matrix} 
$\bm{\pi}$ between $G^1$ and $G^2$ as follows:
\begin{equation}
    \label{eq:ot}
    \min_{\bm{\pi}\in\Pi(\bm\mu,\bm\nu)} \left<\mathbf{C},\bm{\pi}\right>,
\end{equation}
where $\left<\mathbf{C},\bm{\pi}\right> = \sum_i \sum_j \mathbf{C}_{i,j} \bm{\pi}_{i,j}$ is the Frobenius dot-product of $\mathbf{C}$ and $\mathbf{\bm{\pi}}$, and the optimal value is the so-called Wasserstein Distance or Earth Mover's Distance. 
{The lower center of Figure~\ref{fig:cost-coupling} shows a simple hand-crafted cost matrix $\mathbf{C}$ between graphs $G^1$ and $G^2$, defined as follows. Initially, we assume matrix $\mathbf{C}$ is all-zero. If the labels of $u_i$ in $G^1$ and $v_j$ in $G^2$ are different, we increase $\mathbf{C}_{i,j}$ by $1$. Let $d_i$ and $d_j$ be the degrees of $u_i$ and $v_j$. We further increase $\mathbf{C}_{i,j}$ by $|d_i - d_j| / 2$, since the difference between degrees is associated with the number of edge insertions/deletions. The constant $1/2$ is used to avoid double-counting of an edge on its two endpoints. Solving OT over this cost matrix gives the coupling matrix shown in the lower-right corner of Figure~\ref{fig:cost-coupling}, which indicates that $u_2$ is mapped to $v_2$, but $u_1$ can be mapped to either $v_1$ or $v_3$ with 50\% probability each. This directly corresponds to the two optimal node matchings illustrated in Figure~\ref{fig:cost-coupling}, which give a GED value of 2. }
A simple yet efficient method to solve Eq.~\eqref{eq:ot} is by introducing an entropy regularization term into the optimization objective~\cite{wilson1969use}:
\begin{equation}
    \label{eq:eot}
    \min_{\bm{\pi}\in\Pi(\bm\mu,\bm\nu)} \left<\mathbf{C},\bm{\pi}\right>+\varepsilon H(\bm\pi),
\end{equation}
where
$H(\bm\pi)= \sum_i \sum_j \bm{\pi}_{i,j} \left( \log\bm{\pi}_{i,j}-1\right) = \left<\bm\pi,\log(\bm\pi)-1\right>$ is the entropy function and $\varepsilon>0$ is the regularization coefficient. Leveraging the duality theory~\cite{boyd2004convex} and strict convexity of Eq.~\eqref{eq:eot}, the unique solution can be solved by the Sinkhorn algorithm as shown in Algorithm~\ref{algo:sinkhorn}~\cite{cuturi2013sinkhorn}, which alternately updates the dual variables $\bm\psi$ and $\bm\varphi$ to fit the specified mass distribution $\bm\mu$ and $\bm\nu$. 
For more details, please see Appendix~\ref{app:sinkhorn}. 

\begin{algorithm}[t]
\DontPrintSemicolon
    \KwIn{cost matrix $\mathbf{C}$, mass distributions $\bm\mu,~\bm\nu$, regularization coefficient $\varepsilon$, maximum iteration $\textit{maxiter}$}
    $\mathbf{K}\leftarrow\exp(-\mathbf{C}/\varepsilon)$, $\bm\varphi\leftarrow\bm{1}_{n_1}$\;\label{alg2:line1}
    \For{$m=1\text{ to }\textit{maxiter}$}{\label{alg2:line2}
        $\bm\psi\leftarrow\bm\nu\oslash(\mathbf{K}^\top\bm\varphi)$\;\label{alg2:line3}
        $\bm\varphi\leftarrow\bm\mu\oslash(\mathbf{K}~\bm\psi)$\;\label{alg2:line4}
    }
    $\bm{\pi}\leftarrow \text{diag}(\bm\varphi)\,\mathbf{K}\,\text{diag}(\bm\psi)$\;\label{alg2:line5}
    $w\leftarrow\left<\mathbf{C},\bm{\pi}\right>$\;\label{alg2:line6}
    \Return $\bm{\pi}$, $w$\;\label{alg2:line7}
\caption{Sinkhorn algorithm}\label{algo:sinkhorn}
\end{algorithm}
\setlength{\textfloatsep}{5pt}

\vspace{1mm}
\noindent\textbf{Gromov-Wasserstein Discrepancy (GW). } 
In practice, it is challenging to define a reasonable node-to-node cost matrix $\mathbf{C} \in \mathbb{R}^{n_1 \times n_2}$ without specified node embeddings for the two graphs $G^1$ and $G^2$. To address this issue, Gromov-Wasserstein discrepancy (GW)~\cite{DBLP:journals/focm/Memoli11, GW-JML} is introduced for graph alignment tasks~\cite{DBLP:conf/icml/VayerCTCF19, DBLP:conf/nips/XuLC19} as an extension of optimal transport. GW only requires the distances between nodes in the same graph, not inter-graph node distances. Specifically, GW is the optimal value of the following optimization objective:
\begin{equation}\label{eq:gw0}
    \min_{\bm{\pi}\in\Pi(\bm\mu,\bm\nu)}\sum_{i,j,k,l}(\mathbf{C}^1_{i,j}-\mathbf{C}^2_{k,l})^2\bm{\pi}_{i,k}\bm{\pi}_{j,l},
\end{equation}
where $\mathbf{C}^1$ and $\mathbf{C}^2$ are the pre-defined cost matrices (e.g., adjacency matrices, all-pair shortest paths) of graphs $G^1$ and $G^2$, respectively. 
Here, we choose the typical option of $(\mathbf{C}^1_{i,j}-\mathbf{C}^2_{k,l})^2$ to measure the mismatch between two edges $(i, j)\in E^1$ and $(k, l)\in E^2$, but more choices can be found in~\cite{peyre2016gromov}. 
Intuitively, $\bm{\pi}_{i,k}$ (resp.\ $\bm{\pi}_{j,l}$) represents the probability of matching nodes $u_i\in V^1$ and $v_k\in V^2$ (resp.\ $u_j\in V^1$ and $v_l\in V^2$), and Eq.~\eqref{eq:gw0} computes the expectation of edge-pair mismatch.


Let $\mathcal{L}(\mathbf{C}^1,\mathbf{C}^2)$ be the 4-th order tensor $\left((\mathbf{C}^1_{i,j}-\mathbf{C}^2_{k,l})^2\right)_{i,j,k,l}$ and $\mathcal{L}\otimes \bm{\pi}$ denotes the matrix $\left(\sum_{j,l}\mathcal{L}_{i,j,k,l}\bm{\pi}_{j,l}\right)_{i,k}$. Then the objective function can be rewritten into the following simple form:  
\begin{equation}
    \label{eq:gw}
    \min_{\bm{\pi}\in\Pi(\bm\mu,\bm\nu)}\left<\bm\pi,\mathcal{L}(\mathbf{C}^1,\mathbf{C}^2)\otimes \bm\pi\right>,
\end{equation}
which can be solved with the conditional gradient algorithm~\cite{vincent2021semi, braun2022conditional}. 


\section{Learning-Based Method: GEDIOT}
\label{sec:gediot}

In this section, we introduce \textbf{GEDIOT}, our neural network for GED computation based on inverse optimal transport. The training is an inverse process of OT to find (i.e., fit) the cost matrix given the ground-truth node coupling matrix of $(G^1, G^2)$, $\bm{\pi}^{*}$, as supervision.

\vspace{1mm}
\noindent\textbf{Motivation of introducing OT.} 
Recall that a node matching satisfies the constraints in Eq.~\eqref{eq:constraint}. Let us denote its feasible set by
\begin{equation}
    \label{eq:U}
U(\bm{1}_{n_1},\bm{1}_{n_2})=\left\{\bm{\pi}\ge0\,|\, \bm{\pi}\bm{1}_{n_2}=\bm{1}_{n_1},~\bm{\pi}^\top\bm{1}_{n_1}\le\bm{1}_{n_2},~\bm{1}_{n_1}^\top\bm{\pi}\bm{1}_{n_2} = n_1\right\}.      
\end{equation}
Previous learning-based models predict GED and node matching via the interaction information of node/graph embeddings~\cite{yang2021noah,bai2019simgnn,bai2021tagsim,piao2023gedgnn}, but they directly fit the predicted node matching with the ground-truth node coupling using binary cross-entropy loss, without considering all the constraints in $U(\bm{1}_{n_1},\bm{1}_{n_2})$ during the training process.

We propose a novel neural architecture, GEDIOT, for GED computation and GEP generation, which predicts only the node-to-node cost matrix $\mathbf{C}$ from the interaction information of node/graph embeddings, and relies on OT to obtain the node matching from $\mathbf{C}$ so that all the constraints in $U(\bm{1}_{n_1},\bm{1}_{n_2})$ are taken into consideration.

The training process is constructed as a bi-level optimization as formulated in Eq.~\eqref{eq:iot-ged}, 
where the inner minimization computes the coupling matrix $\widehat{\bm{\pi}}$ satisfying the constraints in $U(\bm{1}_{n_1},\bm{1}_{n_2})$ by solving an entropy-regularized OT problem that can be evaluated with our learnable Sinkhorn module, and the outer minimization calculates the difference between the coupling matrix and the ground truth to update the cost matrix $\widehat{\mathbf{C}}$ via backpropagation. 

\begin{align}
    \label{eq:iot-ged}
&\min_{\widehat{\mathbf{C}}}~\mathcal{L}_m\left(\bm{\pi}^{*},\widehat{\bm{\pi}}\right)+\mathcal{L}_v\left(\textit{GED}^{*},\widehat{\textit{GED}}\right),\\
    &\text {where }\widehat{\bm{\pi}}=\argmin_{\bm{\pi}\in U(\bm{1}_{n_1},\bm{1}_{n_2})}\left<\widehat{\mathbf{C}},\bm{\pi}\right>+\varepsilon H(\bm{\pi}),\notag \\
&\quad\qquad\widehat{\textit{GED}}=\left<\widehat{\mathbf{C}},~\widehat{\bm{\pi}}\right>.  \notag
\end{align}
Here, $\bm{\pi}^{*}$ and $GED^*$ are the ground truth coupling matrix and GED of graph pair ($G^1$, $G^2$), respectively. Note that $\bm{\pi}^{*} \in \{ 0, 1 \}^{n_1 \times n_2}$ is a one-to-one mapping and there are $(n_2 - n_1)$ full-zero columns. 
During test, computing GED is simply to solve the (inner) entropy-regularized OT problem, which is thus effective and interpretable.



In Eq.~\eqref{eq:iot-ged}, $\widehat{\mathbf{C}}\in\mathbb{R}^{n_1\times n_2}$ is a learnable cost matrix that encodes the cost of matching each vertex pair across $G^1$ and $G^2$, and $\widehat{\bm{\pi}}\in\mathbb{R}^{n_1\times n_2}$ denotes the coupling matrix induced from $\widehat{\mathbf{C}}$ by minimizing the inner optimization problem. 
Recall from Section~\ref{ssec:def} that when relaxing the binary constraints of $\bm{\pi}\in\{0,1\}^{n_1\times n_2}$ to $\bm{\pi}\in[0,1]^{n_1\times n_2}$,  Eq.~\eqref{eq:constraint} basically defines $\bm{\pi}_{i,j}$ to be the probability mass transported from $u_i\in V^1$ to $v_j\in V^2$, and the row $\bm{\pi}_{i}$ defines the distribution of transported probability mass from $u_i$ to vertices of $V^2$. 
In Eq.~\eqref{eq:iot-ged}, the GED value is approximated with $\left<\widehat{\mathbf{C}},~\widehat{\bm{\pi}}\right>=\sum_{i=1}^{n_1}\sum_{j=1}^{n_2}\widehat{\mathbf{C}}_{i,j}\widehat{\bm{\pi}}_{i,j}$ since $\sum_{j=1}^{n_2}\widehat{\mathbf{C}}_{i,j}\widehat{\bm{\pi}}_{i,j}$ is the expected cost to transport mass from $u_i$, so $\left<\widehat{\mathbf{C}},~\widehat{\bm{\pi}}\right>$ is the expected cost to transport all mass from $G^1$. In the special case when $\bm{\pi}\in\{0,1\}^{n_1\times n_2}$, $\left<\widehat{\mathbf{C}},~\widehat{\bm{\pi}}\right>$ basically adds up the costs of transporting mass from $u_i$ to its matched target in $V^2$ for all $u_i\in V^1$; and the first term in the outer minimization encourages a sparse $\widehat{\bm{\pi}}$ since the ground-truth $\bm{\pi}^{*}$ is sparse.

The objective of the outer optimization contains two terms designed for our two tasks: GED computation and GEP generation. Specifically, $\mathcal{L}_m$ is the matching loss for GEP generation, which we use Binary Cross-Entropy~(BCE) loss between the ground truth $\bm{\pi}^{*}$ and the learned coupling matrix $\widehat{\bm{\pi}}$ {that is then fed into the $k$-best matching framework~\cite{chegireddy1987algorithms,piao2023gedgnn} as described in Section~\ref{sec:gep}. }
$\mathcal{L}_v$ is the value loss for GED computation, which we adopt Mean-Squared Error (MSE) between the ground truth GED and the learned one {obtained from both node and graph embeddings}. 
The inner entropy-regularized OT of Eq.~\eqref{eq:iot-ged} provides an optimal coupling matrix with the current cost matrix $\widehat{\mathbf{C}}$. Then, the outer minimization fits the learned coupling matrix and GED to the ground truths to optimize the neural parameters in the model. Notably, we will formulate $\widehat{\mathbf{C}}$ further using the node features extracted from $(G^1, G^2)$ by GNN (see Figure~\ref{fig:difference}), so ``$\min_{\widehat{\mathbf{C}}}$'' in the outer optimization actually optimizes on the parameters of feature extraction network. More analysis of the process of Eq.~\eqref{eq:iot-ged} can be found in Appendix~\ref{app:error}, where we delve into the gap between the learned and ground truth. 

\begin{figure*}
  \centering
  \includegraphics[width=\textwidth]{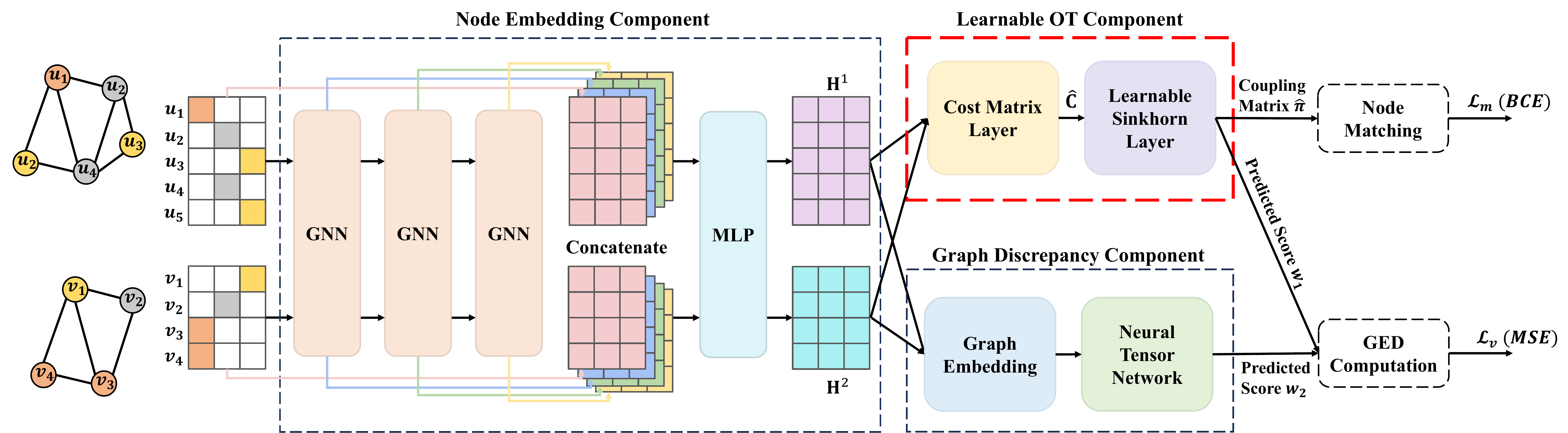}
  \vspace{-2mm}
  \caption{The architecture of GEDIOT}
  \label{fig:architecture}
  \vspace{-2mm}
\end{figure*}

\vspace{1mm}
\noindent\textbf{Model Overview.} Figure~\ref{fig:architecture} illustrates the network architecture of GEDIOT, including three main components: (1)~node embedding component, (2)~learnable OT component, and (3)~graph discrepancy component. {We highlight our new OT module with a red dashed frame in Figure~\ref{fig:architecture}. }
In the node embedding component, 
a GNN is employed to generate node embeddings with multiple graph convolution layers. For both graphs ($G^1$ and $G^2$), the node embeddings outputted by all layers are concatenated to aggregate information from neighbors of different hops (instead of only the last-hop neighbors), to alleviate the GNN over-smoothing issue~\cite{rusch2023survey,qureshi2023limits}. 
The concatenated embeddings are then passed through an MLP to derive the final node embeddings of the desired dimension $d$, {denoted by $\mathbf{H}^1\in\mathbb{R}^{n_1\times d}$ (for $G^1$) and $\mathbf{H}^2\in\mathbb{R}^{n_2\times d}$ (for $G^2$).} Subsequently, the learnable OT component extracts the node-matching matrix from the node embedding matrices through the OT process. This component includes a cost matrix layer that utilizes $\mathbf{H}^1$ and $\mathbf{H}^2$ to measure the node-to-node cost matrix $\widehat{\mathbf{C}}$, and a learnable Sinkhorn layer to read out the learned coupling matrix $\widehat{\bm{\pi}}$ via the Sinkhorn algorithm with a learnable regularization coefficient. This component also provide a GED score $w_1=\left<\widehat{\mathbf{C}},\widehat{\bm{\pi}}\right>$.
Additionally, a graph discrepancy component is employed to measure the edit operations of unmatched nodes (e.g., the $(n_2-n_1)$ nodes in $G^2$) from the graph-to-graph level, which outputs another score $w_2$ for GED prediction. This component includes a neural network to generate the graph embeddings and a neural tensor network (NTN)~\cite{bai2019simgnn} to calculate the predicted score $w_2$. 
Finally, scores $w_1$ and $w_2$ are combined to compute $GED(G^1, G^2)$.
{Note that all the sizes of parameters are user-defined (e.g., embedding dimension $d$) and independent of graph sizes (see more details in Appendix~\ref{app:parameters-sizes}). Figure~\ref{fig:architecture} marks the learnable parts ($\mathbf{H}^1, \mathbf{H}^2, \widehat{\mathbf{C}}, \widehat{\bm{\pi}}$) in GEDIOT for ease of understanding. }

\vspace{-2mm}
\subsection{Node Embedding Component}\label{ssec:NEC}
In this component, a GNN and an MLP are employed to capture the graph topology information and generate the final node embedding. 

\vspace{1mm}
\noindent\textbf{GNN Module.} We adopt a siamese GNN to generate node embeddings by graph convolution operations, following previous graph similarity learning models~\cite{piao2023gedgnn,liu2023mata,ranjan2022greed}. 
Given the graph pair $(G^1, G^2)$, nodes in both $G^1$ and $G^2$ are embedded with the shared network through node feature propagation and aggregation.

Concretely, Graph Isomorphism Network~(GIN)~\cite{xu2018powerful} is adopted to capture the graph topology, since GIN has been shown to be as powerful as the Weisfeiler-Lehman (WL) graph isomorphism test in differentiating different graph structures~\cite{shervashidze2011weisfeiler}. For a graph $G=(V,E,L)$, we initialize the node embedding $\mathbf{h}^{(0)}(u)$ for $u\in V$ as the one-hot encoding of its label. 
If graphs are unlabeled, we set each initial node embedding as a constant number following previous works~\cite{piao2023gedgnn,bai2019simgnn}. 
In the $i\textsuperscript{th}$ layer, the embedding of node $u$, denoted by $\mathbf{h}^{(i)}(u)$, is updated from itself and its neighbors as
\begin{equation}
    \label{eq:app1}
    \mathbf{h}^{(i)}(u)=\text{MLP}\left(\left(1+\delta^{(i)}\right)\mathbf{h}^{(i-1)}(u)+\sum_{v\in\mathcal{N}(u)}\mathbf{h}^{(i-1)}(v)\right)
\end{equation}
where $\delta^{(i)}$ is a learnable parameter of each layer and $\mathcal{N}(u)$ is the set of neighbors of $u$.

\vspace{1mm}
\noindent\textbf{MLP Module.} As the features propagate via GIN, higher-order graph structural information is fused into node embeddings, which may cause over-smoothed node embeddings at the last layer. Note that various GIN layers contain different orders of topological information: $\mathbf{h}^{(0)}(u)$ represents the features of $u$ itself whereas $\mathbf{h}^{(i)}(u)$ contains the feature information from its $i\textsuperscript{th}$-hop neighbors. To obtain sufficiently rich node embeddings for more accurate GED computation, we concatenate the node embeddings from all GIN layers:
$\mathbf{h}=\left[\mathbf{h}^{(0)}\|\mathbf{h}^{(1)}\|\cdots\|\mathbf{h}^{(k)}\right]$.
The concatenated embedding $\mathbf{h}$ is then fed to an MLP to produce the final node embedding $\mathbf{H}\in\mathbb{R}^{n\times d}$:
\begin{equation}
\label{eq:app2}
\mathbf{H}=\text{MLP}\left(\mathbf{h}\right)=\text{MLP}\left(\left[\mathbf{h}^{(0)}\|\mathbf{h}^{(1)}\|\cdots\|\mathbf{h}^{(k)}\right]\right). 
\end{equation}

Suppose that the size of input $\mathbf{h}$ is $n\times D$, then we use an MLP with three dense layers of $D\times2D$, $2D\times D$ and $D\times d$, respectively, to reduce the input $\mathbf{h}$ to the final node embeddings $\mathbf{H}\in\mathbb{R}^{n\times d}$. 

\subsection{Learnable OT Component}\label{sec:learnot}
This component includes a cost matrix layer to extract the cost matrix from node embeddings $\mathbf{H}^1$ and $\mathbf{H}^2$ extracted by the node embedding component introduced in Section~\ref{ssec:NEC}, and a learnable Sinkhorn layer to implement the inner entropy-regularized OT of Eq.~\eqref{eq:iot-ged} to generate the node matching from the cost matrix.

\vspace{1mm}
\noindent\textbf{Cost Matrix Layer.} This layer measures the node-to-node cost matrix $\widehat{\mathbf{C}}\in\mathbb{R}^{n_1\times n_2}$ for the graph pair $(G^1,G^2)$,
by multiplying the final node embeddings $\mathbf{H}^1$, $\mathbf{H}^2$ with a trainable parameter matrix: 
\begin{equation*}
    \widehat{\mathbf{C}}= f\left(\mathbf{H}^1\mathbf{W}(\mathbf{H}^2)^\top\right),
\end{equation*}
where $\widehat{\mathbf{C}}_{i,j}=f(\mathbf{H}^1_i\mathbf{W}(\mathbf{H}^2_{j})^T)=\sum_{k=1}^d\sum_{l=1}^df(\mathbf{H}^1_{i,k}\mathbf{W}_{k,l}\mathbf{H}^2_{l,j})$, $\mathbf{W}\in\mathbb{R}^{d\times d}$ is a learnable interaction matrix, and $f$ is an element-wise activation function. $\mathbf{W}_{k,l}$ can be regarded as a correlation weight for the $k\textsuperscript{th}$ dimension in embedding $\mathbf{H}^1$ and the $l\textsuperscript{th}$ dimension in embedding $\mathbf{H}^2$. 
In this work, we use tanh as the activation function:
\begin{equation}
    \label{eq:app3}
    \widehat{\mathbf{C}}=\tanh\left(\mathbf{H}^1\mathbf{W}(\mathbf{H}^2)^\top\right).
\end{equation}

\vspace{1mm}
\noindent\textbf{Learnable Sinkhorn Layer.} This layer is designed to solve the entropy-regularized OT numerically with the Sinkhorn algorithm in Algorithm~\ref{algo:sinkhorn}.
It takes the learned cost matrix $\mathbf{C}$ as input to generate the coupling matrix $\widehat{\bm{\pi}}$ and the predicted score $w_1$.

{Recall that the core process of the Sinkhorn algorithm is the alternate update of dual variables as shown in Lines~\ref{alg2:line3} and~\ref{alg2:line4} in Algorithm~\ref{algo:sinkhorn}:
\begin{equation*}
    \bm\psi\leftarrow\bm\nu\oslash(\mathbf{K}^\top\bm\varphi), \ \ \bm\varphi\leftarrow\bm\mu\oslash(\mathbf{K}~\bm\psi),
\end{equation*}
where $\mathbf{K} = \exp(-\mathbf{C}/\varepsilon)$ is related to the learned cost matrix $\mathbf{C}$ and regularization coefficient $\varepsilon$, $\bm\varphi$ and $\bm\psi$ are the dual variables, and $\bm\mu$ and $\bm\nu$ are the pre-defined mass distributions (e.g., all-1 vectors).
} However, the constraint set $U(\bm{1}_{n_1},\bm{1}_{n_2})$ in Eq.~\eqref{eq:U} has an inequality constraint $\bm{\pi}^\top\bm{1}_{n_1}\le \bm{1}_{n_2}$, 
which hinders applying the Sinkhorn algorithm directly, since the derivation of Sinkhorn as detailed in Appendix~\ref{app:sinkhorn} only allows equality constraints (with inequality constraints, the dual formulation would introduce additional conditions that require the Lagrangian multipliers to be non-negative for $\bm{\pi}^\top\bm{1}_{n_1}\le \bm{1}_{n_2}$). 
{To address this issue, we reconstruct an equivalent standard-form OT without the inequality constraint by extending the cost matrix $\mathbf{C}$ with a dummy row filled with $0$ and redefining mass distributions as $\widetilde{\bm\mu}, \widetilde{\bm\nu}$ as follows:
\begin{equation*}
            \widetilde{\mathbf{C}}=\begin{bmatrix}
         \widehat{\mathbf{C}} \\
         \bm{0}_{n_2}^\top
        \end{bmatrix},\ \ \ 
         \widetilde{\bm\mu}=[\bm{1}^\top_{n_1},\, n_2-n_1]^\top,\ \ \widetilde{\bm\nu}=\bm{1}_{n_2}. 
\end{equation*}

Accordingly, we denote the new constraint set by
$$\Pi(\widetilde{\bm\mu},\widetilde{\bm\nu}) = \left\{ \bm{\pi}\in\mathbb{R}^{(n_1+1)\times n_2} \ | \ \bm{\pi} \bm{1}_{n_2} = \widetilde{\bm{\mu}},\ \ \bm{\pi}^\top \bm{1}_{n_1+1} = \widetilde{\bm{\nu}},\ \ \bm{\pi}\ge0\right\},$$
and the standard-form OT is formulated as follows:
\begin{equation}
    \label{eq:temp_ot}
    \min_{\bm{\pi}\in\Pi(\widetilde{\bm\mu},\widetilde{\bm\nu})} \left<\widetilde{\mathbf{C}},\bm{\pi}\right>.
\end{equation}
}

\begin{figure}[t]
  \centering
  \includegraphics[width=0.45\columnwidth]{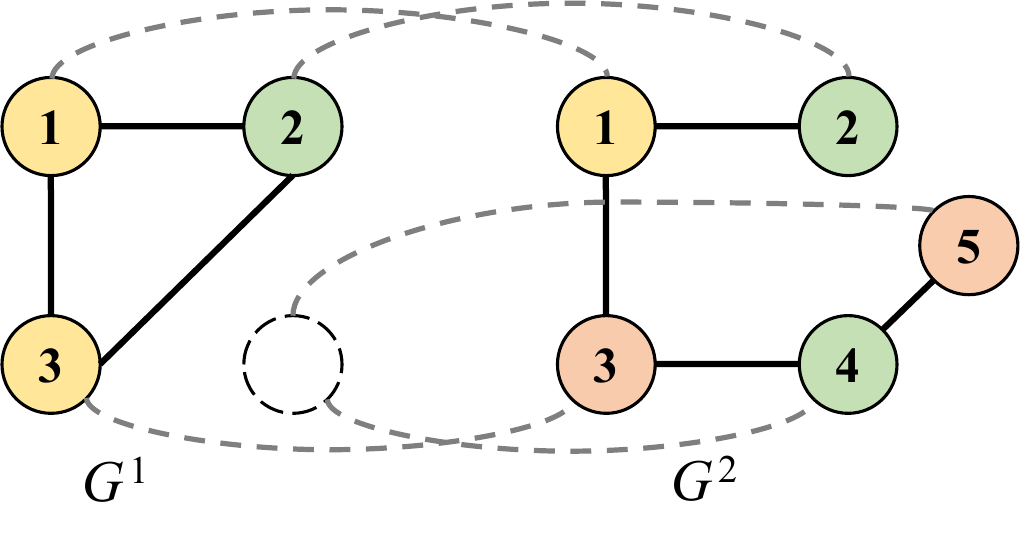}
  \vspace{-2mm}
  \caption{Illustration of the Dummy Supernode}
  \label{fig:dummy_GEDIOT}
\end{figure}
 
{Intuitively, the dummy row in $\widetilde{\mathbf{C}}$ basically adds a dummy supernode in $G^1$ to match $(n_2-n_1)$ nodes in $G^2$, as Figure~\ref{fig:dummy_GEDIOT} illustrates. We set the cost of matching the dummy supernode as 0 since our learnable OT component only accounts for the edit operations related to node matching (i.e., matching each of the $n_1$ node in $G^1$ towards $G^2$); while the edit cost induced by these $(n_2-n_1)$ nodes in $G^2$ will be handled by the graph discrepancy component (see Section~\ref{sec:ntn}).}   

By adding entropy regularization $\varepsilon H(\bm{\pi})$ to Eq.~\eqref{eq:temp_ot}, we can solve for $\bm{\pi}$ by the Sinkhorn algorithm. In each iteration, we update the dual variables $\widetilde{\bm\psi} \in \mathbb{R}^{n_2}$ and $\widetilde{\bm\varphi} \in \mathbb{R}^{n_1+1}$ alternately via:
\begin{equation}\label{eq:epsSinkhorn}    \widetilde{\bm\psi}\leftarrow\widetilde{\bm{\nu}}\oslash\left(\widetilde{\mathbf{K}}^\top\widetilde{\bm\varphi}\right), \ \ \widetilde{\bm\varphi}\leftarrow\widetilde{\bm{\mu}}\oslash\left(\widetilde{\mathbf{K}}\,\widetilde{\bm\psi}\right),
\end{equation}
where $\widetilde{\mathbf{K}}\leftarrow\exp(-\widetilde{\mathbf{C}}/\varepsilon)$ is the element-wise exponential of $-\widetilde{\mathbf{C}}/\varepsilon$. We stack the two operations as feedforward layers to implement the iterations. 
When the iterative updates converge, $\widetilde{\bm{\pi}}=\text{diag}(\widetilde{\bm\varphi})\,\widetilde{\mathbf{K}}\,\text{diag}(\widetilde{\bm\psi})$, and the learned coupling matrix $\widehat{\bm\pi}$ is exactly $\widetilde{\bm{\pi}}$ with the last row removed~\cite{chapel2020partial}. The predicted score $w_1$ is $\left<\widehat{\mathbf{C}},\widehat{\bm{\pi}}\right>$, which estimates the optimal cost of edit operations induced by node matching.

A question remains: how to set a proper regularization coefficient $\varepsilon$? 
While a smaller $\varepsilon$ leads to a closer approximation of the exact OT solution (without regularization). However, it also introduces a greater risk of numerical instability, which may lead to a divide-by-zero error. A straightforward approach is to set different $\varepsilon$ for different datasets manually to achieve satisfactory performance. Nevertheless, the selection of an appropriate $\varepsilon$ is costly.

Rather than fixing $\varepsilon$ in advance for different datasets, we treat it as a learnable parameter and optimize it by gradient descent during training. The regularization coefficient $\varepsilon$ is tuned for different datasets adaptively towards the optimal value, avoiding time-consuming manual adjustments. This is where the term ``learnable'' in the layer name originated (as Eq.~\eqref{eq:epsSinkhorn} is parameter-free).

{We also provide a concrete example from real-world datasets in  Appendix~\ref{app:case-study} to further illustrate our method. }

\vspace{-1mm}
\subsection{Graph Discrepancy Component}\label{sec:ntn}
{Recall that before the learnable Sinkhorn layer, we add a dummy supernode to $G^1$; when the layer completes and outputs $\widehat{\mathbf{\pi}}$, we remove the last row that corresponds to the dummy supernode. The learnable OT component captures only the edit operations induced by the node matching (from the node-to-node level), and some edit operations are not accounted for since $n_1\le n_2$. }
{We thus adopt another graph discrepancy component to supplement the unencoded information from the embedding of unmatched $(n_2-n_1)$ nodes in $G^2$ from the graph-to-graph level.} It includes a graph embedding layer to learn the embeddings of $G^1$ and $G^2$, and a neural tensor network (NTN)~\cite{bai2019simgnn} that reads out the graph discrepancy information from the graph embeddings to enhance GED prediction. 

Specifically, we first generate the graph-level embeddings with the node attentive mechanism~\cite{bai2019simgnn}. 
Given a graph $G$ (can be either $G^1$ or $G^2$) with node embedding matrix $\mathbf{H}\in\mathbb{R}^{n\times d}$ (can be either $\mathbf{H}^1$ or $\mathbf{H}^2$) extracted by our node embedding component, we first calculate the global graph context vector
\begin{equation}
\label{eq:app4}
\mathbf{h}_c=\tanh\left(\mathbf{W}_1\left(\frac{1}{n}\left(\sum_{i=1}^n\mathbf{H}_i\right)^\top\right)\right),
\end{equation}
which averages node features for all nodes of $G$ followed by a non-linear transformation, where $\mathbf{W}_1\in\mathbb{R}^{d\times d}$ is a learnable weight matrix and $\mathbf{H}_i$ is the $i\textsuperscript{th}$ row of $\mathbf{H}\in\mathbb{R}^{n\times d}$. Then, the attention weight of each node $v_i$ is computed as the inner product between $\mathbf{h}_c$ and $\mathbf{H}_i$ and normalized to the range $(0, 1)$, giving the node weight vector: $\mathbf{a}=\sigma(\mathbf{H}\mathbf{h}_c)\in\mathbb{R}^{n}$, where $\sigma$ is the sigmoid function. Finally, the graph embedding $\mathbf{h}_G\in\mathbb{R}^d$ is computed as the weighted sum of node embeddings: $\mathbf{h}_G=\sum_{i=1}^n\mathbf{a}_i\mathbf{H}_i$.

Now that we have obtained graph embeddings for $G^1$ and $G^2$, we use an NTN to calculate the graph-to-graph interaction vector $\mathbf{s}(G^1,G^2)\in\mathbb{R}^L$ where $L$ denotes the output dimension of NTN.
\begin{equation}
\label{eq:app5}
    \mathbf{s}(G^1,G^2)=\text{ReLU}\left(\mathbf{h}_{G^1}^\top\mathbf{W}_2^{[1:L]}\mathbf{h}_{G^2}+\mathbf{W}_3
        [\mathbf{h}_{G^1}^\top\|
        \mathbf{h}_{G^2}^\top]^\top+\mathbf{b}\right),
\end{equation}
where $\mathbf{W}_2^{[1:L]}\in\mathbb{R}^{L\times d\times d}$, $\mathbf{W}_3\in\mathbb{R}^{L\times 2d}$ and $\mathbf{b}\in\mathbb{R}^L$ are learnable, and $\mathbf{h}_{G^1}^\top\mathbf{W}_2^{[1:L]}\mathbf{h}_{G^2}$ denotes the following $L$-dimensional vector: 
 \begin{equation*}
\left[\mathbf{h}_{G^1}^\top\mathbf{W}_2^{(1)}\mathbf{h}_{G^2},\ \ \ \mathbf{h}_{G^1}^\top\mathbf{W}_2^{(2)}\mathbf{h}_{G^2},\ \ \ \ldots,\ \ \ \mathbf{h}_{G^1}^\top\mathbf{W}_2^{(L)}\mathbf{h}_{G^2}\right]^\top,
 \end{equation*}
 where $\mathbf{W}_2^{(i)}$ is the $i\textsuperscript{th}$ learnable weight matrix of $\mathbf{W}_2^{[1:L]}$.

Finally, we apply an MLP to progressively reduce the dimension of $\mathbf{s}(G^1,G^2)$ to a scalar, which outputs the predicted score $w_2$ to measure the edit operations of the unmatched nodes.

\vspace{-1mm}
\subsection{Model Training}\label{sec:train}
GEDIOT is supervised by the ground-truth $GED^{*}(G^1,G^2)$ and the corresponding coupling matrix $\bm{\pi}^{*}$ for node matching between two graphs $G^1$ and $G^2$ during the training process. As shown in Eq.~\eqref{eq:iot-ged}, the loss function consists of two parts: a value loss $\mathcal{L}_v$ to predict the GED and a matching loss $\mathcal{L}_m$ to predict the coupling matrix. The final loss function of GEDIOT is defined as
\begin{equation}
\label{eq:loss}
    \mathcal{L}=\lambda\mathcal{L}_v+(1-\lambda)\mathcal{L}_m,
\end{equation}
where we use a hyperparameter $\lambda$ to balance $\mathcal{L}_v$ and $\mathcal{L}_m$.

Since the range of $GED(G^1, G^2)$ is too large to train a neural network effectively, we normalize the ground-truth GED to the range $[0, 1]$, and the normalized ground-truth GED is given by:
\begin{equation*}
    \text{n}GED^{*}(G^1,G^2)=\dfrac{GED^{*}(G^1,G^2)}{\max(n_1,n_2)+\max(m_1,m_2)}, 
\end{equation*}
where the denominator on the right is the maximum number of edit operations that modify all nodes and edges to 
transform $G^1$ to $G^2$. 
To predict this normalized GED, we define the function:
\begin{equation*}
    \text{score}(G^1,G^2) = \sigma(w_1+w_2),
\end{equation*}
where $ w_1 = \left<\widehat{\mathbf{C}},\widehat{\bm{\pi}}\right>$ is the predicted score from the learnable OT component, and $w_2$ is the predicted score from NTN~\cite{bai2019simgnn}. Here, $\sigma$ is the sigmoid function to ensure that the prediction is within $(0,1)$.

We use MSE as the loss function for value:
\begin{equation*}
    \mathcal{L}_v =\left(\text{score}(G^1,G^2)-\text{n}GED^{*}(G^1,G^2)\right)^2,
\end{equation*}
and we fit the predicted coupling matrix with the ground-truth 0-1 matrix $\bm{\pi}^{*}$, by minimizing the binary cross-entropy loss (BCE) between the learned coupling matrix $\widehat{\bm{\pi}}$ and ground truth $\bm{\pi}^{*}$ : 
\begin{equation*}
    \begin{aligned}
       \mathcal{L}_m =  \dfrac{1}{n_1n_2}BCE\left(\bm{\pi}^{*}|\widehat{\bm{\pi}}\right), 
    \end{aligned}
\end{equation*}
\vspace{-2mm}
where 
\begin{equation*}
\begin{aligned}
    BCE\left(\bm{\pi}^{*}|\widehat{\bm{\pi}}\right)&= \sum_{i=1}^{n_1} \sum_{j=1}^{n_2} \bm{\pi}^{*}_{i,j} \log \widehat{\bm{\pi}}_{i,j} + \left(1-\bm{\pi}^{*}_{i,j}\right)\log\left(1-\widehat{\bm{\pi}}_{i,j}\right)\\
    &=\left<\bm{\pi}^{*},\ \log(\widehat{\bm{\pi}})\right>+\left<1-\bm{\pi}^{*},\ \log(1-\widehat{\bm{\pi}})\right> . 
\end{aligned}
\end{equation*}

\subsection{GEP Generation}\label{sec:gep}
Although we fit $\widehat{\bm{\pi}}$ to the ground-truth node matching $\bm{\pi}^{*}\in\{0,1\}^{n_1\times n_2}$, in practice when the model is trained, the learned coupling matrix $\widehat{\bm{\pi}}$ outputted by GEDIOT is not perfect but in the range $\bm{\pi}^{*}\in[0,1]^{n_1\times n_2}$ representing the confidence of node-to-node matching. 

\begin{figure}
  \centering
  \includegraphics[width=0.98\columnwidth]{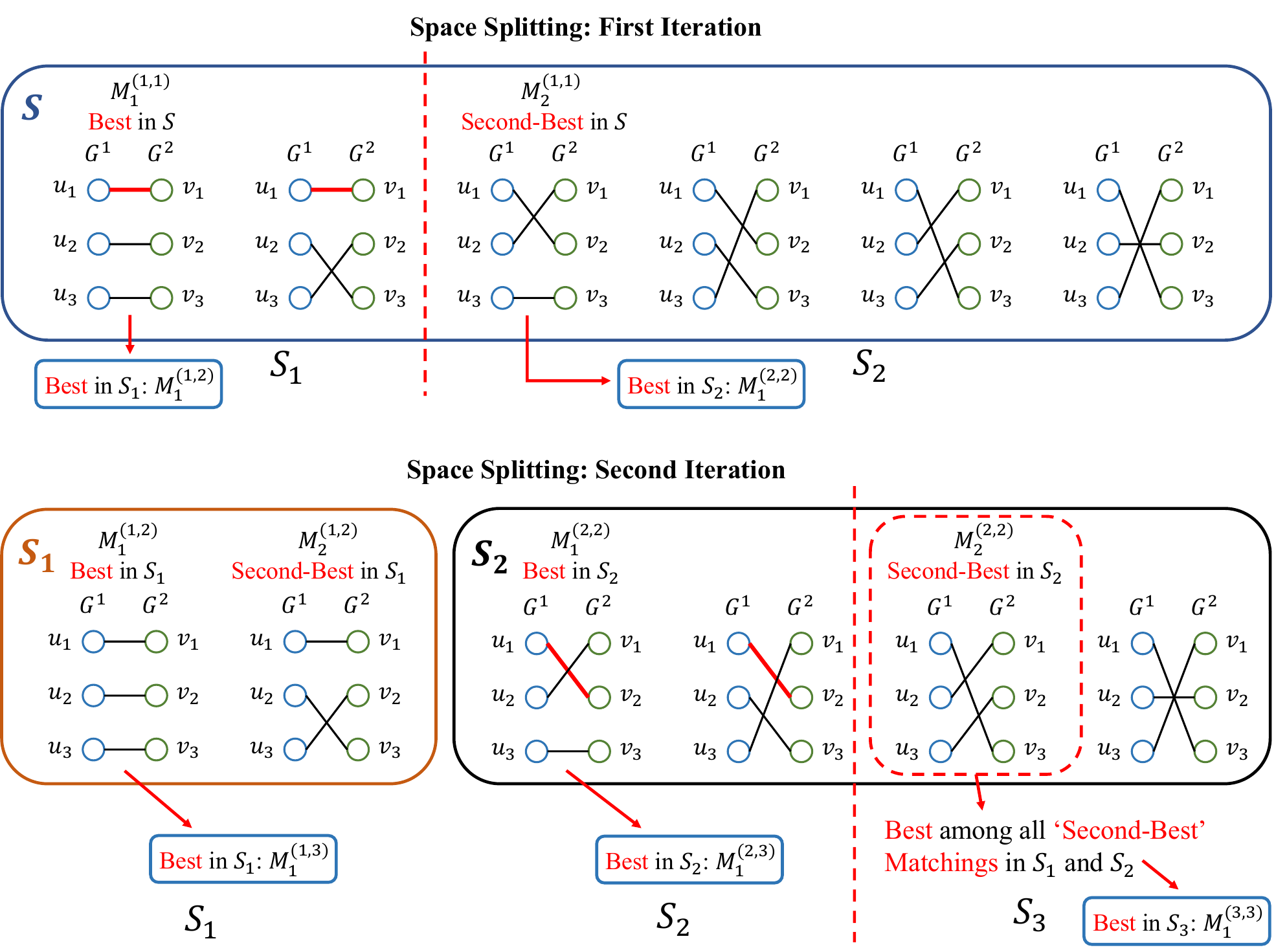}
  \vspace{-3mm}
  \caption{Example of Space Splitting of $k$-Best Matching}
  \label{fig:space-splitting}
\end{figure}

During inference, we adopt the $k$-best matching framework of~\cite{piao2023gedgnn} to generate $\widehat{GEP}(G^1, G^2)$ from the learned coupling matrix $\widehat{\bm{\pi}}$, which utilizes the solution space splitting method~\cite{chegireddy1987algorithms} to obtain a candidate set of $k$-best bipartite node matchings (based on the matching cost specified by the learned coupling matrix $\widehat{\bm{\pi}}$) and searches for the one with the shortest edit path as $\widehat{GEP}(G^1, G^2)$. 
Specifically, let $S$ be a set of node matchings {(Figure~\ref{fig:space-splitting} shows two graphs $G^1$ and $G^2$ each having $3$ nodes and all $6$ possible node matchings in $S$), in which we can find the best and second-best node matchings according to the matching cost from $\widehat{\bm{\pi}}$, denoted by $M_1^{(1,1)}$ and $M_2^{(1,1)}$, respectively, in $O(n^3)$ time~\cite{chegireddy1987algorithms}. 
The first (resp.\ second) ``1'' in the superscript $(1,1)$ means that the two matchings are in the first partition (resp.\ obtained in the first iteration). } 
Let $(u, v)$ be a node pair in $M_1^{(1,1)}$ but not in $M_2^{(1,1)}$ where $u \in V^1$ and $v \in V^2$. We can split $S$ into two subspaces $S_1$ and $S_2$, such that a node matching of $S$ is in $S_1$ if it contains $(u, v)$, and otherwise it is in $S_2$. {As shown in the upper part of  Figure~\ref{fig:space-splitting}, $u_1$ matches $v_1$ in the best matching in $S$, but $u_1$ does not match $v_1$ in the second-best matching in $S$. Then, we split $S$ into $S_1$ and $S_2$ according to whether $u_1$ matches $v_1$ in the first iteration.} Note that $M_1^{(1,1)}$ (resp.\ $M_2^{(1,1)}$) becomes the best node matching in $S_1$ (resp.\ $S_2$) after splitting, which we denote as $M_1^{(1,2)}$ (resp. $M_1^{(2,2)}$). We also search the new second-best node matchings in $S_1$ and $S_2$, denoted by $M_2^{(1,2)}$ and $M_2^{(2,2)}$, respectively. The entire node matching space is partitioned by repeatedly selecting a partition to split in this manner. Assuming that there are $t$ partitions and each has its best and second-best node matching $M^{(r,t)}_1$ and $M^{(r,t)}_2$, where $r=1,2,..,t$, the $(t+1)\textsuperscript{th}$ best node matching is $M^{(t^*,t)}_2$ of the partition $t^*$ with the best `second-best' node matching, so partition $t^*$ is selected for splitting. {Consider the lower part of Figure~\ref{fig:space-splitting}, where we assume the second-best matching $M^{(2,2)}_2$ in $S_2$ is better than the second-best matching $M^{(1,2)}_2$ in $S_1$. Since the best and second-best matchings in $S_2$ differ based on whether $u_1$ is matched to $v_2$, we further split $S_2$ accordingly. After splitting, the second-best matching $M^{(2,2)}_2$ in the original $S_2$ becomes the best matching in $S_3$, which we denote as $M^{(3,3)}_1$.} This process is repeated until $k$ partitions are reached, and GED lower-bound-based pruning~\cite{chang2020speeding,gouda2015improved} is integrated to prune the unfruitful branches. Finally, $2k$ node matchings (2 from each partition) are collected as the candidate set to find the shortest edit path. {More details can be found in Appendix~\ref{app:kbest}.}

\vspace{-2mm}
\section{Unsupervised Method: GEDGW}
\label{sec:gedgw}
Currently, learning-based methods~\cite{bai2019simgnn,yang2021noah,piao2023gedgnn,bai2021tagsim} show the best performance of approximate GED computation, but they need ground truth for training set. This section presents our unsupervised optimization approach, \textbf{GEDGW}, that is able to achieve performance comparable to learning-based methods. GEDGW is based on the Gromov-Wasserstein discrepancy, which bridges GED computation and optimal transport from an optimization perspective. 

\begin{figure}[t]
  \centering
  \includegraphics[width=0.45\columnwidth]{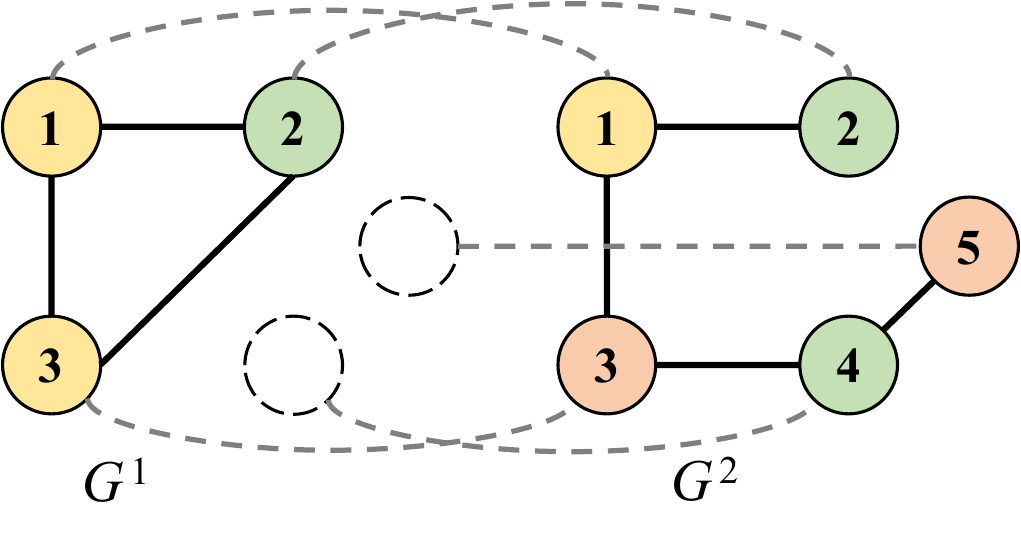}
  \vspace{-2mm}
  \caption{Illustration of Adding Dummy Nodes in $G^1$}
  \label{fig:GEDGW_dummy}
\end{figure}

\vspace{-1mm}
\subsection{Formulation of GEDGW}
\label{sec:eq_for_gedgw}
Recall that the total edit operations that transform $G^1$ to $G^2$ can be determined with a given node matching between $G^1$ and $G^2$, {where GED is the smallest one.} 
Consequently, the GED computation of the graph pair $(G^1, G^2)$ can be formulated as an optimization problem {related to} node matching.

Since there can be $(n_2-n_1)$ nodes in $G^2$ that do not match any nodes in $G^1$, we add $(n_2-n_1)$ dummy nodes in $G^1$ without any labels and edges following previous works~\cite{justice2006binary,riesen2013novel}, as Figure~\ref{fig:GEDGW_dummy} illustrates.
This ensures that the two graphs have the same number of nodes without affecting the GED computation. For simplicity, we abuse the notations to still denote the graph after adding dummy nodes by $G^1$ and let $n = n_2=\max \{n_1, n_2\}$ in this section. 

Given a node matching, we can derive its induced edit operations into those on nodes and edges. Accordingly, GED computation can be derived by solving the following quadratic programming problem where the first (resp.\ second) term in the objective models the cost of node (resp.\ edge) edit operations. Appendix~\ref{app:diagram_gedgw} provides a detailed illustration of the GEDGW formulation.  
\begin{align}\label{eq:ged-match}
\min_{\bm{\pi}}~ \sum_{i,k}\mathbf{M}_{i,k}&\bm{\pi}_{i,k}+\frac{1}{2}\sum_{i,j,k,l} (\mathbf{A}^1_{i,j}-\mathbf{A}^2_{k,l})^2 \bm{\pi}_{i, k} \bm{\pi}_{j, l}, \\
\text { s.t. } & \bm{\pi}\bm{1}_{n}=\bm{1}_{n},\ \ \bm{\pi}^\top\bm{1}_{n}=\bm{1}_{n},\ \ \bm{\pi} \in\{0,1\}^{n\times n}. \notag
\end{align}
Here, $\mathbf{M}\in\{0,1\}^{n\times n}$ is the node label matching matrix between nodes of $G^1$ and $G^2$, where $\mathbf{M}_{i,k}=1$ if nodes $u_i\in V^1$ and $v_k\in V^2$ have the same label; otherwise $\mathbf{M}_{i,k}=0$. Matrices $\mathbf{A}^1\in\{0,1\}^{n\times n}$ and $\mathbf{A}^2\in\{0,1\}^{n\times n}$ are the adjacency matrices of $G^1$ and $G^2$, respectively. The factor $\frac{1}{2}$ in the second term is to avoid the double counting of $\bm{\pi}_{i, k} \bm{\pi}_{j, l}$ and $\bm{\pi}_{j, l} \bm{\pi}_{i, k}$ since the graphs are undirected. 

More concretely, the first linear term of Eq.~\eqref{eq:ged-match} measures the cost of the node edit operations, including (1)~The insertion/deletion of a node as indicated by matching a node in $G^2$ and a dummy node in $G^1$, and (2)~the relabeling operation as represented by matching a node in $G^2$ to an original node in $G^1$ whose labels are different.

The second quadratic term measures the cost of edge insertion/deletion since each element $(\mathbf{A}^1_{i,j}-\mathbf{A}^2_{k,l})^2\bm{\pi}_{i,k}\bm{\pi}_{j,l}$ in the sum measures whether edge $(u_i, u_j)\in E^1$ and edge $(v_k, v_l)\in E^2$ exist simultaneously when $u_i$ matches $v_k$ and $u_j$ matches $v_l$. 

After relaxing the binary constraint to allow elements of $\bm\pi$ to take values in $[0,1]$, the solution $\bm\pi$ represents the confidence of node-to-node matching between $G^1$ and $G^2$. 
Note that Eq.~\eqref{eq:ged-match} with relaxation on binary variables can be regarded as a linear combination of optimal transport (OT) and Gromov-Wasserstein Discrepancy (GW), {where the first linear term models the edit operations on nodes as an OT problem (the right part of Figure~\ref{fig:diagram_gedgw} in Appendix~\ref{app:diagram_gedgw}) and the second quadratic term models the edit operations on edges as a GW problem (the left part of Figure~\ref{fig:diagram_gedgw} in Appendix~\ref{app:diagram_gedgw}).} So we call this method GEDGW. The optimization problem of GEDGW is reformulated as follows: 
\begin{equation}
\label{eq:ged-fgw}
   \min_{\bm{\pi}\in\Pi(\bm{1}_{n},\bm{1}_{n})}~ \left<\bm{\pi},\mathbf{M} \right>+\dfrac{1}{2}\left<\bm{\pi},\mathcal{L}(\mathbf{A}^1,\mathbf{A}^2)\otimes\bm{\pi}\right> 
\end{equation}
where $\Pi(\bm{1}_{n},\bm{1}_{n})=\left\{\bm{\pi}\in\mathbb{R}^{n\times n}\,|\, \bm{\pi}\bm{1}_{n}=\bm{1}_{n},~\bm{\pi}^\top\bm{1}_{n}=\bm{1}_{n}, \bm{\pi}\ge0\right\}$ is the feasible set of coupling matrices. 
We exploit the Conditional Gradient (CG) method~\cite{braun2022conditional,vincent2021semi} to solve GEDGW, which is presented in detail in Appendix~\ref{app:cg}. {An example in Appendix~\ref{app:case-study} further illustrates our GEDGW method. } 

\vspace{-1mm}
\subsection{Further Improvement by Ensembling}
Recall that GEDGW and GEDIOT model the GED computation from two different perspectives via optimal transport.
To achieve better performance, we combine these two OT-based methods into an ensemble \textbf{GEDHOT} (GED with Hybrid Optimal Transport), which combines the results from GEDGW and GEDIOT to enhance the performance of GED computation and GEP generation during test.
 
Specifically, given an input of graph pair $(G^1, G^2)$, we run GEDGW and GEDIOT to get the GEDs and coupling matrices denoted by $\widehat{GED}_{\text{GW}}(G^1, G^2)$ and $\widehat{\bm{\pi}}_{\text{GW}}$, $\widehat{GED}_{\text{IOT}}(G^1, G^2)$ and $\widehat{\bm{\pi}}_{\text{IOT}}$, respectively. Since GED is the minimum number of edit operations, we choose the smaller of $\widehat{GED}_{\text{GW}}(G^1, G^2)$ and $\widehat{GED}_{\text{IOT}}(G^1, G^2)$ as $\widehat{GED}(G^1,G^2)$. 
\begin{equation*}
    \widehat{GED}(G^1,G^2)=\min \left\{ \widehat{GED}_{\text{GW}}(G^1,G^2), \widehat{GED}_{\text{IOT}}(G^1,G^2) \right\}.
\end{equation*}
For GEP generation, we generate the best edit paths via the $k$-best matching framework~\cite{piao2023gedgnn} from $\widehat{\bm{\pi}}_{\text{GW}}$ and $\widehat{\bm{\pi}}_{\text{IOT}}$, respectively, and then choose the shorter one.

\vspace{-1mm}
\subsection{Time Complexity Analysis}
Due to space limitation, we provide a comprehensive analysis of the time complexity of our proposed methods in Appendix~\ref{app:time2}.

In a nutshell, for GEDIOT, since the model training can be done offline given a graph dataset, we consider the computation cost of its forward propagation, the time complexity of which is given by
$$O\left(N(md+nd^2+nN^2d^2)+Ld^2+nd^2+n^2d+Mn^2\right)\approx O(n^2),$$
where we assume that the number of GNN layers is $N$, the dimension of hidden layers of GNN and MLP is $d$, the output dimension of NTN is $L$, $n=n_2$, $m=\max(m_1,m_2)$ and $M$ denotes the number of iterations of the Sinkhorn algorithm. {As the hyperparameters are regarded as constants, the time complexity can be simplified to $O(n^2)$, which is the same as previous learning-based methods in the worst case.}
For GEP generation via the $k$-best matching framework, the time complexity is $O(kn^3)$.

For GEDGW, we use the CG method~\cite{braun2022conditional,vincent2021semi} (see Appendix~\ref{app:cg} for details) to solve Eq.~\eqref{eq:ged-fgw}. The time complexity of CG is bounded by $O(Kn^3)$, where $K$ is the number of iterations. For GEDHOT, the time complexity is $O(n^2+Kn^3)=O(Kn^3)$, and the time complexity to generate GEP using the $k$-best matching framework is $O(kn^3)$. {Note that both GEDGW and GEDHOT have the same time complexity as the classical heuristic algorithms (e.g. Hungarian and VJ).}

\section{Experiment}
\label{sec:experiment}

This section evaluates the performance of our proposed methods and compares with existing approximate GED computing methods. Our code is released at \url{https://github.com/chengqihao/GED-via-Optimal-Transport}. 

\begin{table}[!t]
    \begin{center}
    \caption{Statistics of Graph Datasets}
    \label{Table:graph}
    \vspace{-2mm}
    \resizebox{0.8\columnwidth}{!}{
        \begin{tabular}{c|cccccc}
    \hline
    $\mathcal{D}$  & $|\mathcal{D}|$ & $|V|_{avg}$ & $|E|_{avg}$ & $|V|_{max}$ & $|E|_{max}$ & $|L|$\\
    \hline
    AIDS &700&8.9 &8.8 &10 &14 &29\\ 
            LINUX &1000 &7.6 &6.9 &10 &13 &1\\ 
            IMDB &1500 &13 &65.9 &89 &1467 &1\\ \hline
        \end{tabular}}
    \end{center}
\end{table}

\vspace{-1mm}
\subsection{Datasets}\label{sec:dataset}
We use three real-world graph datasets: AIDS, Linux, and IMDB. Table~\ref{Table:graph} summarizes their statistics including the number of graphs ($|\mathcal{D}|$), the average number of nodes $(|V|_{avg})$ and edges ($|E|_{avg}$), the maximum number of nodes $(|V|_{max})$ and edges ($|E|_{max}$), and the number of labels ($|L|$). For graph pairs with no more than 10 nodes, we use the A* algorithm~\cite{riesen2013novel} to generate the exact ground truth, and for the remaining graphs with more than 10 nodes, we use the ground-truth generation technique in~\cite{piao2023gedgnn,bai2021tagsim} to generate 100 synthetic graphs for each graph. 
For each dataset, we sample 60\% graphs and pair every two of them to create graph pairs of the training set. As for the test set, we sample 20\% graphs; for each selected graph, 100 graphs are randomly chosen from the training graphs to generate 100 graph pairs for the test set. The validation set is formed in the same manner as the test set. Appendix~\ref{app:data} describes the datasets, data preprocessing, and dataset partitions in detail.

\vspace{-2mm}
\subsection{Compared Methods}
\label{subsec:baseline}
Recall that GEDGW is a non-learning approximation algorithm, GEDIOT is a learning-based method, and GEDHOT is a combination of both. We compare them with the classical approximation algorithms and learning-based methods.

\vspace{1mm}
\noindent\textbf{Classical Algorithms. }We select three representative classical approximate algorithms for GED computation. (1) {\bf Hungarian}~\cite{riesen2009Hungarian} is based on the Hungarian method for weighted graph matching which takes cubic time. (2) {\bf VJ}~\cite{fankhauser2011VJ} is based on bipartite graph matching which takes cubic time. (3)~\textbf{Classic} runs both Hungarian and VJ to find the GEPs, and takes the better GEP. 
We do not include the heuristic A*-beam algorithm~\cite{neuhaus2006A*beam} since Noah in the paragraph below is an optimized version of A*-beam with better performance.

\vspace{1mm}
\noindent\textbf{Learning-based Methods. }We choose four state-of-the-art learning-based methods for GED computation. (1) {\bf SimGNN}~\cite{bai2019simgnn} is the very first learning method applying GNN for GED computation. (2) {\bf Noah} and {\bf GPN}~\cite{yang2021noah}. Noah employs the well-designed graph path network (GPN) to optimize the search direction of the A*-Beam algorithm~\cite{neuhaus2006A*beam} to find GEP. Additionally, GPN can also be utilized independently for GED computation only. (3) {\bf TaGSim}~\cite{bai2021tagsim} categories edit operations to four different types, and learns the number of edit operations in each type to achieve competitive GED approximation. (4) {\bf GEDGNN}~\cite{piao2023gedgnn} is the latest method for both GED computation and GEP generation. See Section~\ref{sec:related} for a detailed review. 

\vspace{1mm}
\noindent\textbf{Our Methods. } We propose \textbf{GEDIOT}, \textbf{GEDGW}, and \textbf{GEDHOT} for comparison. The detailed setup can be found in Appendix~\ref{app:setup}. 

\subsection{Evaluation metrics}
We consider four kinds of metrics to evaluate the performance, which have been widely used~\cite{bai2019simgnn,bai2021tagsim,piao2023gedgnn,yang2021noah}. 

\vspace{1mm}
\noindent\textbf{Metrics for GED Computation. }(1) {\bf Mean Absolute Error} (MAE) measures the average absolute error between ground-truth GEDs and approximate GEDs. For a graph pair $(G^1,G^2)$, it is formulated as $|GED^{*}(G^1,G^2)-\widehat{GED}(G^1,G^2)|$. 
(2) {\bf Accuracy} measures the ratio of approximate GEDs that equal the ground-truth GEDs after rounding to the nearest integer. 
(3) {\bf Feasibility} measures the ratio that the approximate GEDs are no less than the ground-truth GEDs, so that a GEP of this length is feasible (i.e., can be found).

\vspace{1mm}
\noindent\textbf{Metrics for Ranking. } These metrics measure the matching ratio between the ranking results of the approximate GED and the ground truth. 
They include 
(4) {\bf Spearman’s Rank Correlation Coefficient} ($\rho$). (5) {\bf Kendall’s Rank Correlation Coefficient} ($\tau$). (6) {\bf Precision at $k$} ($p@k$). 
The first two metrics focus on global ranks while the last focuses on top $k$. We use $p@10$ and $p@20$. 

\vspace{1mm}
\noindent\textbf{Metrics for Path. } These metrics measure how well the generated edit path $GEP$ matches the ground-truth $GEP^{*}$. They include (7)~$Recall=\frac{|GEP\cap GEP^{*}|}{|GEP^{*}|}$, (8)~$Precision=\frac{|GEP\cap GEP^{*}|}{|GEP|}$, and (9)~F1 score defined as $F1=2\cdot\frac{Recall\cdot Precision}{Recall+Precision}$.

\vspace{1mm}
\noindent\textbf{Metrics for Efficiency. }(10) Running Time ($sec/100p$), where $p=$ ``pairs''. 
It records the time for every 100 graph pairs during test.

\begin{table*}[t]
    \centering
	\caption{Performance Evaluations of GED Computation.}
    \vspace{-2mm}
    \resizebox{1.38\columnwidth}{!}{
    \begin{tabular}{c|c|c c|c c c c|c|c}
    \hline
    \multirow{2}{*}{\tabincell{c}{ Datasets}} &\multirow{2}{*}{\tabincell{c}{ Methods}} & \multicolumn{2}{c|}{Value} & \multicolumn{4}{c|}{Ranking} & \multirow{2}{*}{\tabincell{c}{Feasibility $\uparrow$}} & \multirow{2}{*}{\tabincell{c}{Time $\downarrow$\\ ($sec/100p$)}}\\
    \cline{3-8}
      & &MAE $\downarrow$ &Accuracy $\uparrow$ &$\rho$ $\uparrow$ &$\tau$ $\uparrow$ &$p@10$ $\uparrow$ &$p@20$ $\uparrow$ & \\ \hline
      \multirow{9}{*}{\tabincell{c}{AIDS}} &SimGNN &0.880 &34.7\% &0.841 &0.704 &0.632 &0.741 &61.5\% &0.279 \\
                                           &GPN    &0.924 &35.6\% &0.816 &0.680 &0.606 &0.713 &66.5\% &\underline{0.245} \\
                                           &TaGSim &0.807 &37.4\% &0.862 &0.730 &0.669 &0.754 &66.2\% &\bf{0.087} \\
                                           &GEDGNN &0.763 &40.4\% &0.870 &0.742 &0.716 &0.774 &72.1\% &0.307 \\ 
                                           &GEDIOT&\underline{0.581} &\underline{49.7\%} &\underline{0.922} &\underline{0.813} &\underline{0.814} &\underline{0.853} &\underline{73.9\%} &0.318 \\ \cline{2-10}
                                           &Classic &6.594 &3.3\% &0.529 &0.418 &0.545 &0.614 &\bf{100\%} &1.463 \\
                                           &GEDGW &1.247 &41.2\% &0.789 &0.670 &0.752 &0.765 &\bf{100\%} &0.430 \\ \cline{2-10}
                                           &Noah &3.164 &5.6\% &0.704 &0.585 &0.681 &0.721 &\bf{100\%} &161.023 \\
                                           &GEDHOT &\bf{0.484} &\bf{59.3\%} &\bf{0.936} &\bf{0.838} &\bf{0.863} &\bf{0.885} &\underline{73.9\%} &0.745 \\
                                           \hline \hline                             

     \multirow{9}{*}{\tabincell{c}{Linux}} &SimGNN &0.408 &63.3\% &0.939 &0.856 &0.911 &0.916 &75.6\% &0.278 \\
                                           &GPN    &0.142 &87.1\% &0.959 &0.896 &0.947 &0.974 &90.5\% &\underline{0.265} \\
                                           &TaGSim &0.346 &69.6\% &0.937 &0.859 &0.888 &0.910 &85.9\% &\bf{0.069} \\
                                           &GEDGNN &0.094 &91.6\% &0.961 &0.897 &0.980 &0.976 &95.9\% &0.282 \\ 
                                           &GEDIOT &\underline{0.034} &\underline{97.2\%} &\underline{0.969} &\underline{0.911} &\underline{0.992} &\underline{0.995} &\underline{98.5\%} &0.326 \\ \cline{2-10}
                                           &Classic &2.471 &21.5\% & 0.785 &0.707 &0.762 &0.835 &\bf{100\%} &0.915 \\
                                           &GEDGW &1.198 &48.1\% &0.817 &0.705 &0.827 &0.811 &\bf{100\%} &0.382 \\ \cline{2-10}
                                           &Noah &1.736 &8.4\% &0.870 &0.798 &0.906 &0.936 &\bf{100\%} &71.646 \\
                                           &GEDHOT &\bf{0.026} &\bf{97.9\%} &\bf{0.970} &\bf{0.915} &\bf{0.994} &\bf{0.997} &\underline{98.5\%} &0.754 \\ \hline \hline
    \multirow{9}{*}{\tabincell{c}{IMDB}}   &SimGNN &1.191 &40.4\% &0.735 &0.648 &0.759 &0.799 &68.1\% &0.291 \\
                                           &GPN    &1.614 &28.2\% &0.742 &0.668 &0.669 &0.708 &34.3\% &\underline{0.229} \\
                                           &TaGSim &5.247 &14.8\% &0.496 &0.441 &0.666 &0.699 &47.7\% &\bf{0.095} \\
                                           &GEDGNN &0.735 &59.6\% &0.859 &0.781 &0.838 &0.856 &80.2\% &0.305 \\
                                           &GEDIOT &\underline{0.584} &65.3\% &\underline{0.930} &0.858 &0.902 &0.912 &78.6\% &0.347 \\ \cline{2-10}
                                           &Classic &12.980 &62.8\% &0.764 &0.718 &0.837 &0.831 &\bf{100\%} &3.483 \\
                                           &GEDGW &0.818 &\bf{83.0\%} &0.926 &\underline{0.896} &\underline{0.968} &\underline{0.951} &\underline{93.6\%} &0.247\\ \cline{2-10}
                                           &Noah & 10.467 & 38.4\% & 0.717 & 0.688 & 0.755 & 0.795 & \bf{100\%} & 4816.67 \\
                                           &GEDHOT &\bf{0.506} &\underline{69.9\%} &\bf{0.956} &\bf{0.899} &\bf{0.978} &\bf{0.972} &73.1\% &0.607 \\ \hline
    \end{tabular}
    }
    \begin{tablenotes}
        \item\ \ \ \ \ \ \ \ \ \ \ \ \ \ \ \ $\qquad\quad$ $\uparrow$: higher is better, $\downarrow$: lower is better\ \ $\qquad\quad$ \textbf{Bold}: best, \underline{Underline}: runner-up.
    \end{tablenotes}
    \label{table:ged-prediction}
    \vspace{-2mm}
\end{table*}

\begin{table*}[t]
    \centering
	\caption{Performance Evaluations of GEP Generation.}
    \vspace{-2mm}
    \resizebox{1.44\columnwidth}{!}{
    \begin{tabular}{c|c|c c|c c c c|c c c|c}
    \hline
    \multirow{2}{*}{\tabincell{c}{ Datasets}} &\multirow{2}{*}{\tabincell{c}{ Methods}} & \multicolumn{2}{c|}{Value} & \multicolumn{4}{c|}{Ranking} & \multicolumn{3}{c|}{Path}
    &\multirow{2}{*}{\tabincell{c}{Time $\downarrow$\\ ($sec/100p$)}}\\
    \cline{3-11}
      & &MAE $\downarrow$ &Accuracy $\uparrow$ &$\rho$ $\uparrow$ &$\tau$ $\uparrow$ &$p@10$ $\uparrow$ &$p@20$ $\uparrow$ &$Recall$ $\uparrow$ &$Precision$ $\uparrow$ &$F1$ $\uparrow$ & \\ \hline
      \multirow{6}{*}{\tabincell{c}{AIDS}} 
                                           &Classic  &6.594 &3.3\% &0.529 &0.418 &0.545 &0.614 &0.572 &0.345 &0.423 &\bf1.752 \\
                                           &Noah &3.164 &5.6\% &0.704 &0.585 &0.681 &0.721 &0.609 &0.505 &0.548 &163.153 \\
                                           &GEDGNN    &1.503 &42.2\% &0.795 &0.690 &0.849 &0.838 &0.715 &0.646 &0.675 &\underline{56.439}   \\
                                           &GEDIOT &1.266 &49.9\% &0.814 &0.715 &\underline{0.881} &\underline{0.858} &\underline{0.756} &\underline{0.692} &\underline{0.719} &57.857\\ 
                                           &GEDGW &\underline{0.829} &\underline{53.2\%} &\underline{0.862} &\underline{0.774} &0.842 &\underline{0.858}  &0.715 &0.675 &0.692 &57.102 \\
                                           &GEDHOT &\bf0.440 &\bf71.2\% &\bf0.923 &\bf0.864 &\bf0.951 &\bf0.935 &\bf0.809 &\bf0.786 &\bf0.796 &112.161  \\  \hline \hline                  
     \multirow{6}{*}{\tabincell{c}{Linux}} 
                                           &Classic &2.471 &21.5\% & 0.785 &0.707 &0.762 &0.835 &0.770 &0.541 &0.623 &\bf0.954  \\
                                           &Noah &1.736 &8.4\% &0.870 &0.798 &0.906 &0.936 &0.851 &0.772	&0.802 &73.018 \\
                                           &GEDGNN    &0.156 &93.5\% &0.970 &0.954 &0.987 &0.980 &0.917 &0.904 &0.909 &\underline{19.317} \\
                                           &GEDIOT &\underline{0.114} &\underline{95.4\%} &\underline{0.976} &\underline{0.965} &\underline{0.988} &\underline{0.987} &\underline{0.924} &\underline{0.914} &\underline{0.918} &19.514  \\ 
                                           &GEDGW &0.591 &72.2\% &0.898 &0.836 &0.925 &0.887 &0.837 &0.780 &0.802 &26.788  \\
                                           &GEDHOT &\bf0.033 &\bf98.4\% &\bf0.994 &\bf0.990 &\bf0.992 &\bf0.996 &\bf0.928 &\bf0.924 &\bf0.926 &47.523  \\  \hline \hline
    \multirow{6}{*}{\tabincell{c}{IMDB}}   
                                           &Classic &12.980 &62.8\% &0.764 &0.718 &0.837 &0.831 &0.833 &0.628 &0.654 &\bf3.663\\
                                           &Noah & 10.467 & 38.4\% & 0.717 & 0.688 & 0.755 & 0.795 & 0.845 & 0.670 & 0.682 & 4864.38 \\
                                           &GEDGNN    &3.574 &79.6\% &0.888 &0.859 &0.924 &0.924 &\underline{0.907} &0.808 &0.826 &93.893 \\
                                           &GEDIOT &3.638 &82.0\% &0.903 &0.878 &0.923 &0.928 &\underline{0.907} &\underline{0.816} &\underline{0.831} &93.091  \\ 
                                           &GEDGW &\underline{0.374} &\underline{93.2\%} &\underline{0.969} &\underline{0.955} &\underline{0.988} &\underline{0.983} &0.763 &0.736 &0.744 &\underline{81.948}  \\
                                           &GEDHOT &\bf0.254 &\bf95.0\% &\bf0.983 &\bf0.972 &\bf0.995 &\bf0.993 &\bf0.946 &\bf0.927 &\bf0.933 &170.412 \\  \hline
    \end{tabular}
    }
    \begin{tablenotes}
        \item\ \ \ \ \ \ $\qquad\quad$ $\uparrow$: higher is better, $\downarrow$: lower is better\ \ $\qquad\quad$ \textbf{Bold}: best, \underline{Underline}: runner-up.
    \end{tablenotes}
    \label{table:gep-generation}
    \vspace{-3mm}
\end{table*}

\begin{table*}[t]
    \centering
	\caption{GED Computation of Unseen Graph Pairs.}
    \vspace{-3mm}
    \resizebox{1.36\columnwidth}{!}{
    \begin{tabular}{c|c|c c|c c c c|c|c}
    \hline
    \multirow{2}{*}{\tabincell{c}{ Datasets}} &\multirow{2}{*}{\tabincell{c}{ Methods}} & \multicolumn{2}{c|}{Value} & \multicolumn{4}{c|}{Ranking} & \multirow{2}{*}{\tabincell{c}{Feasibility $\uparrow$}} & \multirow{2}{*}{\tabincell{c}{Time $\downarrow$\\ ($sec/100p$)}}\\
    \cline{3-8}
      & &MAE $\downarrow$ &Accuracy $\uparrow$ &$\rho$ $\uparrow$ &$\tau$ $\uparrow$ &$p@10$ $\uparrow$ &$p@20$ $\uparrow$ & \\ \hline
      \multirow{6}{*}{\tabincell{c}{AIDS}} &SimGNN &0.925 &34.4\% &0.808 &0.668 &0.631 &0.731 &63.6\% &0.284 \\
                                           &GPN    &1.038 &33.4\% &0.771 &0.631 &0.578 &0.683 &64.5\% &\underline{0.235} \\
                                           &TaGSim &0.880 &34.8\% &\underline{0.832} &0.694 &0.674 &0.739 &66.0\% &\bf0.093 \\
                                           &GEDGNN &\underline{0.826} &\underline{38.0\%} &0.831 &\underline{0.696} &\underline{0.702} &\underline{0.750} &\underline{69.4\%} &0.298 \\ 
                                           &GEDIOT &\bf0.684 &\bf44.5\% &\bf0.897 &\bf0.776 &\bf0.791 &\bf0.835 &\bf71.3\% &0.313 \\  \hline \hline
     \multirow{6}{*}{\tabincell{c}{Linux}} &SimGNN &0.399 &63.2\% &0.953 &0.877 &0.934 &0.918 &77.6\% &0.288 \\
                                           &GPN    &0.147 &86.6\% &\underline{0.973} &\underline{0.916} &0.948 &0.967 &90.5\% &\underline{0.279} \\
                                           &TaGSim &0.347 &69.3\% &0.951 &0.877 &0.878 &0.905 &87.4\% &\bf0.079 \\
                                           &GEDGNN &\underline{0.122} &\underline{89.8\%} &0.965 &0.904 &\underline{0.968} &\underline{0.973} &\underline{95.1\%} &0.291 \\ 
                                           &GEDIOT &\bf0.051 &\bf96.1\% &\bf0.976 &\bf0.925 &\bf0.983 &\bf0.990 &\bf97.6\% &0.336 \\  \hline \hline
    \multirow{6}{*}{\tabincell{c}{IMDB}}   &SimGNN &1.236 &39.3\% &0.733 &0.642 &0.755 &0.801 &67.4\% &0.307 \\
                                           &GPN    &1.635 &27.7\% &0.741 &0.664 &0.670 &0.710 &33.9\% &\underline{0.226} \\
                                           &TaGSim &4.811 &15.4\% &0.501 &0.445 &0.665 &0.700 &47.2\% &\bf0.107 \\
                                           &GEDGNN &\underline{0.743} &\underline{59.2\%} &\underline{0.858} &\underline{0.777} &\underline{0.842} &\underline{0.857} &\bf79.8\% &0.294 \\ 
                                           &GEDIOT &\bf0.595 &\bf65.5\% &\bf0.925 &\bf0.850 &\bf0.903 &\bf0.913 &\underline{78.5\%} &0.353 \\  \hline
    \end{tabular}
    }
    \begin{tablenotes}
        \item\ \ \ \ \ \ \ \ \ \ \ $\qquad\quad$ $\uparrow$: higher is better, $\downarrow$: lower is better\ \ $\qquad\quad$ \textbf{Bold}: best, \underline{Underline}: runner-up.
    \end{tablenotes}
    \label{table:ged-learning}
    \vspace{-4mm}
\end{table*}

\vspace{-3mm}
\subsection{Experimental Results}
We evaluate the performance of various methods for both GED computation and GEP generation. 

\vspace{1mm}
\noindent\textbf{Performance of GED Computation. } We first compare our proposed methods (i.e., GEDGW, GEDIOT, and GEDHOT) with the six baselines mentioned in Section~\ref{subsec:baseline} (Hungarian and VJ are dominated by Classic and are hence omitted due to space limit). We categorize the methods into three types: learning-based methods, non-learning methods, and hybrid methods. We count Noah also as a hybrid method since it combines GPN with A*-Beam.

Table~\ref{table:ged-prediction} reports the results. We can see that among the learning-based methods, GEDGNN achieves the best performance on all three datasets for value, ranking, and feasibility metrics.
Meanwhile, GEDIOT significantly outperforms GEDGNN (as well as the other learning-based baselines) in terms of value and ranking metrics with comparable time consumption. 
For instance, compared with the state-of-the-art method GEDGNN, the MAE of our proposed GEDIOT is $23.9\%$, $63.8\%$, $20.5\%$ smaller on AIDS, Linux, and IMDB, respectively; 
also, on AIDS, the accuracy of GEDGNN and our GEDIOT is $40.4\%$ and $49.7\%$, respectively. 
{Note that TaGSim is the most time-efficient (e.g., on AIDS, the training time for an epoch of TaGSim is 151~s, while that of GEDIOT is 581~s) but cannot return high-quality results. We train TaGSim for more epochs so that the total training time of TaGSim is roughly equal to GEDIOT, and the results are similar to that reported in Table~\ref{table:ged-prediction}. On AIDS, Linux, and IMDB, the MAE and accuracy of TaGSim with more training time are $0.816$ and $37.9\%$, $0.316$ and $70.6\%$, $4.962$ and $11.4\%$ respectively, which are still worse than our model. It demonstrates that our experimental setup is sufficient to converge.
}

For the non-learning methods, Classic and GEDGW, it is obvious that GEDGW achieves much better performance on all the value and ranking metrics with up to $14 \times$ faster computational speed. More surprisingly, on AIDS and IMDB, GEDGW even achieves a higher accuracy than the state-of-the-art learning-based method GEDGNN. Note that the training phase for the learning-based methods always takes several hours, while GEDGW does not need that phase and directly outputs results within a second. Moreover, all the learning-based methods need the ground truths of GED and node matching for model training. 
The performance of GEDGW suggests that it is possible to approximate high-quality GEDs in a non-learning way. 

Finally, for the two hybrid methods, Noah and GEDHOT, we can see that compared to Noah, the MAE of our GEDHOT is up to $20 \times$ smaller with hundreds of times smaller computational time (recall that Noah runs the expensive A* algorithm). 
In addition, in Table~\ref{table:ged-prediction}, GEDHOT clearly outperforms all the other methods, followed by the proposed GEDIOT and GEDGW with a consistent second-best performance on all three datasets. For instance, on AIDS, the accuracies of GEDIOT, GEDGW and GEDHOT are $49.7\%$, $41.2\%$, and $59.3\%$ respectively, while that of GEDGNN is only $40.4\%$. This shows that GEDHOT can combine the merits of both GEDIOT and GEDGW to get better results. 

\vspace{1mm}
\noindent\textbf{Performance of GEP Generation. } We next compare the performance of GEP generation of the methods above. Note that among the learning-based baselines, Noah and GEDGNN are the only two that can generate GEP, so we include Noah, GEDGNN, and Classic as baselines in Table~\ref{table:gep-generation} for comparison. 
We can see that Classic takes the shortest computational time, but the MAE is several times larger than other methods. Among the other four methods, similar to GED results in Table~\ref{table:ged-prediction}, GEDHOT achieves the best performance for value and ranking metrics on all the three datasets. For example, on AIDS, the accuracy of GEDGNN and GEDHOT is $42.2\%$ and $71.2\%$ respectively; also, on Linux, GEDHOT obtains $4.7 \times$, $17.9 \times$, $3.5 \times$ smaller MAE compared with GEDGNN, GEDGW, GEDIOT, respectively. Moreover, the second-best is either GEDIOT or GEDGW.

Note that the computational time of GEDHOT is about twice as large as the time of the other three methods except for Classic. 
Even if a smaller time cost is preferred, our proposed GEDGW and GEDIOT are preferred compared to GEDGNN, which is the latest method for GEP generation. It is worth noting that on AIDS and IMDB, the non-learning method GEDGW even achieves $1.8 \times$ and $9.6 \times$ smaller MAE than the learning-based method GEDGNN. 

Regarding path quality metrics, Recall, Precision, and F1 score, Table~\ref{table:gep-generation} shows that GEDHOT consistently performs the best, and GEDIOT is consistently the second-best.


{We also study the contribution of GEDIOT and GEDGW for the ensemble method GEDHOT. 
For example, on AIDS, for GED computation, most graph pairs ($80.8\%$) use the results from GEDIOT instead of GEDGW. For GEP generation, $63.1\%$ of the graph pairs use the results from GEDIOT, and $36.9\%$ of the graph pairs use the results from GEDGW. More results can be found in Appendix~\ref{app:adoption-ratio}. }

{Note that GED is a distance metric, satisfying the triangle inequality. Without loss of generality, we conduct experiments on AIDS and Linux to evaluate the fraction of triangle inequality violations in the predicted GEDs. The results shown in Appendix~\ref{app:adoption-ratio} indicate that our methods satisfy this property in most cases ($>95\%$).} 


\vspace{-2mm}
\subsection{Generalizability}\label{sec:genexp}
Since all the learning-based methods require training data supervision, it is interesting to explore how they generalize beyond the training data distribution, including our GEDIOT model.

\vspace{1mm}
\noindent\textbf{Modeling GED Computation of Unseen Graphs. } Recall that we prepared the test set by sampling 100 training graphs for each test graph, which models the graph similarity search task. To evaluate the generalizability, now we instead sample 100 test graphs (rather than training graphs) for each test graph, so that both graphs in a graph pair of the test set are unseen during training.

Table~\ref{table:ged-learning} shows the results of the five learning-based methods, where GEDIOT still significantly outperforms GEDGNN and the others in terms of value and ranking metrics. For example, on Linux, the MAE of GEDIOT is $2.4 \times$ smaller than GEDGNN, and the accuracy reaches $96.1 \%$  while that of GEDGNN is below $90 \%$. 

Compared with the results in Table~\ref{table:ged-prediction}, the performance of all methods decreases since the test set is more challenging. Nevertheless, the amount of degradation is not significant. For example, the accuracy of GEDIOT decreases by $10.5\%$ and $1.1\%$ on AIDS and Linux, respectively, which demonstrates its generalizability. 

\begin{table*}[t]
    \centering
    \caption{Ablation Study of GEDIOT Components.}
    \vspace{-2mm}
    \resizebox{1.83\columnwidth}{!}{
    \begin{tabular}{c|c c c c c c|c c c c c c}
    \hline
    \multirow{2}{*}{\tabincell{c}{ Method}} &\multicolumn{6}{c|}{AIDS} & \multicolumn{6}{c}{Linux} \\ \cline{2-13}
                  &MAE $\downarrow$ &Accuracy $\uparrow$ &$\rho$ $\uparrow$ &$\tau$ $\uparrow$ &$p@10$ $\uparrow$ &$p@20$ $\uparrow$ &MAE$\downarrow$ &Accuracy $\uparrow$ &$\rho$ $\uparrow$  &$\tau$ $\uparrow$ &$p@10$ $\uparrow$ &$p@20$ $\uparrow$\\ \hline
GEDIOT &0.581 &49.7\%    &0.922        &0.813       &0.814    &0.853        &0.034 &97.2\% &0.969 &0.911 &0.992  &0.995 \\
GEDIOT (w/ GCN) &0.578 &49.1\%   &0.917        &0.805       &0.794    &0.838        &0.064 &93.8\% &0.967 &0.909  &0.980 &0.985\\
GEDIOT (w/o MLP) &0.854    &35.9\%        &0.814       &0.677    &0.599        &0.678 &0.158 &85.9\% &0.958  &0.889 &0.934 &0.956\\
GEDIOT (w/o Cost)  &0.794	&38.4\%	&0.870	&0.741	&0.692	&0.765
&0.132	&87.5\%	 &0.964	&0.901	&0.953	&0.966\\
GEDIOT (w/o learnable $\varepsilon$) &0.767    &38.5\%        &0.906       &0.790    &0.801        &0.831 &0.063 &94.7\%  &0.967 &0.910 &0.988 &0.991\\
\hline
    \end{tabular}
    }
    \label{table:ablation}
    \vspace{-3mm}
\end{table*}

\begin{figure}[t]
    \centering
    \subfigure[IMDB - MAE]{
    \includegraphics[width=0.485\linewidth]{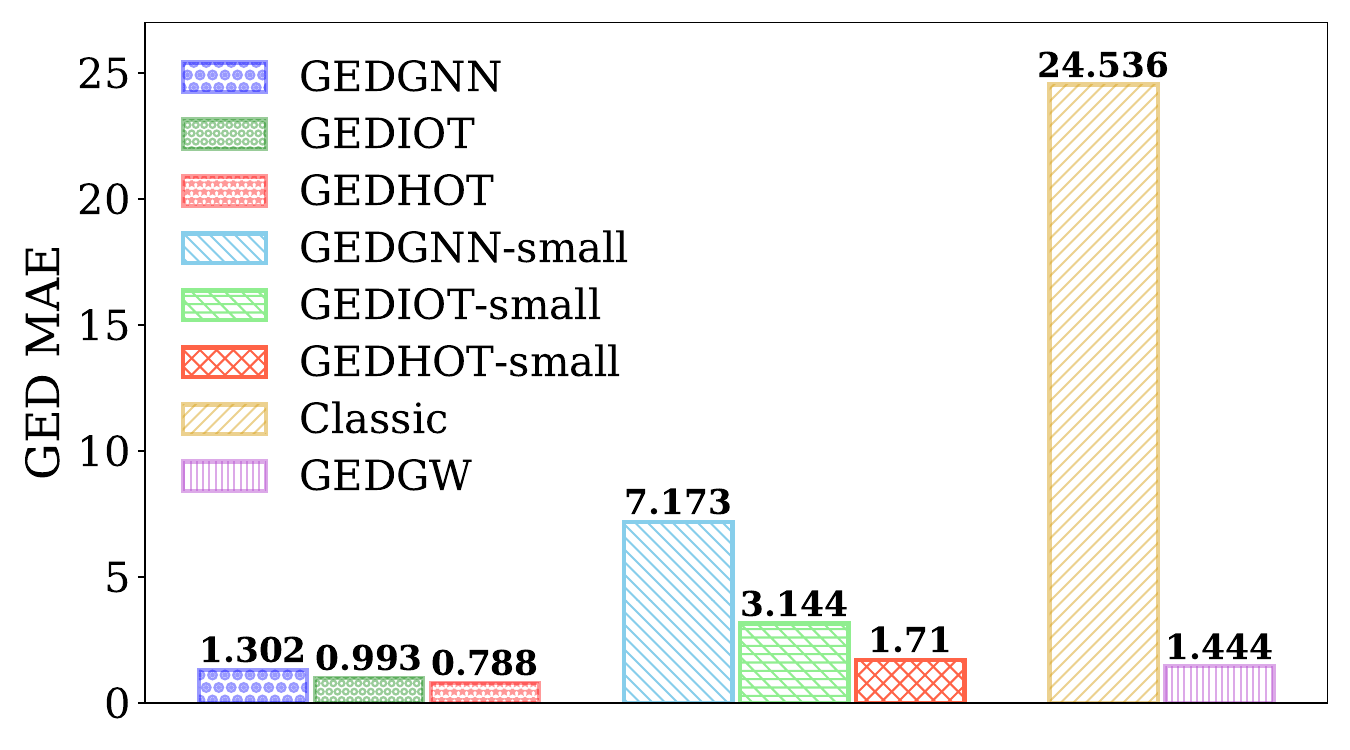}}
    \vspace{-2mm}
    \subfigure[IMDB - Accuracy]{
    \includegraphics[width=0.485\linewidth]{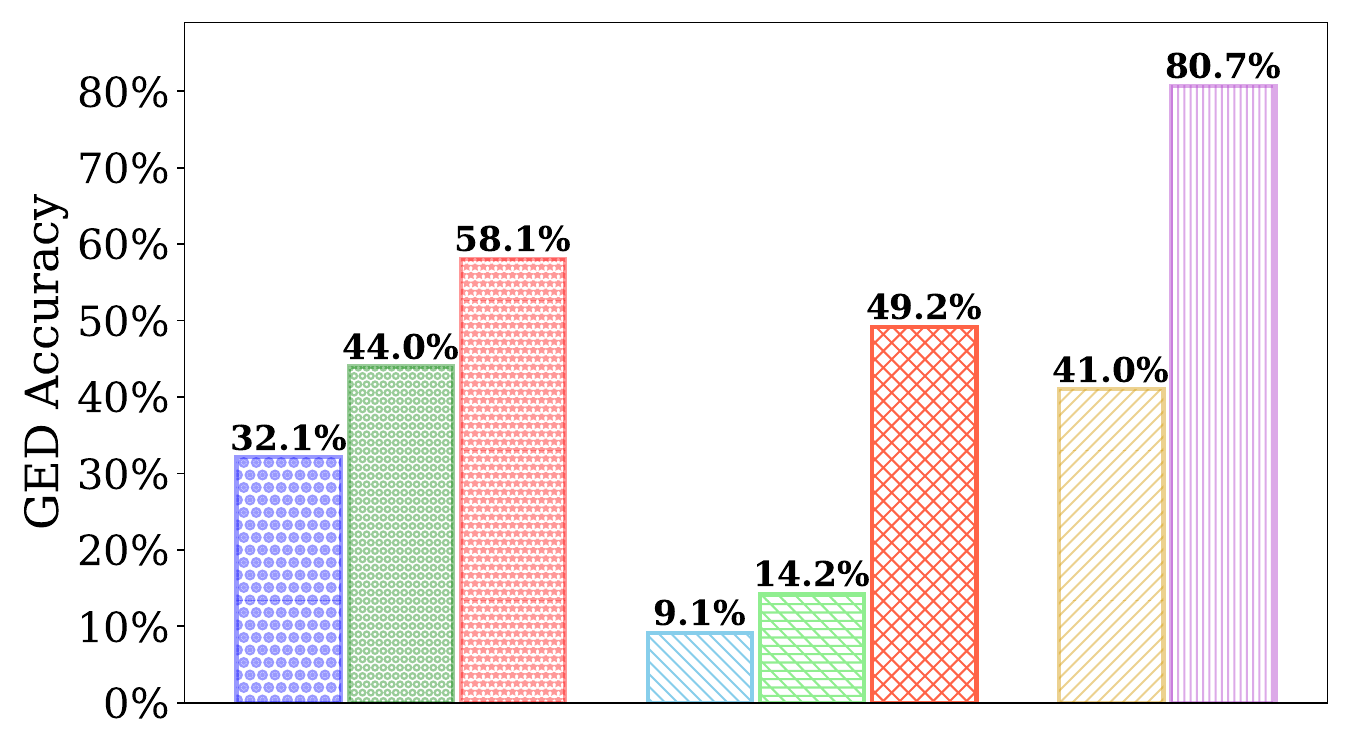}}
    \vspace{-2mm}
    \caption{Generalizability for Large Unseen Graphs on IMDB}
    \label{fig:generalization}
\end{figure}
\setlength{\textfloatsep}{5pt}

\vspace{1mm}
\noindent\textbf{Generalization to Large Unseen Graphs. } Ground truth is crucial for supervised learning-based methods. In GED computation, ground truth is difficult to obtain for large graphs due to the NP-hardness of the problem. For instance, there are plenty of graphs with more than 10 nodes in the IMDB dataset, and it is too expensive to calculate the GEDs of these graph pairs with exact algorithms. 
Therefore, we consider training the model only with small graphs and testing the performance of the learning-based methods on large unseen graphs. More concretely, we select the graph pairs from the training set of IMDB that are formed by the small graphs (at most 10 nodes) to build a new training set. All the methods trained on it are appended with the ``-small'' suffix, i.e., \textbf{GEDGNN-small}, \textbf{GEDIOT-small} and \textbf{GEDHOT-small}. To evaluate generalizability, we also construct a new test set, which consists of the graph pairs from the test set of IMDB that are formed by the large graphs (more than 10 nodes). The results are shown in Figure~\ref{fig:generalization}, where GEDGNN, GEDIOT, and GEDHOT denote the methods trained on the complete training set of IMDB.
%
We can see that models trained on small graphs have an inferior performance compared to training on complete training set. 
However, GEDHOT-small and GEDIOT-small are still significantly better than GEDGNN-small in terms of MAE and accuracy. 
Notably, GEDGW achieves the highest accuracy of 80.7\% since it is unsupervised, demonstrating its robustness as compared to learning-based methods that face generalizability challenges. 

We further discuss how the generalizability is impacted when synthesizing test graph pairs with larger GEDs. Similarly, GEDGW achieves the best performance and our neural model outperforms GEDGNN. Detailed results are shown in Figure~\ref{fig:generalization_small} in Appendix~\ref{app:gen_exp}. 

{We notice that the state-of-the-art methods Nass~\cite{kim2021boosting} and AStar-BMao~\cite{chang2022accelerating} for graph similarity search (introduced in Section~\ref{sec:related}) can be applied for exact GED computation by setting the similarity threshold to infinity. 
As indicated in~\cite{piao2023gedgnn}, exact methods suffer from huge computation costs when the graph size increases. 
We compare our method GEDIOT with Nass and AStar-BMao on two large real-world datasets. The detailed setup and the running time of the three methods can be found in Appendix~\ref{app:exact-methods}. 
We find that the running time of the two exact methods Nass and AStar-BMao is quite sensitive w.r.t. the graph size and the GED value. Our method GEDIOT shows a consistent advantage compared to the two exact algorithms, particularly for larger graphs and GEDs, since the time complexity of GEDIOT is only $O(n^2)$, whereas AStar-BMao and Nass are still exponential-time algorithms. }

{We also generate synthetic power-law graphs of
various sizes (from 50 to 400 nodes). The results are reported in Appendix~\ref{app:power-law}, where we find that the GED relative error of our GEDGW and GEDHOT is nearly $0$ while that of GEDGNN is always almost $2$, and the computational time of learning-based methods is orders of magnitude faster than the exact algorithms.}

\vspace{-1.3mm}
\subsection{Ablation and Parameter Study}
We conduct ablation study to verify the effectiveness of various modules in GEDIOT, and to show the robustness of GEDIOT w.r.t.\ hyperparameters by varying their values. 

\vspace{1mm}
\noindent\textbf{Effect of Modules in GEDIOT. } In this ablation study, we modify GEDIOT into four variants and compare their performance with GEDIOT. Table~\ref{table:ablation} shows the results, where we use ``w/ GCN'' to denote the variant substituting GIN with GCN in GEDIOT, and use ``w/o MLP'', ``w/o Cost'', and ``w/o learnable $\varepsilon$'' to denote GEDIOT that removes the MLP in the node embedding component, that replaces cost matrix module in the learnable OT component with $\mathbf{H}^1 ( \mathbf{H}^2)^\top$ (i.e., to model node interactions with simple inner product of their embeddings), and that fixes the regularization coefficient $\varepsilon$ in the learnable Sinkhorn layer as $\varepsilon_0=0.05$, respectively.  

As Table~\ref{table:ablation} shows, replacing or removing a module in GEDIOT can significantly degrade the performance of both value and ranking metrics, which verifies the effectiveness of our proposed components for GED computation. For instance, on AIDS, if fixing the regularization coefficient $\varepsilon$, the accuracy decreases from $49.7 \%$ to $38.5 \%$ and MAE increases from 0.581 to 0.767. 

\vspace{1mm}
\noindent{\textbf{Varying Parameters in the Sinkhorn Algorithm.} We study how the performance of GEDIOT is impacted as the initial value of the regularization coefficient, denoted by $\varepsilon_0$ and the number of iterations varies in the learnable Sinkhorn layer. The results are presented in Appendix~\ref{app:ablation_exp}.
We find that both MAE and accuracy are stable with various $\varepsilon_0$, which shows the robustness of the learnable regularization method to $\varepsilon_0$. }
Moreover, we observe that the MAE decreases and the accuracy increases as the number of iterations increases, but after $15$ (resp.\ $10$) iterations on AIDS (resp.\ Linux), the MAE and accuracy become fairly stable as the Sinkhorn algorithm converges.  
Note that the computational time also increases when conducting more iterations. Considering the time-accuracy tradeoff, we set the iteration number to $5$ by default. 

\vspace{1mm}
\noindent{\textbf{Varying $\lambda$ in the Loss Function.} As presented in Appendix~\ref{app:ablation_exp}, we also discuss the effect of varying $\lambda$ in Eq.~\eqref{eq:loss} (from $0$ to $1$) that balances the two terms $\mathcal{L}_m$ and $\mathcal{L}_v$ of the loss function. The results show that the performance improves with the increase of $\lambda$ in $[0, 1]$ and becomes stable when $\lambda$ is around $0.8$. }

\vspace{1mm}
\noindent\textbf{Varying the Size of Training Set.} 
In this experiment, we evaluate the effect of varying the training set size. Concretely, we randomly sample 10\%-100\% of the original training set of AIDS and Linux to retrain GEDIOT. The results in Appendix~\ref{app:ablation_exp} describe its influence on training time, MAE, and accuracy of GEDIOT. It can be observed that as the training set size increases, the MAE decreases and the accuracy increases, while the training time increases linearly. Furthermore, the observed trends of MAE and accuracy with increasing training set size appear to be flattening, which shows that training set size is sufficient.

\vspace{1mm}
\noindent\textbf{$k$-Best Matching.} We further verify the effect of $k$ in $k$-best matching for GEP generation. As depicted in Appendix~\ref{app:ablation_exp}, the MAE constantly decreases and the accuracy increases as the parameter $k$ increases. Nevertheless, computational time also increases with the increase of $k$ since the search space becomes larger. 

\vspace{-2mm}
\section{Conclusion}
\label{sec:conclusion}
In this paper, we proposed novel optimal-transport-based methods for graph edit distance computation and graph edit path generation from both learning and optimization perspectives. We first proposed a neural network with inverse optimal transport called GEDIOT. By modeling the node edit operations and edge edit operations as optimization problems, we also proposed an unsupervised method GEDGW to approximate the GED value without the need of training. 
Additionally, we combine the two methods and propose an ensemble method GEDHOT which achieves a higher performance. Experiments demonstrate that our methods outperform the state-of-the-art methods for GED computation and GEP generation with remarkable result quality and generalizability. 


\bibliographystyle{ACM-Reference-Format}
\bibliography{ref_gedot}


\appendix
\clearpage
\begin{appendix}
\section{Review of Approximate GED Computation}\label{app:review}
\subsection{Review of Heuristic Algorithms}
\label{app:heuristic}
{There are plenty of heuristic algorithms including A*-Beam~\cite{neuhaus2006A*beam}, Hungarian~\cite{riesen2009Hungarian}, and VJ~\cite{fankhauser2011VJ}, all of which provide an approximate GED in polynomial time. 
A*-Beam~\cite{neuhaus2006A*beam} bounds the search space in the exact algorithm A* with a user-defined beam size for efficiency. Hungarian~\cite{riesen2009Hungarian} constructs a cost matrix to estimate the number of edit operations induced by matching two nodes across two graphs and model the computation of GED as a linear sum assignment problem. It produces a matching matrix and approximate GED by solving it with the Hungarian algorithm~\cite{munkres1957algorithms}. In~\cite{riesen2009Hungarian}, the solution to a linear sum assignment problem concerning node matching is regarded as the approximation of the GED value. VJ~\cite{fankhauser2011VJ} improves upon the primal-dual method used in the Hungarian algorithm and incorporates more effective search strategies to reduce the computational overhead.} 
\subsection{Review of GNN-based Methods}
\label{app:gnn-based}
Recently, graph neural networks (GNN) have become popular since the extracted node and graph embeddings can greatly help the performance in node classification~\cite{xiao2022graph,zhou2020graph}, link prediction~\cite{zhang2018link,zhang2022graph}, and other classical graph problems~\cite{zhao2021learned,wang2024neural,li2023coclep}, etc. 
Consequently, a number of GNN-based methods, such as SimGNN~\cite{bai2019simgnn}, TaGSim~\cite{bai2021tagsim}, Noah~\cite{yang2021noah}, MATA*~\cite{liu2023mata} and GEDGNN~\cite{piao2023gedgnn}, have also been proposed to generate embedding for GED computation with adequate training data, which achieve best performance in approximate GED computation. 
SimGNN~\cite{bai2019simgnn} proposes to simply aggregate node embeddings into a graph embedding with attention mechanism for each graph, and then generate features for a given graph pair $(G^1, G^2)$ with their embeddings using a neural tensor network. The features are then used to predict the GED for regression. 
TaGSim~\cite{bai2021tagsim} categorizes the graph edit operations into different types and predicts the number of operations in each type more precisely. 
Noah~\cite{yang2021noah} applies the heuristic A*-beam algorithm~\cite{neuhaus2006A*beam} guided by a GNN model
called graph path network (GPN) to find small feasible edit paths. 
MATA*~\cite{liu2023mata} employs a structure-enhanced GNN to learn the differentiable top-$k$ candidate matching vertices which prunes the unpromising search directions of A*LSa~\cite{chang2020speeding} for approximate GED computation. 
Finally, GEDGNN~\cite{piao2023gedgnn} utilizes two separate cross-matrix modules to generate a cost matrix $\mathbf{A}_\text{cost}$ and a vertex-matching matrix $\mathbf{A}_\text{match}$, respectively, from GNN-extracted vertex features, where $\mathbf{A}_\text{match}$ is used for edit path generation, and both matrices are used to regress the GED. However, the correlation between $\mathbf{A}_\text{cost}$ and $\mathbf{A}_\text{match}$ is not captured, and $\mathbf{A}_\text{match}$ is directly used to fit the ground-truth vertex coupling relationship. In contrast, our GEDIOT model explicitly captures their correlation as $\mathbf{A}_\text{match}=\text{OT}(\mathbf{A}_\text{cost})$, where OT is our learnable Sinkhorn layer to be introduced in Section~\ref{sec:learnot} that ensures the matching constraints to be established in Eq.~\eqref{eq:constraint}.

\section{Theoretical Analysis}
\subsection{Sinkhorn Algorithm for OT}\label{app:sinkhorn}
In this section, we derive the Sinkhorn algorithm of optimal transport (OT) with Lagrange duality theory.

Recall that Sinkhorn is to solve the entropy relaxation of OT as specified in Eq.~\eqref{eq:eot}. The Lagrangian of Eq.~\eqref{eq:eot} can be written as
\begin{equation*}
    \begin{aligned}
        L(\bm{\pi},\bm{\alpha},\bm{\beta}) &=\sum_{i=1}^{n_1}\sum_{j=1}^{n_2}\left(
        \mathbf{C}_{i,j}\bm{\pi}_{i,j}+\bm{\pi}_{i,j}(\log\bm{\pi}_{i,j}-1)\right)\\
        &+ \sum_{i=1}^{n_1}\bm{\alpha}_i\left(\sum_{j=1}^{n_2}\bm{\pi}_{i,j}-\bm{\mu}_i\right)+\sum_{j=1}^{n_2}\bm{\beta}_j\left(\sum_{i=1}^{n_1}\bm{\pi}_{i,j}-\bm{\nu}_j\right)\\
        &=\left<\mathbf{C},\bm{\pi} \right>+\varepsilon H(\bm{\pi})+\left<\bm{\alpha},\bm{\pi}\bm{1}_{n_2}-\bm{\mu} \right>+\left<\bm{\beta},\bm{\pi}^\top\bm{1}_{n_1}-\bm{\nu} \right>
    \end{aligned}
\end{equation*}
Taking the derivative of the above Lagrangian with respect to $\bm{\pi}_{i,j}$ and setting it to zero, we get
\begin{equation*}
    \bm{\pi}_{i,j} = \bm{\varphi}_i\, \mathbf{K}_{i,j}\, \bm{\psi}_j, 
\end{equation*}
where $\mathbf{K}\in\mathbb{R}^{n_1\times n_2}$ is the kernel matrix with $\mathbf{K}_{i,j} = \exp \left( - \mathbf{C}_{i,j} / \varepsilon \right)$; $\bm\varphi \in\mathbb{R}^{n_1}$ and $\bm\psi \in \mathbb{R}^{n_2}$ are the dual variables with $\bm{\varphi}_{i} = \exp \left( - \bm{\alpha}_i / \varepsilon \right)$ and $\bm{\psi}_j = \exp \left( - \bm{\beta}_j / \varepsilon \right)$. 
Taking the derivatives of the Lagrangian with respect to $\bm{\alpha}_i$ and $\bm{\beta}_j$,  and setting them to zero, we obtain
\begin{equation*}
    \bm{\mu}_i = \bm{\varphi}_i \sum_j^{n_2} \mathbf{K}_{i,j} \bm{\psi}_j, \quad \bm{\nu}_j = \bm{\psi}_j \sum_{i}^{n_1} \mathbf{K}_{i,j} \bm{\varphi}_i. 
\end{equation*}
The Sinkhorn algorithm is to update the dual variables $\bm{\varphi}$ and $\bm{\psi}$ via the element-wise computation:
\begin{equation*}
\begin{aligned}
    &\bm\psi=\bm\nu\oslash(\mathbf{K}^\top\bm\varphi),\\
    &\bm\varphi=\bm\mu\oslash(\mathbf{K}~\bm\psi),
\end{aligned}
\end{equation*}
where the notation $\oslash$ is element-wise division. Note that the element $\mathbf{K}_{i,j}$ in $\mathbf{K}$ is strictly positive, and thus the denominators 
$\mathbf{K}^\top\bm\varphi$ and $\mathbf{K}~\bm\psi$ are always non-zero. 

\subsection{Error Analysis of GEDIOT}\label{app:error}
In this section, we analyze the solution of our proposed GEDIOT during training. 

The following theorem shows that in an ideal situation, the well-trained GEDIOT can output a coupling matrix that is the same as the ground truth node matching. 
\begin{theorem}
    There exists a cost matrix $\widehat{\mathbf{C}}^*$, such that the optimal coupling matrix $\widehat{\bm{\pi}}^*$ of the optimization problem 
    $$\min_{\bm{\pi} \in U(\bm{1}_{n_1},\bm{1}_{n_2})} \left<\widehat{\mathbf{C}}^*,\bm{\pi}\right>+\varepsilon \left< \bm{\pi}, \log\bm{\pi}-1\right>$$
    is exactly the ground truth node matching $\bm{\pi}^*$. 
\end{theorem}
\begin{proof}
    First, following Section~\ref{sec:learnot}, we add a dummy row and consider the optimization problem
    \begin{equation}
    \label{eq:eot-iot}
    \min_{\bm{\pi}\in\Pi(\widetilde{\bm\mu},\widetilde{\bm\nu})} \left<\widetilde{\mathbf{C}},\bm{\pi}\right>+\varepsilon \left< \bm{\pi}, \log\bm{\pi}-1\right>, 
    \end{equation}
    $$\Pi(\widetilde{\bm\mu},\widetilde{\bm\nu}) = \left\{ \bm{\pi}\in\mathbb{R}^{(n_1+1)\times n_2} \ | \ \bm{\pi} \bm{1}_{n_2} = \widetilde{\bm{\mu}},\ \ \bm{\pi}^\top \bm{1}_{n_1+1} = \widetilde{\bm{\nu}},\ \ \bm{\pi}\ge0\right\}.$$
    Given arbitrary $\mathbf{a} \in \mathbb{R}^{n_1+1}$ and $\mathbf{b} \in \mathbb{R}^{n_2}$, let 
    $$\widetilde{\mathbf{C}}_{i,j} = -\left( \mathbf{a}_i + \mathbf{b}_j + \varepsilon \log \bm{\pi}^*_{i,j} \right), \text{ for } i = 1, \cdots, n_1, j = 1, \cdots, n_2, $$
    and $\widetilde{\mathbf{C}}_{n_1+1, j} = 0$, for $j = 1, \cdots, n_2$. 
    The Lagrange duality of Eq.~\eqref{eq:eot-iot} is
    \begin{multline*}
        L(\bm{\pi},\bm{\alpha},\bm{\beta}) =\left<\widetilde{\mathbf{C}},\bm{\pi} \right>+\varepsilon \left< \bm{\pi}, \log\bm{\pi}-1\right> \\
        +\left<\bm{\alpha},\bm{\pi}\bm{1}_{n_2}-\widetilde{\bm{\mu}} \right>+\left<\bm{\beta},\bm{\pi}^\top\bm{1}_{n_1+1}-\widetilde{\bm{\nu}} \right>
    \end{multline*}
    Verifying the KKT condition~\cite{boyd2004convex}
    \begin{equation*}
    \begin{aligned}
        \frac{\partial L}{\partial \bm{\pi}_{i,j}} &= \widehat{\textbf{C}}_{i,j} + \varepsilon \log \bm{\pi}_{i,j} + \bm{\alpha}_i + \bm{\beta}_j \\
        &= -\left( \mathbf{a}_i + \mathbf{b}_j + \varepsilon \log \bm{\pi}^*_{i,j} \right) + \varepsilon \log \bm{\pi}_{i,j} + \bm{\alpha}_i + \bm{\beta}_j = 0, 
    \end{aligned}
    \end{equation*}
    Thus $\widehat{\bm{\pi}}^*=\bm{\pi}^*$, $\widehat{\bm{\alpha}}^*=\mathbf{a}$, and $\widehat{\bm{\beta}}^*=\mathbf{b}$ is a group of optimal solutions, where $\widehat{\bm{\pi}}^*$ is the optimal coupling matrix of Eq.~\eqref{eq:eot-iot} without the last row, and $\widehat{\bm{\alpha}}^*$ and $\widehat{\bm{\beta}}^*$ are corresponding optimal dual variables. Then, the first term of the objective function in the outer minimization in Eq.~\eqref{eq:iot-ged} reaches 0. 

    Particularly, according to Eq.~\eqref{eq:ged-match}, the approximate GED value is exactly the ground truth GED value when 
    \begin{equation}
    \label{eq:exact-cost}
        \widehat{\mathbf{C}}^* = \mathbf{M}+\frac{1}{2}\mathcal{L}(\mathbf{A}^1,\mathbf{A}^2)\otimes \bm{\pi}^*.
    \end{equation}
\end{proof}
Then based on the analysis in~\cite{li2019learning}, we show the relation between errors in the learned cost matrix and errors in the learned coupling matrix during training. 

\begin{theorem}
    We assume that the ground truth node-matching matrix is $\bm{\pi}^*$ and one of the corresponding cost matrices is $\mathbf{C}^*$ (defined in Eq.~\eqref{eq:exact-cost}). During training, the coupling matrix and cost matrix are denoted as $\widehat{\mathbf{C}}$ and $\widehat{\bm{\pi}}$ respectively. Let $\Delta \mathbf{C} = \mathbf{C}^* - \widehat{\mathbf{C}}$ and $\Delta \log\bm{\pi} = \log\bm{\pi}^* - \log\widehat{\bm{\pi}}$, then
    \begin{equation*}
    \begin{aligned}
        \| \Delta \mathbf{C} \|_{F} \geq \varepsilon^2 \left( \| \Delta \log\bm{\pi} \|_{F} \right) - \mathbf{f}^\top \mathbf{A}^\dagger \mathbf{f}, \\
        \| \Delta \log\bm{\pi} \|_{F} \geq \varepsilon^{-2} \left( \| \Delta \mathbf{C} \|_{F} - \mathbf{g}^\top \mathbf{A}^\dagger \mathbf{g} \right),
    \end{aligned}
    \end{equation*}
    where $\mathbf{A}=\left[\begin{array}{cc}n_2 \mathbf{I}_{n_1 \times n_1} & \mathbf{1}_{n_1}\mathbf{1}_{n_2}^\top \\ \mathbf{1}_{n_2} \mathbf{1}_{n_1}^\top & n_1 \mathbf{I}_{n_2 \times n_2}\end{array}\right]$, $\mathbf{A}^{\dagger}$ is the Moore-Penrose inverse of matrix $\mathbf{A}$, Frobenius norm $\|\mathbf{A}\|_F=\sqrt{\sum_{i=1}^{n_1} \sum_{j=1}^{n_2} \mathbf{A}_{i j}^2}$, $\mathbf{f}=\left[(\Delta \log \bm{\pi} \mathbf{1})^\top, \mathbf{1}^\top \left(\Delta \log \bm{\pi} \right)\right]^\top$, $\mathbf{g}=\left[(\Delta \mathbf{C} \mathbf{1})^\top, \mathbf{1}^\top \left(\Delta \mathbf{C} \right)\right]^\top$, and $\varepsilon$ is the regularization coefficient. 
\end{theorem}
\begin{proof}
    For the sake of simplicity, we assume a dummy row has already been added to the cost matrix. 
    According to the KKT condition, given the cost matrix $\mathbf{C}$ and the coupling matrix $\bm{\pi}$, there exist $\bm{\alpha}, \bm{\beta}$ such that
    \begin{equation*}
        \bm{\pi}_{i,j} = \exp \left( - \left( \mathbf{C}_{i,j} + \bm{\alpha}_i + \bm{\beta}_j \right) / \varepsilon \right). 
    \end{equation*}
    Thus, there exist $\bm{\alpha}^*, \bm{\beta}^*$ and $\widehat{\bm{\alpha}}, \widehat{\bm{\beta}}$ such that
    \begin{equation*}
    \begin{aligned}
        \mathbf{C}^*_{i,j} &= -\varepsilon \log \bm{\pi}^*_{i,j} - \bm{\alpha}_i - \bm{\beta}_j, \\
        \widehat{\mathbf{C}}_{i,j} &= -\varepsilon \log \widehat{\bm{\pi}}_{i,j} - \widehat{\bm{\alpha}}_i - \widehat{\bm{\beta}}_j. 
    \end{aligned}
    \end{equation*}
    Let $\Delta \bm{\alpha} = \bm{\alpha}^* - \widehat{\bm{\alpha}}, \Delta \bm{\beta} = \bm{\beta}^* - \widehat{\bm{\beta}}$, and we have
    \begin{equation*}
        \Delta \mathbf{C}_{i,j} = - \varepsilon \Delta \bm{\pi}_{i,j} - \Delta \bm{\alpha}_i - \Delta \bm{\beta}_j. 
    \end{equation*}
    Viewing $\Delta \bm{\alpha}, \Delta \bm{\beta}$ as variables and taking the minimum value of the right-hand side according to Lemma 3 in~\cite{li2019learning}, it follows
    \begin{equation*}
        \| \Delta \mathbf{C} \|_{F} \geq \varepsilon^2 \left( \| \Delta \log\bm{\pi} \|_{F} \right) - \mathbf{f}^\top \mathbf{A}^\dagger \mathbf{f}. 
    \end{equation*}
    Similarly, consider
    \begin{equation*}
        \begin{aligned}
            \log \bm{\pi}^*_{i,j} &= -\varepsilon^{-1} \left( \mathbf{C}^*_{i,j} + \bm{\alpha}_i + \bm{\beta}_j \right), \\
        \log \widehat{\bm{\pi}}_{i,j} &= -\varepsilon^{-1} \left( \widehat{\mathbf{C}}_{i,j} + \widehat{\bm{\alpha}}_i + \widehat{\bm{\beta}}_j \right), 
        \end{aligned}
    \end{equation*}
    and we have
    \begin{equation*}
        \| \Delta \log\bm{\pi} \|_{F} \geq \varepsilon^{-2} \left( \| \Delta \mathbf{C} \|_{F} - \mathbf{g}^\top \mathbf{A}^\dagger \mathbf{g} \right). 
    \end{equation*}
\end{proof}

Moreover, we derive a bound for the gap between the approximate GED and the exact GED. 
\begin{theorem}
    Given the ground-truth node-matching matrix $\bm{\pi}^*$, its corresponding cost matrix $\mathbf{C}^*$, and the learned coupling matrix and cost matrix $\widehat{\bm{\pi}}$ and $\widehat{\mathbf{C}}$, the gap between the approximate GED value $\widehat{GED}$ and the exact GED value $GED^*$ is bounded by
    \begin{equation*}
        n \| \Delta \mathbf{C} \|_{F} + \| \mathbf{C}^* \|_{F} \| \Delta \bm{\pi} \|_{F}, 
    \end{equation*}
    where $n = \max \{ n_1, n_2 \}$, $\Delta \mathbf{C} = \mathbf{C}^* - \widehat{\mathbf{C}}$, and $\Delta \bm{\pi} = \bm{\pi}^* - \widehat{\bm{\pi}}$. 
\end{theorem}
\begin{proof}
    Considering that $\mathbf{C}^* = \mathbf{M}+\frac{1}{2}\mathcal{L}(\mathbf{A}^1,\mathbf{A}^2)\otimes \bm{\pi}^*$ according to Eq.~\eqref{eq:ged-match}, we analyze
    \begin{equation*}
    \begin{aligned}
        | \widehat{GED} - GED^*| &= \left| \left< \widehat{\mathbf{C}}, \widehat{\bm{\pi}} \right> - \left< \mathbf{C}^*, \bm{\pi}^* \right> \right| \\
        &= \left| \left< \widehat{\mathbf{C}} - \mathbf{C}^*, \widehat{\bm{\pi}} \right> + \left< \mathbf{C}^*, \widehat{\bm{\pi}} - \bm{\pi}^* \right> \right| \\
        &\leq \| \widehat{\mathbf{C}} - \mathbf{C}^* \|_F \| \widehat{\bm{\pi}} \|_{F} + \| \mathbf{C}^* \|_{F} \| \widehat{\bm{\pi}} - \bm{\pi}^* \|_{F} \\
        &\leq n \| \Delta \mathbf{C} \|_{F} + \| \mathbf{C}^* \|_{F} \| \Delta \bm{\pi} \|_{F}. 
    \end{aligned}
    \end{equation*}
    The first ``$\leq$'' is derived from the Cauchy-Schwarz inequality and the second is based on the fact that $\bm{\pi}_{i,j} \leq 1$. 
\end{proof}

\begin{figure*}
  \centering
  \includegraphics[width=0.9\linewidth]{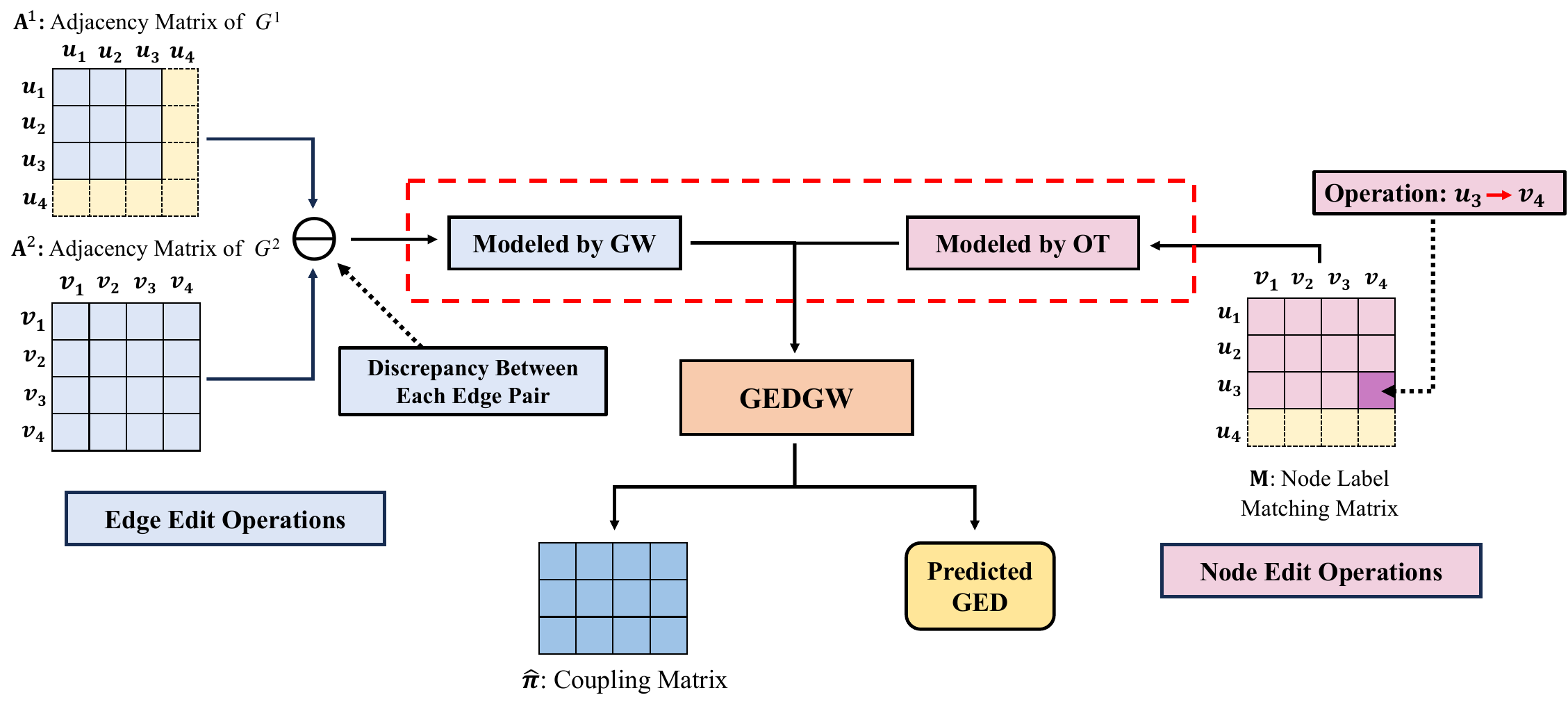}
  \vspace{-2.5mm}
  \caption{Diagram of the Proposed GEDGW}
  \label{fig:diagram_gedgw}
  \vspace{-1mm}
\end{figure*}
\setlength{\textfloatsep}{5pt}

\subsection{Diagram of GEDGW}\label{app:diagram_gedgw}
{
We present the diagram of our unsupervised method GEDGW in Figure~\ref{fig:diagram_gedgw} based on the graphs $G^1$ and $G^2$ in Figure~\ref{fig:matching}. Note that $G^1$ has 3 nodes while $G^2$ has 4. We add a dummy node $u_4$ in $G^1$ so that the two graphs have the same number of nodes, and the elements in the matrices in Figure~\ref{fig:diagram_gedgw} corresponding to the dummy node are represented by dashed lines. 
GEDGW formulates the GED computation as an optimization problem in Eq.~\eqref{eq:ged-fgw} related to node matching, since GED can be obtained according to node matching. 
It first divides editing operations into two categories: the edge edit operations and the node edit operations, and two terms $\frac{1}{2}\sum_{i,j,k,l} (\mathbf{A}^1_{i,j}-\mathbf{A}^2_{k,l})^2 \bm{\pi}_{i, k} \bm{\pi}_{j, l}$ and $\sum_{i,k}\mathbf{M}_{i,k}\bm{\pi}_{i,k}$ in the objective function of the optimization problem measure the two types of edit operations, respectively. 
As shown in the left part of Figure~\ref{fig:diagram_gedgw}, matrices $\mathbf{A}^1\in\{0,1\}^{4\times 4}$ and $\mathbf{A}^2\in\{0,1\}^{4\times 4}$ are the adjacency matrices of $G^1$ and $G^2$, respectively. 
As illustrated in the right part of Figure~\ref{fig:diagram_gedgw}, $\mathbf{M}\in\{0,1\}^{4\times 4}$ is the node label matching matrix between nodes of $G^1$ and $G^2$, where $\mathbf{M}_{i,k}=1$ if nodes $u_i\in V^1$ and $v_k\in V^2$ have the same label; otherwise $\mathbf{M}_{i,k}=0$. 

More concretely, each element $\left((\mathbf{A}^1_{i,j}-\mathbf{A}^2_{k,l})^2\right)_{i,j,k,l}$ in the 4-th order tensor indicates the discrepancy between every two edges $(u_i, u_j)\in E^1$ and $(v_k, v_l)\in E^2$. Subsequently, $(\mathbf{A}^1_{i,j}-\mathbf{A}^2_{k,l})^2\bm{\pi}_{i,k}\bm{\pi}_{j,l}$ measures the cost of the edge edit operations including edge insertion/deletion, since it represents whether edge $(u_i, u_j)\in E^1$ and edge $(v_k, v_l)\in E^2$ exist simultaneously when $u_i$ matches $v_k$ and $u_j$ matches $v_l$. We model it as Gromov-Wasserstein Discrepancy~(GW) (the left part of Figure~\ref{fig:diagram_gedgw}). 

Each element $\mathbf{M}_{i,k}\bm{\pi}_{i,k}$ measures the cost of the node edit operations including node relabeling and node insertion/deletion, since it represents matching a node in $G^2$ to a node in $G^1$ with a different label. We model it as Optimal Transport~(OT) (the right part of Figure~\ref{fig:diagram_gedgw}). 

Then, we combine GW and OT to compute GED between $G^1$ and $G^2$ (marked with the red dashed frame) and output the approximate GED and the coupling matrix $\bm{\pi}$ for GEP generation as shown in the lower part of Figure~\ref{fig:diagram_gedgw}.}

\begin{algorithm}[!t]
\DontPrintSemicolon
    \KwIn{graphs $G^1$, $G^2$}
    Compute $\mathbf{M}$ via the labels of nodes between $G^1$ and $G^2$\; 
    \For{$k=1,2,\dots$}{\label{alg1:line2}
        $\mathbf{G}^{(k)}\leftarrow$ compute based on Eq.~\eqref{eq:cg}\;\label{alg1:line3}
        $\bm{\tilde{\pi}}^{(k)}\leftarrow \argmin\limits_{\bm{\pi}\in\Pi(\bm{1}_n,\bm{1}_n)}\left<\mathbf{G}^{(k)},\bm{\pi}\right>$\;\label{alg1:line4}
        $\gamma^{(k)}\leftarrow$ line search to find the optimal step size \;\label{alg1:line5}
        $\bm{\pi}^{(k)}\leftarrow(1-\gamma^{(k)})\cdot\bm{\pi}^{(k-1)}+\gamma^{(k)}\cdot\bm{\tilde{\pi}}^{(k)}$\;\label{alg1:line6}
    }
    $\widehat{\bm{\pi}}\leftarrow\bm{\pi}^{(k)}$\;\label{alg1:line7}
    $\widehat{GED}\leftarrow \left<\widehat{\bm{\pi}},\mathbf{M} \right>+\dfrac{1}{2}\left<\widehat{\bm{\pi}},\mathcal{L}(\mathbf{A}^1,\mathbf{A}^2)\otimes\widehat{\bm{\pi}}\right> $\;\label{alg1:line8}
    \Return $\widehat{GED}$, $\widehat{\bm{\pi}}$\;\label{alg1:line9}
\caption{Conditional gradient algorithm for GEDGW}
\label{algo:cg}
\end{algorithm}

\subsection{Conditional Gradient Method}
\label{app:cg}
We solve Eq.~\eqref{eq:ged-fgw} formulated in Section~\ref{sec:eq_for_gedgw} to compute GED estimate $\widehat{GED}$ and the coupling matrix $\widehat{\pi}$ using the Conditional Gradient (CG) method.
The main idea of CG method is to solve a linear approximate subproblem repeatedly and improve a solution within a feasible region. The key advantage is that it only requires solving a simpler linear subproblem at each iteration, which can be computationally efficient. 
The pseudo-code is presented in Algorithm~\ref{algo:cg}. 

At each iteration $k$, it first computes the gradient $\mathbf{G}^{(k)}$ with the current coupling matrix $\bm{\pi}^{(k-1)}$ (Line~\ref{alg1:line3}) by the following equation:
\begin{equation}
    \label{eq:cg}
    \mathbf{G}^{(k)}\leftarrow\mathbf{M}+\frac{1}{2}\mathcal{L}(\mathbf{A}^1,\mathbf{A}^2)\otimes \bm{\pi}^{(k-1)}.
\end{equation}
The descent direction $\bm{\tilde{\pi}}^{(k)}$ is obtained by solving an OT problem with $\mathbf{G}^{(k)}$ as the cost matrix over the set $\Pi(\bm{1}_{n},\bm{1}_{n})$ (Line~\ref{alg1:line4}). 
Then the step size $\gamma^{(k)}$ in the line search is determined (Line~\ref{alg1:line5}) according to the constrained minimization of a second-order polynomial:
\begin{align}
\argmin_{\gamma\in[0,1]}&~\left<\bm{\bar{\pi}}^{(k)},\mathbf{M} \right>+\dfrac{1}{2}\left<\mathcal{L}(\mathbf{A}^1,\mathbf{A}^2)\otimes\bm{\bar{\pi}}^{(k)},\bm{\bar{\pi}}^{(k)}\right>\\
&\text{where }\bm{\bar{\pi}}^{(k)}=(1-\gamma)\cdot\bm{\pi}^{(k-1)}+\gamma\cdot\bm{\tilde{\pi}}^{(k)} \notag
\end{align}
More details of the line-search algorithm can be found in~\cite{titouan2019optimal,chapel2020partial}.
The transport plan $\bm{\pi}^{(k)}$ is then updated for next iteration (Line~\ref{alg1:line6}).


Finally, it outputs GED estimate $\widehat{GED}$ and the coupling matrix $\widehat{\bm{\pi}}$, calculated in Lines~\ref{alg1:line7}-\ref{alg1:line8}. Moreover, $\widehat{\bm{\pi}}$ can be used for GED generation with the same $k$-best matching framework discussed in Section~\ref{sec:gep}. 
\begin{algorithm}[!t]
\DontPrintSemicolon
    \KwIn{graphs $G^1=(V^1,E^1,L^1)$, $G^2=(V^2,E^2,L^2)$, \\ \qquad\quad node matching $\mathbf{M}\in\{0,1\}^{n_1\times n_2}$}
    EPath = []\;
    Generate the node mapping $f: V^1\to V^2$ and $\qquad \qquad$ inverse mapping $f^{-}: V^2\to V^1$ from $\mathbf{M}$\; \label{algo:path_line2}
    \tcp{Node Relabeling}
    \ForEach{$\mathrm{node} \ u\in V^1$}{\label{algo:path_line3}
        \If{$L^1(u)\neq L^2(f(u))$}{\label{algo:path_line4}
            EPath.append(Relabel $u$ with $L^2(f(u))$)\;\label{algo:path_line5}
        }
    }
    \tcp{Node Insertion}
    \ForEach{$\mathrm{node} \ v\in V^2\backslash f(V^1)$}{\label{algo:path_line6}
        EPath.append(Insert a node with label $L^2(v)$ in $G^1$)\;\label{algo:path_line7}
    }
    \tcp{Edge Deletion}
    \ForEach{$\mathrm{edge} \ (u,u')\in E^1$}{\label{algo:path_line8}
        \If{$(f(u),f(u'))\notin E^2$}{\label{algo:path_line9}
            EPath.append(Delete edge $(u_1,u_2)$ from $G^1$)\;\label{algo:path_line10}
        }
    }
    \tcp{Edge Insertion}
    \ForEach{$\mathrm{edge} \ (v,v')\in E^2$}{\label{algo:path_line11}
        \If{$(f^-(v),f^-(v'))\notin E^1$}{\label{algo:path_line12}
            EPath.append(Insert edge $(f^-(v),f^-(v'))$ in $G^1$)\;\label{algo:path_line13}
        }
    }
    \Return EPath \tcp{edit path that transforms $G^1$ to $G^2$}
\caption{{Edit Path Generation (EPGen)}}
\label{algo:path}
\end{algorithm}

\begin{algorithm}[!t]
\DontPrintSemicolon
    \KwIn{graphs $G^1=(V^1,E^1)$, $G^2=(V^2,E^2)$, coupling matrix $\bm{\pi}$, $k$}
    Construct bipartite graph $G = (V^1, V^2, V^1\times V^2, \bm{\pi})$ \; \label{algo:kbest_line1}
    $BestPath \leftarrow None$; \  \ 
    $S_1 \leftarrow \left\{ M \mid M \text{ is a node matching} \right\} $\;\label{algo:kbest_line2}
    $M_1(S_1) \leftarrow$ \textsf{BestMatch}($S_1$)\;\label{algo:kbest_line3}
    $M_2(S_1) \leftarrow$ \textsf{SecondBestMatch}($S_1$)\;\label{algo:kbest_line4}
    $LB(S_1) \leftarrow \textsf{GEDLowerBound}(S_1)$\;\label{algo:kbest_line5}
    \textsf{Update}($BestPath$, \textsf{EPGen}($M_1(S_1))$) \;\label{algo:kbest_line6}
    \textsf{Update}($BestPath$, \textsf{EPGen}($M_2(S_1))$) \;\label{algo:kbest_line7}
    \For{$t = 2$ to $k$}{\label{algo:kbest_line8}
        $id \leftarrow None$, $max\_weight \leftarrow -\infty$\;\label{algo:kbest_line9}
        \For{$S_i \in \{ S_1, \cdots, S_{t-1}\}$}{\label{algo:kbest_line10}
            \If{$LB(S_i) < len(BestPath)$}{\label{algo:kbest_line11}
                $weight \leftarrow \langle \bm{\pi}, M_2(S_i) \rangle$\;\label{algo:kbest_line12}
                \If{$weight > max\_weight$}{\label{algo:kbest_line13}
                    $(id, max\_weight) \leftarrow (i, weight)$\;\label{algo:kbest_line14}
                }
            }
        }
        $(S_{id}, S_t) \gets\textsf{SpaceSplit}(G, S_{id})$\;\label{algo:kbest_line15}
        $LB(S_{id}) \leftarrow \textsf{GEDLowerBound}(S_{id})$\;\label{algo:kbest_line16}
        \textsf{Update}($BestPath$, \textsf{EPGen}($M_2(S_{id}))$) \;\label{algo:kbest_line17}
        \textsf{Update}($BestPath$, \textsf{EPGen}($M_2(S_t))$) \;\label{algo:kbest_line18}
    }
    \KwRet{$BestPath$}\label{algo:kbest_line19}
    
    \SetKwFunction{FMyFunction}{SpaceSplit}
    \SetKwProg{Fn}{Function}{:}{}
    \Fn{\FMyFunction{$G, S$}}{\label{algo:kbest_line20}
        Choose an arbitrary edge $e \in M_1(S)$ but $e \notin M_2(S)$\;\label{algo:kbest_line21}
        $S^{\prime} = \left\{ M \in S \mid e \in M \right\}$, $S^{\prime\prime} = \left\{ M \in S \mid e \notin M \right\}$\;\label{algo:kbest_line22}
        $M_1(S^{\prime}) \leftarrow M_1(S)$, $M_2(S^{\prime}) \leftarrow$ \textsf{SecondBestMatch}($S^{\prime}$)\;\label{algo:kbest_line23}
        $M_1(S^{\prime\prime}) \leftarrow M_2(S)$, $M_2(S^{\prime\prime}) \leftarrow$ \textsf{SecondBestMatch}($S^{\prime\prime}$)\;\label{algo:kbest_line24}
        $LB(S^{\prime\prime}) \leftarrow LB(S)$\;\label{algo:kbest_line25} 
        \KwRet{$S^{\prime}, S^{\prime\prime}$}\label{algo:kbest_line26}
    }
\caption{{$k$-best Matching Framework}}
\label{algo:kbest}
\end{algorithm}

\section{$k$-Best Matching} \label{app:kbest}
{In this section, we provide the pseudocode of $k$-best matching framework that combines the label set based lower bound of GED and space splitting techniques. Algorithm~\ref{algo:kbest} obtains the top-$k$ best node matchings according to the length of their corresponding edit paths. 

We begin by presenting a formal description of how to generate an edit path from a node matching between $G^1$ and $G^2$.
The edit path generation procedure is shown in function EPGen($\cdot$) in Algorithm~\ref{algo:path}. With a given node matching $\mathbf{M}$ between $G^1$ and $G^2$, we first denote the node mapping as $f: V^1\to V^2$ and the corresponding inverse mapping as $f^-: V^2\to V^1$ (Line~\ref{algo:path_line2}), where for $u\in V^1$ and $v\in V^2$, $f(u)=v$ and $f^-(v)=u$ if and only if $\mathbf{M}_{u,v}=1$. The edit operations can be categorized into four types: node relabeling (Lines \ref{algo:path_line3}-\ref{algo:path_line5}), node insertion (Lines \ref{algo:path_line6}-\ref{algo:path_line7}), edge deletion (Lines \ref{algo:path_line8}-\ref{algo:path_line10}), and edge insertion (Lines \ref{algo:path_line11}-\ref{algo:path_line13}). For the two types of node edit operations, the algorithm checks whether node $u$ in $G^1$ has a corresponding node $f(u)$ in $G^2$, and (if $f(u)$ exists) whether $u$ and $f(u)$ have the same label.
For each edge $(u, u')$ in $G^1$, the algorithm checks whether the corresponding $(f(u), f(u'))$ in $G^2$ exist. If $(f(u), f(u'))$ does not exist, an edge deletion operation is needed. Similarly, for each edge $(v, v')$ in $G^2$, it checks whether the corresponding $(f^-(v), f^-(v'))$ in $G^1$ exist. If $(f^-(v), f^-(v'))$ does not exist, an edge insertion operation is needed. 

Then we introduce the label set based GED lower bound~\cite{chang2020speeding}, which can be calculated in linear time and prune out unnecessary node matchings in $k$-best matching framework. It is formulated as: 
\begin{equation}
\label{eq:gedlowerbound}
\operatorname{GEDLB}\left(G^1, G^2\right)=\left|L\left(V^1\right) \oplus L\left(V^2\right)\right|+\left|\left|E^1\right|-\left|E^2\right|\right|
\end{equation}
where $L\left(V^1\right)$ and $L\left(V^2\right)$ denote the multi-set of node labels of $G^1$ and $G^2$ respectively, and $\oplus$ denotes a multi-set function that $A \oplus B=$ $A \cup B-A \cap B$. 

Now, we explain the $k$-best matching framework in Algorithm~\ref{algo:kbest}. Line~\ref{algo:kbest_line1} construct a weighted complete bipartite graph between $V^1$ and $V^2$, where the weight of edge $(u,v)$ ($u\in V^1,\ v\in V^2$) is $\bm{\pi}_{u,v}$. We also define the weight of a node matching $M$ as the Frobenius product of $\bm{\pi}$ and $M$ (i.e., $\left\langle \bm{\pi}, M \right\rangle$). Lines \ref{algo:kbest_line2}-\ref{algo:kbest_line7} initialize the first solution subspace $S_1$, where $M_1(S_1)$ and $M_2(S_1)$ denote the best and second-best node matchings in $S_1$ respectively, which can be found in $O\left(n^3\right)$ time by classical algorithms~\cite{chegireddy1987algorithms}. The function GEDLowerBound($\cdot$) in Lines~\ref{algo:kbest_line5} calculates the label-set-based GED lower bound via Eq.~\eqref{eq:gedlowerbound}.
In Lines \ref{algo:kbest_line6}-\ref{algo:kbest_line7}, Update($\cdot$) means replacing the current best solution BestPath by the edit path output from EPGen($\cdot$) if BestPath is None or that path is shorter.

Lines~\ref{algo:kbest_line8}-\ref{algo:kbest_line26} show the iterative space-splitting method. Suppose that there are $(t-1)$ subspaces, and each subspace has its own best and second-best node matching $M_1(S_i)$ and $M_2(S_i)$, 
we choose the subspace where the second-best node matching has the maximum weight among all the subspaces for further splitting (Lines~\ref{algo:kbest_line9}-\ref{algo:kbest_line14}). 
If the GED lower bound of a subspace $S$ is greater or equal to the length of the current best path, it is unpromising, so there is no need to further split $S$ (Line~\ref{algo:kbest_line11}). Then, we split the chosen subspace $S_{id}$ and update the GED lower bound and best path of the new subspaces (Lines~\ref{algo:kbest_line15}-\ref{algo:kbest_line18}). 

Lines~\ref{algo:kbest_line20}-\ref{algo:kbest_line26} specify the SpaceSplit($\cdot$) function using Line~\ref{algo:kbest_line15},
which splits $S$ into two subspaces $S^{\prime}$ and $S^{\prime\prime}$, such that a node matching of $S$ is in $S^{\prime}$ if it contains $e$, and otherwise it is in $S^{\prime\prime}$ (Lines \ref{algo:kbest_line21}-\ref{algo:kbest_line22}). Note that $M_1(S)$ (resp.\ $M_2(S)$) becomes the best node matching in $S^{\prime}$ (resp.\ $S^{\prime\prime}$) after splitting (Lines \ref{algo:kbest_line23}-\ref{algo:kbest_line24}).
The entire node matching space is partitioned by repeatedly selecting a subspace to split in this manner. This process is repeated until $k$ subspaces are reached. Finally, $2k$ node matchings (2 from each subspace) are collected as the candidate set to find the shortest edit path. More details can be found in Section~4 in~\cite{piao2023gedgnn}. 
}

\begin{figure*}[t]
    \centering
    \includegraphics[width=\linewidth]{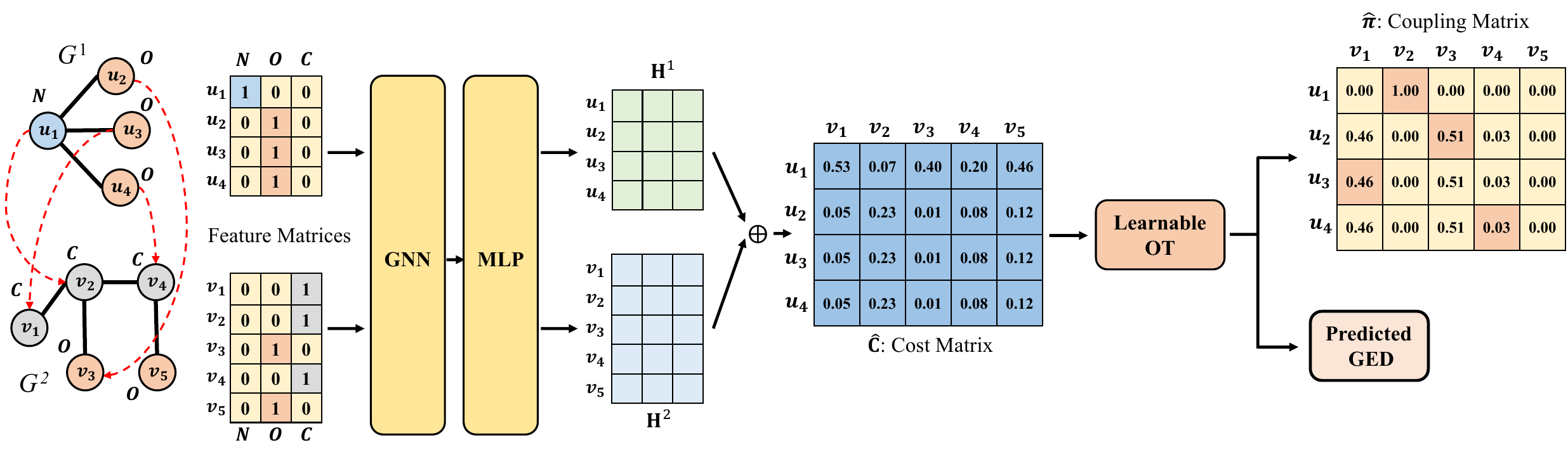}
    \caption{A Case Study for GEDIOT}
    \label{fig:case-study}
\end{figure*}
\setlength{\textfloatsep}{5pt}

\begin{figure*}[t]
    \centering
    \includegraphics[width=0.65\linewidth]{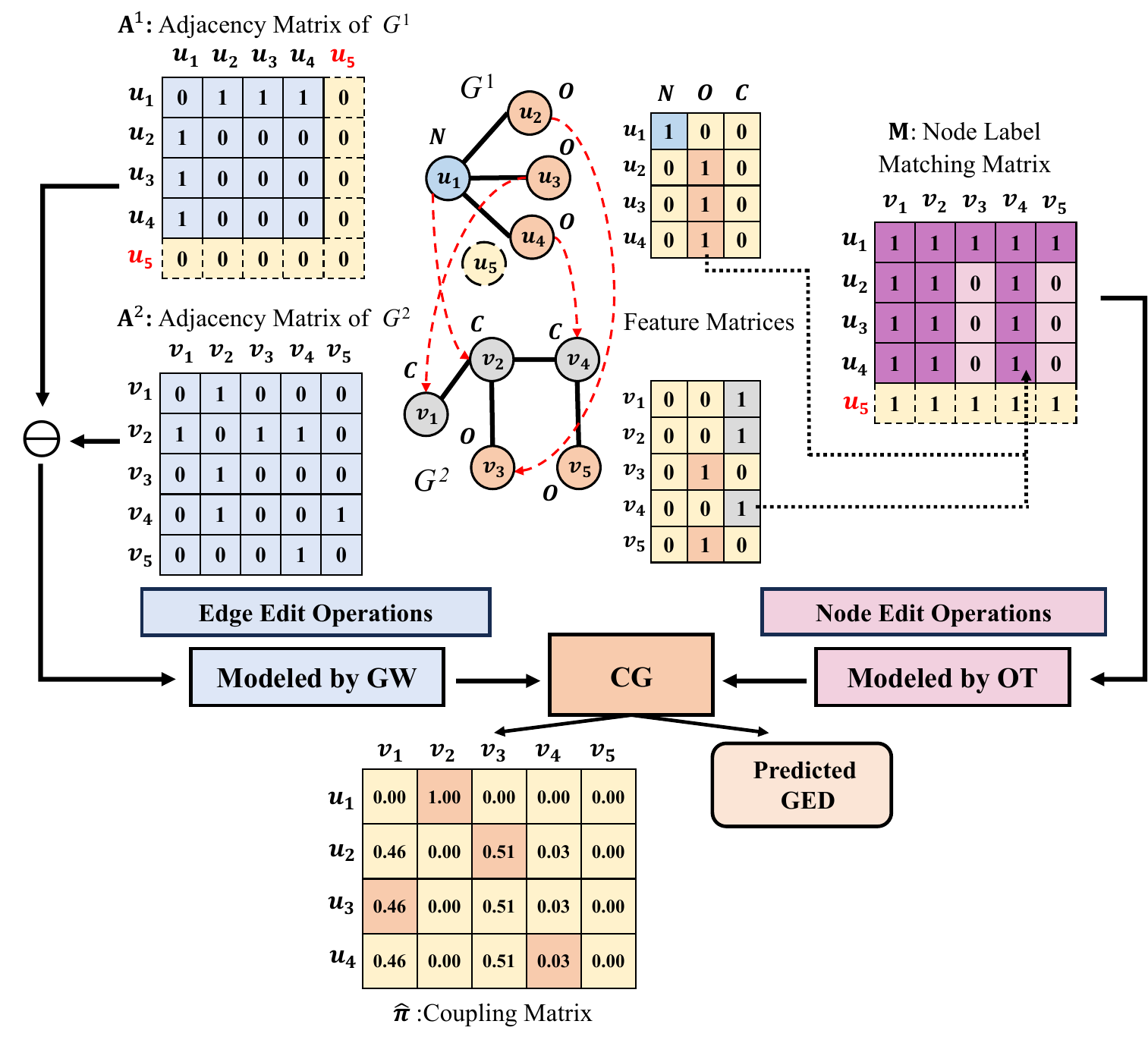}
    \vspace{-2mm}
    \caption{A Case Study for GEDGW}
    \label{fig:case-study-gw}
\end{figure*}
\setlength{\textfloatsep}{5pt}

\section{Case Study} \label{app:case-study}
{We conduct a case study of GED computation between a 4-node $G^1$ and a 5-node $G^2$ from AIDS by our proposed GEDIOT in Figure~\ref{fig:case-study}. The graphs are converted from the chemical compounds where nodes and edges represent the atoms and covalent bonds, respectively.   The color of a node indicates its label (i.e., the type of atoms). $G^1$ contains a Nitrogen (i.e., $N$) atom and three Oxygen (i.e., $O$) atoms, and $G^2$ contains three Carbon atoms (i.e., $C$)  and two Oxygen atoms. The ground-truth node-matching is shown by the dashed red lines (e.g., $u_1$ in $G^1$ corresponds to $v_2$ in $G^2$). The $i\textsuperscript{th}$ row of the feature matrix is the one-hot encoding of the label of node $u_i$ (or $v_i$).  We initialize the node embedding as the feature matrix and obtain the final embedding from GNN and MLP modules (i.e., node embedding component). Given two node embedding matrices obtained from $G^1$ and $G^2$, the pairwise scoring operation~$\oplus$ returns a cost matrix (discrepancy matrix) $\mathbf{C}$ where element $\mathbf{C}_{i,j}$ is the pairwise score computed from the embeddings of node~$u_i$ in $G^1$ and node~$v_j$ in $G^2$, which is illustrated in Figure~\ref{fig:difference}. We can see that in the cost matrix, the cost between node $u_1$ in $G^1$ and node $v_2$ in $G^2$ is much smaller than the cost between $u_1$ and other nodes in $G^2$, which is consistent with the fact that $u_1$ and $v_2$ are similar (e.g., their degrees are both $3$). 
Note that the numbers of nodes in the two graphs are different, and we extend the cost matrix with a dummy row filled with $0$ and redefine mass distributions $\widetilde{\bm\mu}$ and $\widetilde{\bm\nu}$ according to Section~\ref{sec:learnot}. 
Then, the cost matrix is fed into the learnable OT component to seek a global decision that minimizes the total cost of transporting masses from nodes of $G^1$ to nodes of $G^2$. The OT component outputs a coupling matrix that fits the ground-truth node-matching matrix for GED computation and GEP generation. 
Each row in the coupling matrix is a probability vector for $u_i$, representing the probability of matching $u_i$ to each $v_j$. Elements $\bm{\pi}_{ij}$ in the coupling matrix highlighted in the darker color in Figure~\ref{fig:case-study} correspond to the non-zero elements $\bm{\pi}^{*}_{ij}$ (i.e., $u_i$ in $G^1$ matches $v_j$ in $G^2$) in the ground-truth node-matching. }

{In Figure~\ref{fig:case-study-gw}, we also present a case study for GEDGW with the same graphs as Figure~\ref{fig:case-study}. We first construct a binary node label matching matrix $\mathbf{M}$, where each element $\mathbf{M}_{ij}$ is $0$ if and only if $u_i$ and $v_j$ have the same label. Noticing that $|V^1| < |V^2|$, we add a dummy node $u_5$ in $G^1$ so that the two graphs have the same number of nodes. The node label matching matrix $\mathbf{M}$ and the two extended adjacency matrices $\mathbf{A}^1$ and $\mathbf{A}^2$ are used to model node and edge operations in the optimization problem of GEDGW, which is then solved with the Conditional Gradient~(CG) method. Same as GEDIOT, GEDGW also outputs a predicted GED and a coupling matrix that fits the node-matching matrix. }

\section{Time Complexity Analysis}
\label{app:time2}
In this section, we provide a comprehensive analysis of the time complexity of our proposed methods. 
\subsection{Time Complexity of GEDIOT}
As the model training can be done offline, we consider the computation cost of the forward
propagation for GEDIOT. For ease of description, we assume that the number of GNN layers is $N$, the dimension of hidden layers of GNN and MLP is $d$, and the output dimension of NTN is $L$. The dimension $D$ of the input $\mathbf{h}$ of MLP is $(N+1)d$ since it is the concatenation of the output of each GNN layer and the initial node features. Let $n=n_2$, $m=\max(m_1,m_2)$ for the given graph pair $(G^1, G^2)$ and $M$ be the number of iterations of the Sinkhorn algorithm. 
Note that for two matrices $\mathbf{A} \in \mathbb{R}^{p \times q}$ and $\mathbf{B} \in \mathbb{R}^{q \times r}$, the time complexity of matrix multiplication $\mathbf{A}\mathbf{B}$ is $O(pqr)$, which we will use without mentioning again in the following analysis. 
We introduce the computation cost of all modules and sum them up to get the total cost.

In the node embedding component, in each layer of GNN, the aggregation of node features from every neighbor of GNN takes $O(md)$ time, and the linear transformation of the features of each node consumes a total of $O(nd^2)$ time. Therefore, GNN takes $O(N(md+nd^2))$ time to generate the node embedding $\mathbf{h}$. To obtain the final node embedding $\mathbf{H}$, the three-layer MLP module requires a total of $O(n((N+1)d)^2)$ time for the transformation. The computation cost of the node embedding component is therefore bounded by $O(N(md+nd^2)+n((N+1)d)^2)$.

In the graph discrepancy component, the node attentive mechanism costs $O(nd+d^2)$ time to generate the graph-level embedding $\mathbf{H}_G$. Then, it takes $O(Ld^2)$ time to compute the interaction vector $\mathbf{s}(G^1,G^2)$. The fully connected neural networks consume $O(Ld^2)$ to obtain the predicted score. The computation cost of the neural tensor network is bounded by $O(Ld+d^2)$. The computation cost of the graph discrepancy component is therefore bounded by $O(nd+Ld^2)$. 

As for the learnable OT component, it first computes the cost matrix with the two final node embeddings $\mathbf{H}^1$ and $\mathbf{H}^2$, which takes $O(nd^2+n^2d)$ time. Subsequently, the Sinkhorn layer runs Algorithm~\ref{algo:sinkhorn} for $M$ iterations, with each iteration requiring $O(n^2)$ time. This results in $O(Mn^2)$ cost in total for the Sinkhorn layer. The computation cost of the learnable OT component is therefore bounded by $O(nd^2+n^2d+Mn^2)$. 

Combining all the costs, the time complexity for the forward
propagation of GEDIOT is 
$$O\left(N(md+nd^2+nN^2d^2)+Ld^2+nd^2+n^2d+Mn^2\right).$$
Since $m=O(n^2)$ and $N$, $L$, $d$ and $M$ are fixed in GEDIOT, it can be simplified to $O(n^2)$, which is related to the size of the input graph.

For GEP generation, the two main steps of the $k$-best matching framework are finding the best node matching and edit path generation via this node matching, which are repeated $k$ times to find the best edit path. The first task takes $O(n^3)$ time to find the maximum node matching~\cite{piao2023gedgnn}. Recall that we can generate an edit path by traversing all vertices and edges with a given node matching, which takes $O(m+n)$ time. In total, the time complexity of GEP generation is therefore $O(k(m+n+n^3))=O(kn^3)$.

\subsection{Time Complexity of GEDGW and GEDHOT}
For GEDGW, we use the CG method~\cite{braun2022conditional,vincent2021semi} (see Algortihm~\ref{algo:cg} in  Appendix~\ref{app:cg} for details) to solve Eq.~\eqref{eq:ged-fgw}. The main time cost in each iteration of Algorithm~\ref{algo:cg} lies in the tensor product $\mathcal{L}(\mathbf{A}^1,\mathbf{A}^2)\otimes\bm{\pi}$ as shown in Eq.~\eqref{eq:cg}. Directly computing it takes $O(n^4)$ time, but according to Proposition~1 in~\cite{peyre2016gromov}, its computation can be accelerated to $O(n^3)$ time by decomposing $\mathcal{L}(\mathbf{A}^1,\mathbf{A}^2)\otimes\bm{\pi}$ into multiple matrix multiplications. Therefore, the total time complexity of Algorithm~\ref{algo:cg} is bounded by $O(Kn^3)$, where $K$ is the number of iterations.

For the process of GEDHOT, the two methods GEDIOT and GEDGW are called separately. Recall that the time complexity of the forward propagation of GEDIOT and GEDGW in Algorithm~\ref{algo:cg} are $O(n^2)$ and $O(Kn^3)$, respectively, where $K$ is the number of iterations of GW computation. Therefore, the time complexity of GEDHOT to approximate GED is bounded by $O(n^2+Kn^3)\approx O(Kn^3)$. Since the time complexity to generate GEP using the $k$-best matching framework is $O(kn^3)$, the total time of predicting both GED and GEP is bounded by $O((K+k)n^3)$.

\section{Experimental Setting}
\subsection{Datasets}\label{app:data}
We use three real-world graph datasets: AIDS, Linux, and IMDB.

\vspace{1mm}
\noindent\textbf{AIDS. }The AIDS dataset consists of chemical compounds from the Developmental Therapeutics Program at NCI/NIH. The chemical compounds are converted into graphs where nodes and edges represent the atoms and covalent bonds, respectively. Each node is labeled with one chemical symbol, e.g. C, N, O, etc., while the edges are unlabeled.

\vspace{1mm}
\noindent\textbf{Linux. }The Linux dataset consists of program dependence graphs generated from the Linux kernel, where each graph represents a function. The nodes and edges represent the statements and the dependency between the two statements, which are both unlabeled.

\vspace{1mm}
\noindent\textbf{IMDB. }The IMDB dataset consists of ego-networks of movie actors and actresses. Each node denotes a movie actor or actress, and each edge between two nodes denotes the two people acting in the same movie. The nodes and edges are unlabeled. 

\vspace{1mm}
\noindent\textbf{Data Preprocessing. }We use the A* algorithm~\cite{riesen2013novel} to generate the exact ground truth for the graph pairs from AIDS, Linux, and a part of IMDB where each graph has no more than 10 nodes. 
Since the GEP to transform $G^1$ to $G^2$ may not be unique, for training, we produce up to 10 ground-truth paths for each graph pair if they exist. Note that each ground-truth path $GEP^*_i$ corresponds to a binary node-matching matrix $\bm{\pi}^*_i$. We use all these $\bm{\pi}^*_i$ as the ground-truth node matching (i.e., $\bm{\pi}^*$ in the matching loss $\mathcal{L}_m$ in Eq.\eqref{eq:iot-ged}) during training to enrich the datasets and improve the model performance. 

For the rest of IMDB where the number of nodes is larger than 10, we generate 100 synthetic graphs for each graph $G$ with the ground-truth generation technique in~\cite{piao2023gedgnn,bai2021tagsim}. Concretely, each synthetic graph $G'$ is randomly generated with $\Delta$ edit operations on nodes/edges, where $\Delta$ is a random number within $(0,10]$ if the nodes of the original graph are larger than 20; otherwise it is within $(0,5]$. Here, $\Delta$ is regarded as an approximation of the ground truth $GED^{*}(G,G')$, and the $\Delta$ edit operations are regarded as the GEP.

\vspace{1mm}
\noindent\textbf{Training Set. }Following the experimental settings of~\cite{piao2023gedgnn}, we sample 60\% graphs in each dataset to form the training set. 
Each sample in the training set is a graph pair. 
For AIDS and Linux, the two graphs in each graph pair in the training set are directly sampled from the graph dataset, since the GED exact ground truth of the graph pairs can be obtained with the A* algorithm. 
For IMDB, we denote those sampled training graphs with at most (resp. larger than) 10 nodes as small (resp. large) graphs. The training set contains two parts: 1) graph pairs formed by two small graphs; 2) graph pairs formed by a large graph and its corresponding synthetic graph. 

\vspace{1mm}
\noindent\textbf{Validation and Test Sets. }We select 20\% graphs in each dataset to form the test set. We evaluate our methods on the scenarios following the setting of~\cite{piao2023gedgnn,bai2019simgnn}, which is to model the graph similarity search. The training set is regarded as the graph database and the graphs in the test set can be regarded as the queries. We sample 100 training graphs for each test graph to form the test set. 
The remaining 20\% graphs of each dataset are sampled to form the validation set in the same way as the test set.
\subsection{Detailed Setup of Our Methods}
\label{app:setup}
Our code is written in Python and all the models are implemented by PyTorch. We use PyTorch Geometric for GNN implementation.

\vspace{1mm}
\noindent\textbf{Parameter Settings. }For the node embedding component, the number of GIN layers is set to 3. The output dimension for each GIN layer is 128, 64, 32, respectively. The dimension of the final node embedding outputted by the MLP is set to $d=32$.  
For the learnable OT component, we use a $32\times 32$ learnable interaction matrix $\mathbf{W}$ in the cost matrix layer. 
Then, we perform the Sinkhorn algorithm with an initial $\varepsilon_0=0.05$ for 5 iterations in the learnable Sinkhorn layer. For the graph discrepancy component, the output dimension of NTN is set to $L=16$. The output dimensions of the four succeeding dense layers are set to 16, 8, 4, 1. The hyper-parameter $\lambda$ in the loss function is set to 0.8. During training, 
we set the batch size to 128 and use the Adam optimizer with initial learning rate and weight decay set to 0.001 and $5\times 10^{-4}$, respectively. For GEP generation, we set $k$ to 100 in the $k$-best matching framework.

\section{Additional Experiments}\label{app:exp}

\begin{figure}[t]
    \centering
    \subfigure[IMDB - MAE]{
    \includegraphics[width=0.485\linewidth]{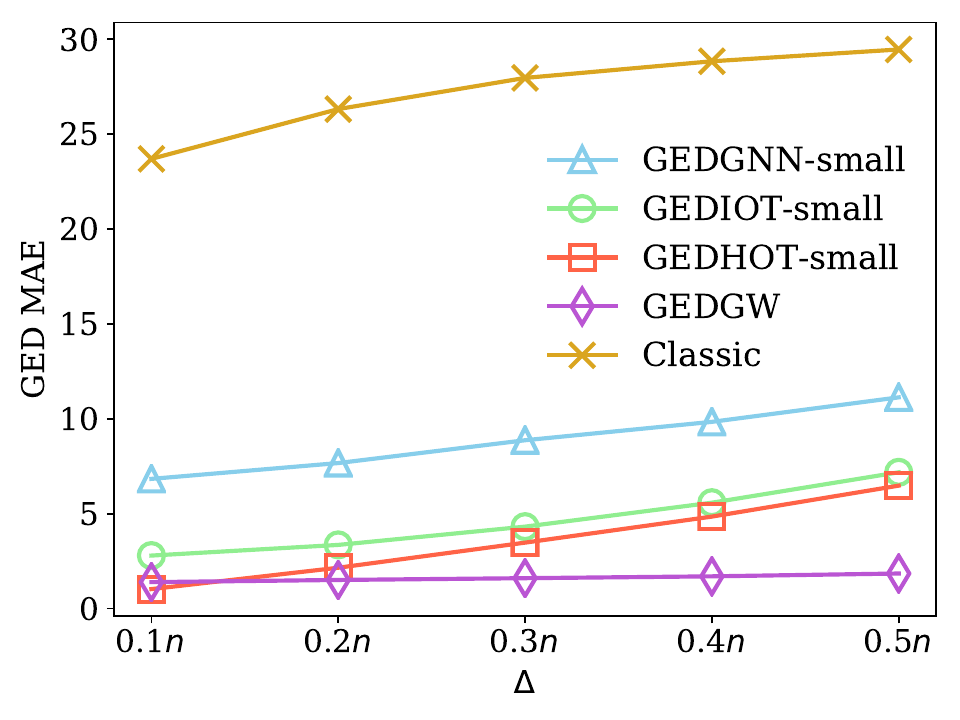}}
    \vspace{-2mm}
    \subfigure[IMDB - Accuracy]{
    \includegraphics[width=0.485\linewidth]{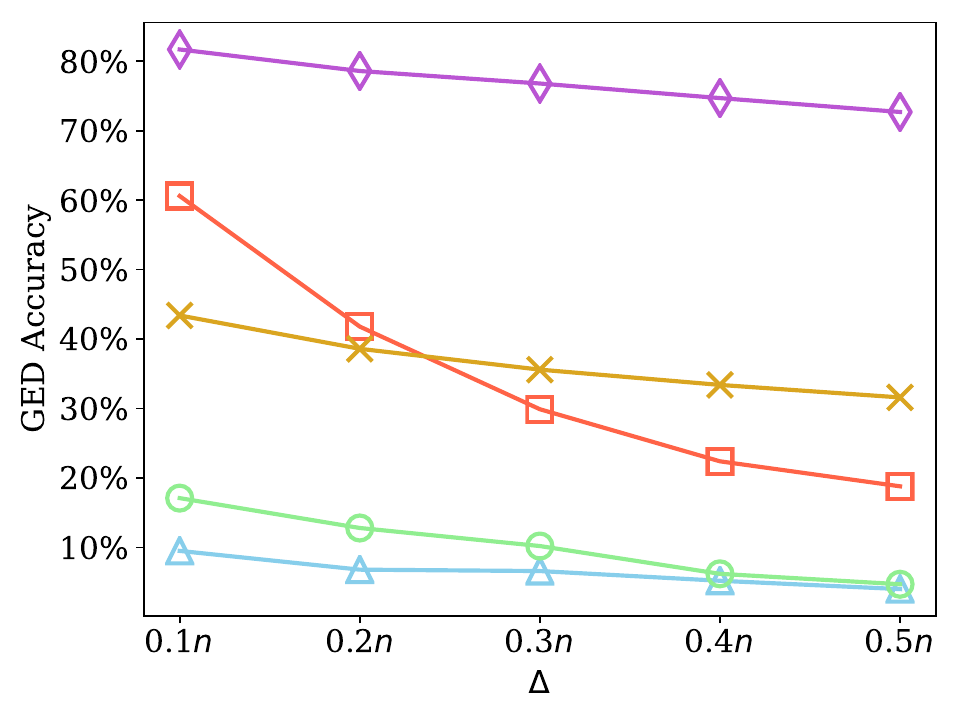}}
    \vspace{-2mm}
    \caption{Further evaluation of generalizability for large unseen graphs on IMDB with Increasing GED}
    \label{fig:generalization_small}
\end{figure}
\setlength{\textfloatsep}{5pt}

\subsection{Generalizability}\label{app:gen_exp}
We further discuss how the generalizability of GEDGNN-small, GEDIOT-small, and GEDHOT-small is impacted when synthesizing large test graph pairs (more than 10 nodes) with larger GEDs (i.e., the discrepancy between two graphs becomes more pronounced). 
Concretely, for each original large graph with $n$ nodes ($n>10$) in the test set of IMDB, we regenerate 100 synthetic graphs with edit operations $\Delta=\lceil r\cdot n\rceil$, where $r$ is in the range of $(0,1)$ and $\Delta$ can be viewed as an approximation of the ground-truth GED as described in Section~\ref{sec:dataset}. We vary $r$ from 10\% to 50\% and Figure~\ref{fig:generalization_small} depicts the influence on MAE and accuracy. 

We can see that both non-learning methods Classic and GEDGW are quite stable as $\Delta$ varies since they do not need ground-truths. The MAE of Classic is several times worse than that of the other four methods, but as $\Delta$ increases, Classic achieves a better accuracy than the learning-based methods. This implies that Classic can recover the exact GED in several instances but struggles in others.  Our proposed GEDGW significantly outperforms all the others including the learning-based methods in terms of MAE and accuracy, showing the great robustness of GEDGW compared with other methods.

Among the three learning-based methods trained on the small training set (graphs with nodes no more than 10), GEDHOT-small achieves the best performance with the help of GEDGW, and GEDIOT-small is consistently better than GEDGNN-small. This indicates that our proposed neural network model exhibits superior generalizability compared to the existing learning-based methods.

\begin{figure}[t]
    \centering
    \subfigure[Adoption Rate in GED Computation]{
    \includegraphics[width=0.485\linewidth]{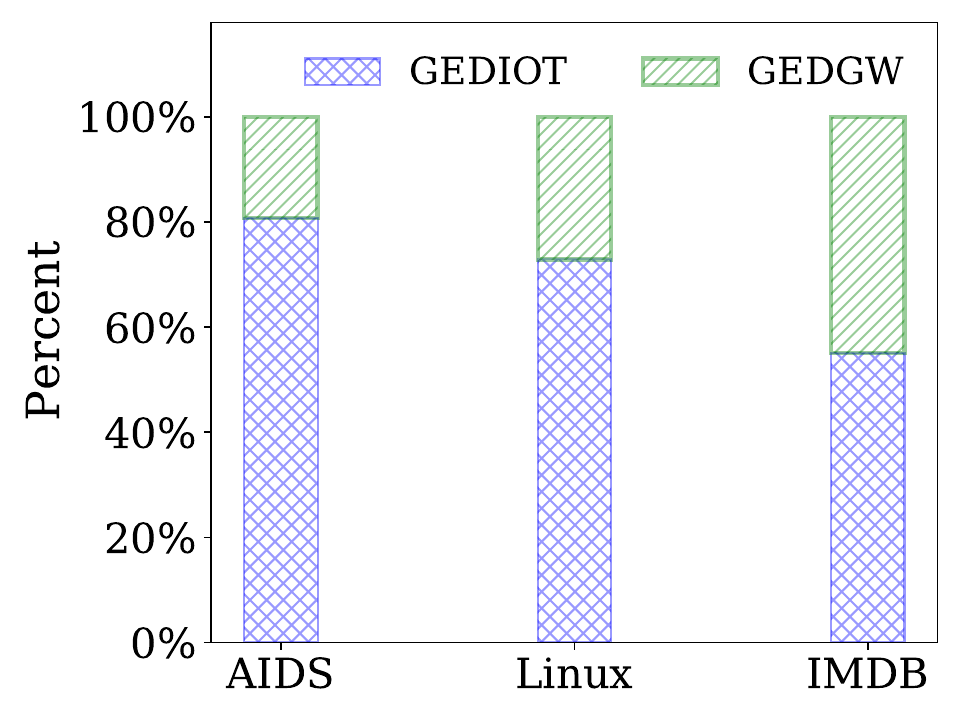}}
    \subfigure[Adoption Rate in GEP Generation]{
    \includegraphics[width=0.485\linewidth]{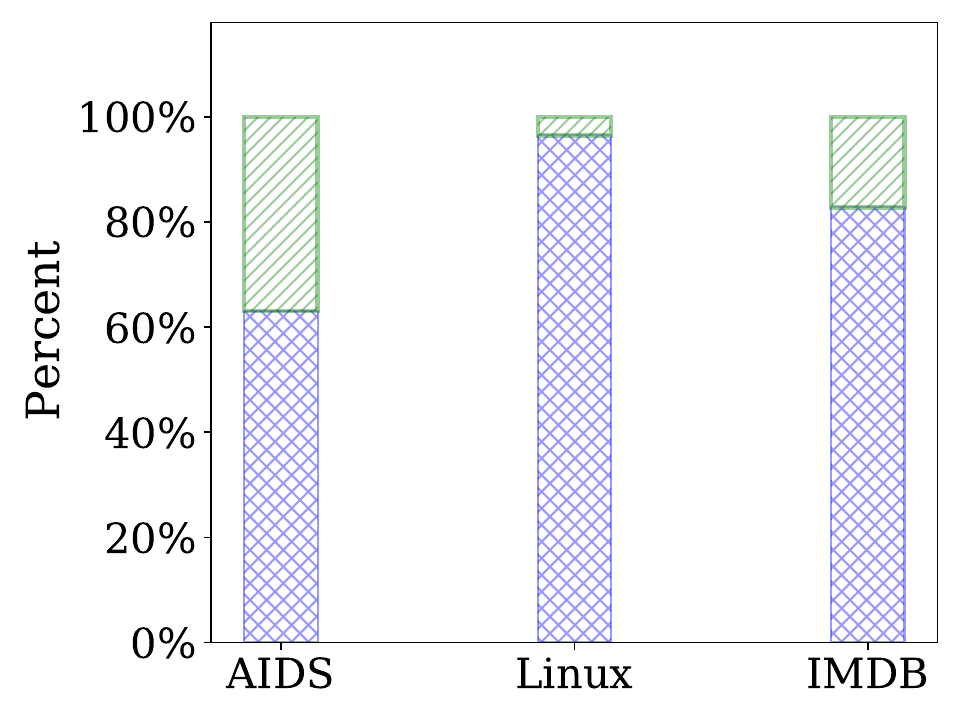}}
    
    \caption{Adoption Rate of GEDIOT/GEDGW for GEDHOT}
    \label{fig:contribution}
\end{figure}
\setlength{\textfloatsep}{5pt}

\begin{figure}[t]
    \centering
    \includegraphics[width=0.8\linewidth]{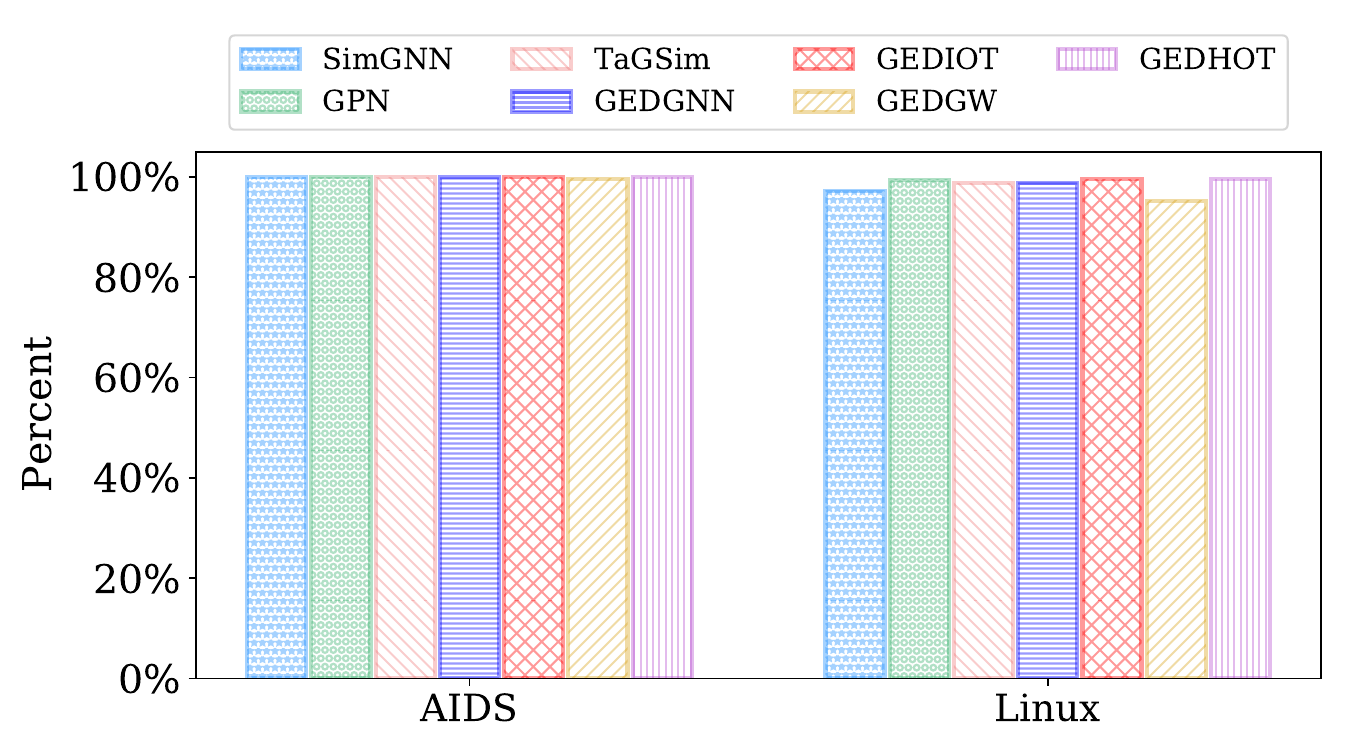} 
    \caption{Triangle Property Preservation of the Predicted GEDs}
    \label{fig:triangle}
\end{figure}
\setlength{\textfloatsep}{5pt} 

\subsection{More Evaluation on Proposed Methods}
\label{app:adoption-ratio}
\vspace{1mm}
\noindent {\textbf{Adoption Ratio of GEDIOT and GEDHOT.} The ensemble method GEDHOT adopts the smaller GED of GEDIOT and GEDGW, and the shorter GED path of GEDIOT and GEDGW. GEDHOT uses the values and paths from GEDIOT by default unless GEDGW outputs better results. We evaluate the ratio of the cases in which GEDGW outperforms GEDIOT and vice versa. As shown in Figure~\ref{fig:contribution}, on AIDS, for GED computation, most graph pairs ($80.8\%$) use the results from GEDIOT instead of GEDGW. 
For GEP generation, $63.1\%$ of the graph pairs use the results from GEDIOT, and $36.9\%$ of the graph pairs use the results from GEDGW. 
The results show the need to apply GEDGW (as a non-learning method) to offset the potential weakness of GEDIOT (and learning-based methods in general) for GED computation and GEP generation, particularly on pairs of larger graphs in IMDB that are difficult to train well. }

\vspace{1mm}
\noindent {\textbf{Triangle Property Preservation of the Predicted GEDs.} To evaluate whether learning-based methods preserve the triangle inequality in the GED predictions, We randomly sample triples of graphs of the form $(G^1, G^2, G^3)$ and report the fraction of violations for various learning-based methods (including ours). Figure~\ref{fig:triangle} shows that on AIDS and Linux, our methods preserve the GED triangle inequality for more than $95\%$ cases. Particularly, on AIDS, GEDIOT and GEDHOT preserve the property for $99.9\%$ cases. }

\begin{figure}[t]
    \centering
    \subfigure[AIDS-total ($n=20$)]{
    \includegraphics[width=0.485\linewidth]{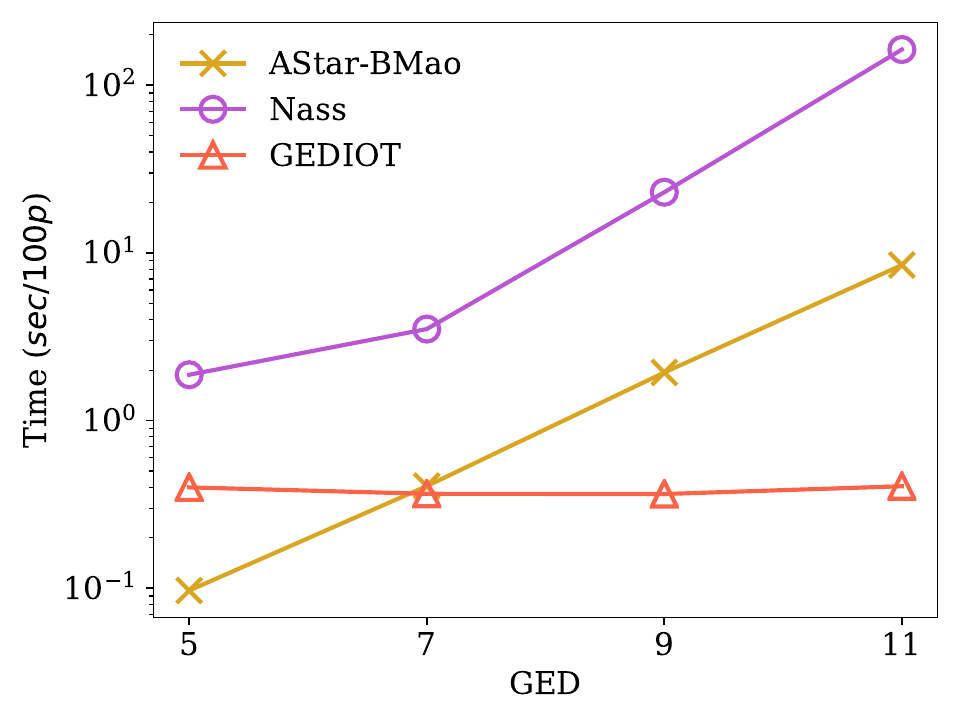}}
    \subfigure[AIDS-total ($n=30$)]{
    \includegraphics[width=0.485\linewidth]{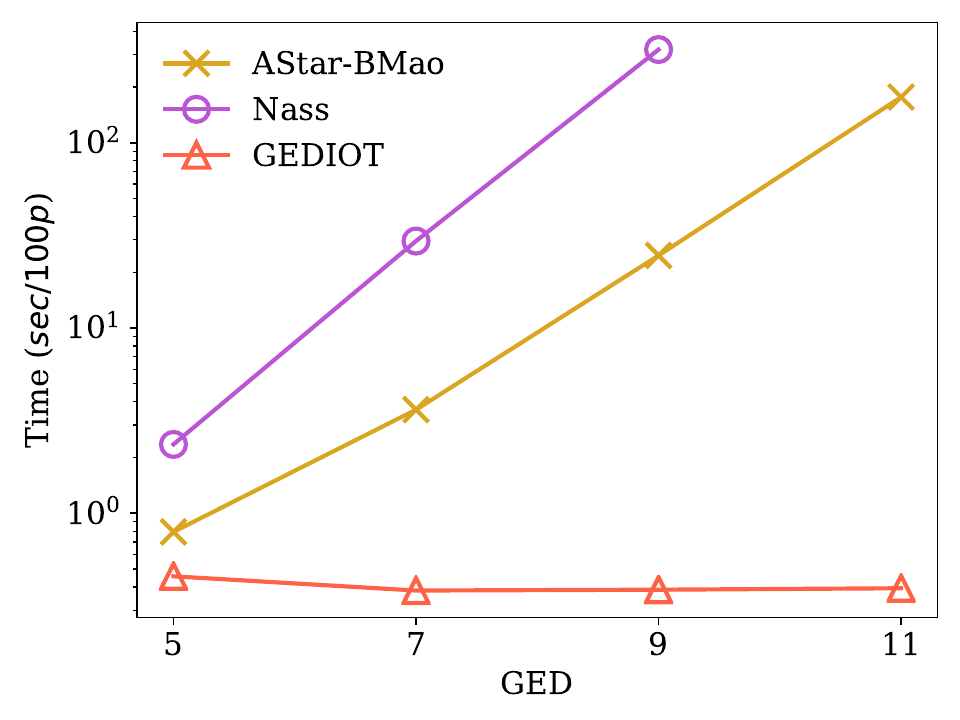}}

    \subfigure[AIDS-total ($n=40$)]{
    \includegraphics[width=0.485\linewidth]{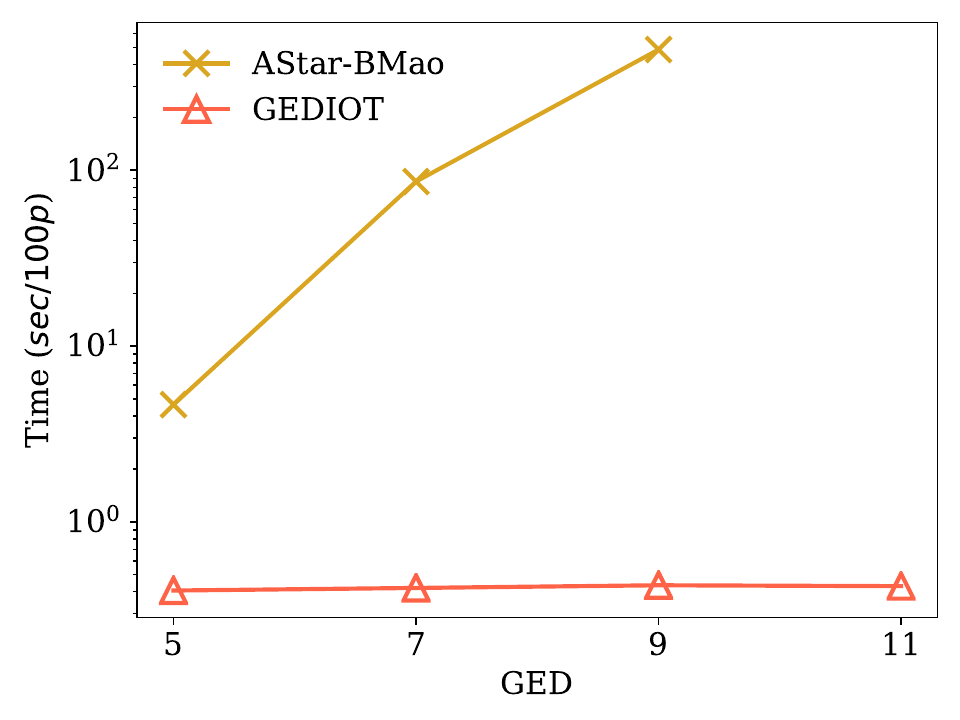}}
    \subfigure[IMDB ($n=20$)]{
    \includegraphics[width=0.485\linewidth]{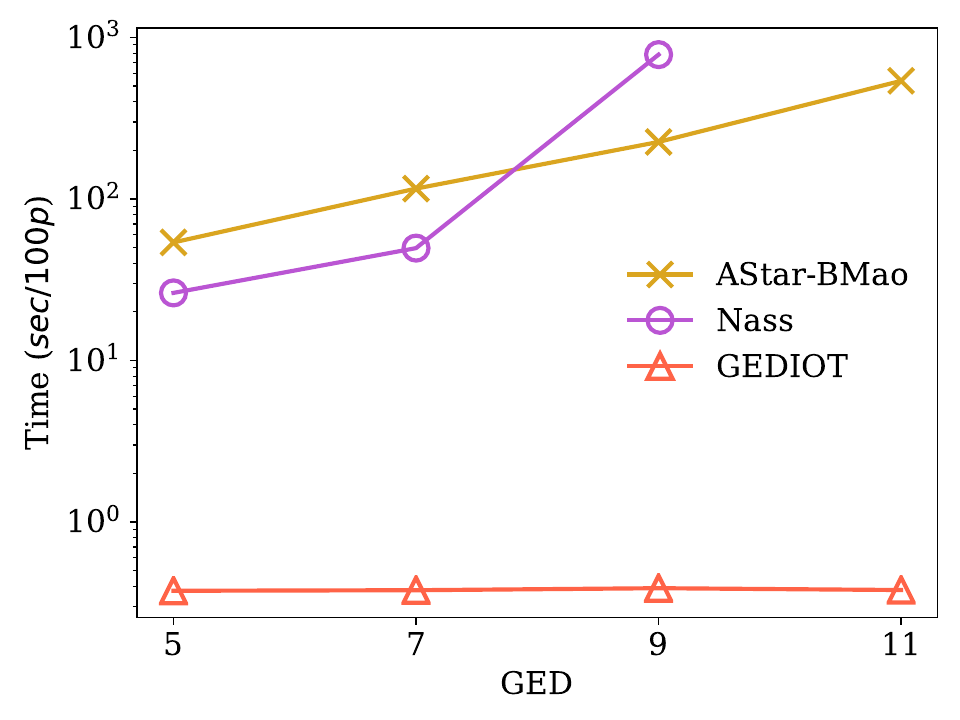}}
    \caption{Efficiency Comparison with Exact Algorithms}
    \label{fig:exact}
\end{figure}
\setlength{\textfloatsep}{5pt}

    

\subsection{Comparison with Exact Methods}\label{app:exact-methods}
{We notice that the state-of-the-art methods Nass~\cite{kim2021boosting} and AStar-BMao~\cite{chang2022accelerating} for graph similarity search (introduced in Section~\ref{sec:related}) can be applied for exact GED computation by setting the threshold in search task to infinity. 
As indicated in~\cite{piao2023gedgnn}, exact methods suffer from huge computation costs when the graph size increases. 
We first compare our method GEDIOT with the exact ones on two large real-world datasets: AIDS-total\footnote{\url{https://cactus.nci.nih.gov/download/nci/AID2DA99.sdz}} and IMDB. Differing from AIDS introduced in Table~\ref{Table:graph} and Section~\ref{sec:dataset}, here we use the large dataset AIDS-total that contains 42,689 graphs, with 25.60 nodes per graph on average. IMDB is the same as the dataset in Table~\ref{Table:graph}, consisting of 1500 unlabeled graphs, with 13 nodes per graph on average. 

We remove the edge labels in the large AIDS-total dataset following the convention of learning-based methods~\cite{piao2023gedgnn, bai2021tagsim}, and compare Nass and AStar-BMao with our learning method GEDIOT on AIDS-total and IMDB. For each graph dataset, we select subsets of graphs from the dataset with $n$ nodes, where all the graphs in each subset have $n$ nodes. We use four groups of graphs with $n = 20, 30, 40$ from AIDS-total and only one group of graphs with $n = 20$ from IMDB 
since on larger graphs of IMDB, AStar-BMao cannot output the results within 24 hours and Nass returns a bus error, likely caused by too deep recursion. 
We sample 60\% graphs to train GEDIOT and 40\% for efficiency evaluation and use the ground-truth generation technique as described in Data Preprocessing in Appendix~\ref{app:data} to generate graph pairs. 
Concretely, we fix $\Delta=5,7,9,11$ to generate four groups of graph pairs for each subset, where each group has 100 graph pairs.

In Figure~\ref{fig:exact}, we report the average running time of every 100 pairs for each group using the three methods. The computational time of the two exact methods Nass and AStar-BMao is quite sensitive w.r.t. the graph size and the GED value. We do not report the results of some groups of AStar-BMao and Nass since they fail to return the GED value due to bus error. Our method GEDIOT shows a consistent advantage compared to the two exact algorithms, particularly for larger graphs and GEDs. In particular, on AIDS-total ($n=40$) and IMDB ($n=20$), GEDIOT outperforms the state-of-the-art exact algorithm in time efficiency by orders of magnitude, as the time complexity of GEDIOT is only $O(n^2)$, whereas AStar-BMao and Nass are still exponential-time algorithms. 

Moreover, notice that the scalability of the exact methods on IMDB is worse than that on AIDS. The reason could be that the graphs in IMDB are denser and have no labels leading to a huge search space when using exact algorithms.}

\setcounter{figure}{17}
\begin{figure*}[t]
    \centering
    \subfigure[Number of Iterations - Time]{
    \includegraphics[width=0.32\linewidth]{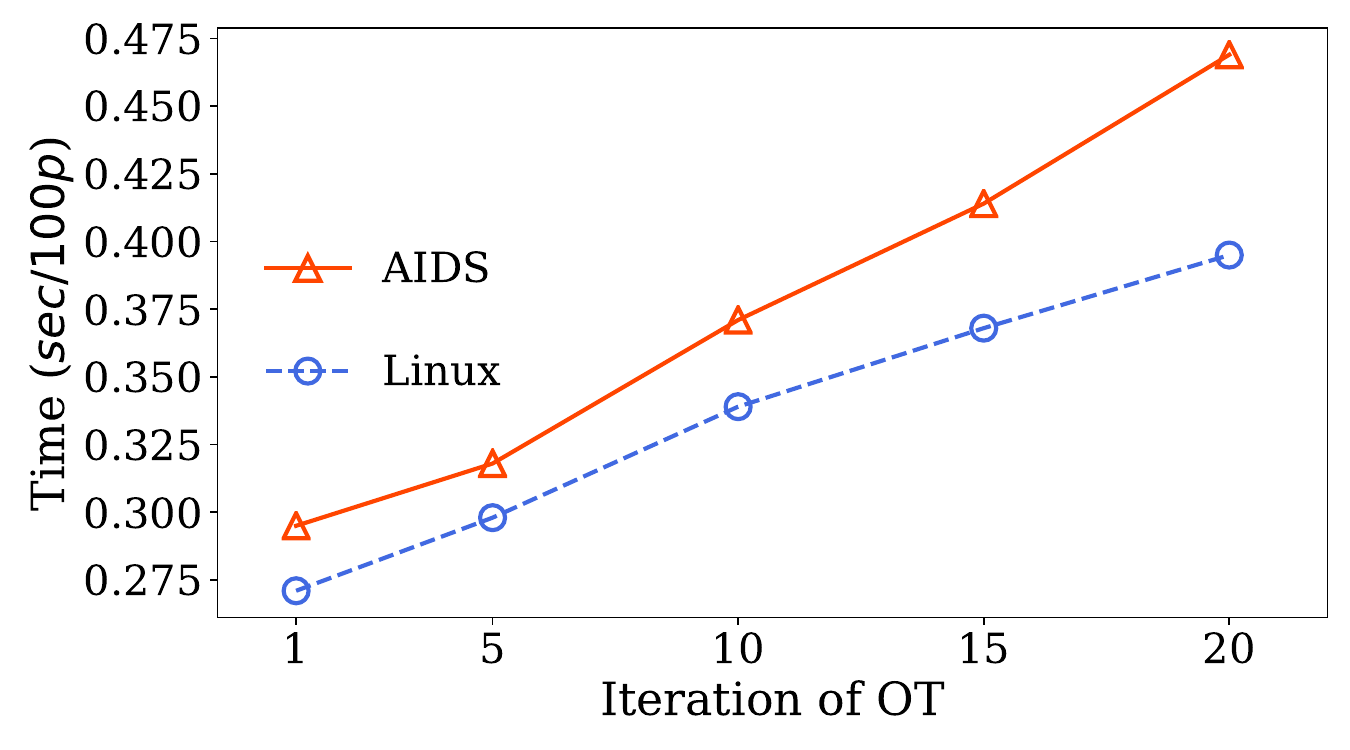}}
    \subfigure[Number of Iterations - MAE]{
    \includegraphics[width=0.32\linewidth]{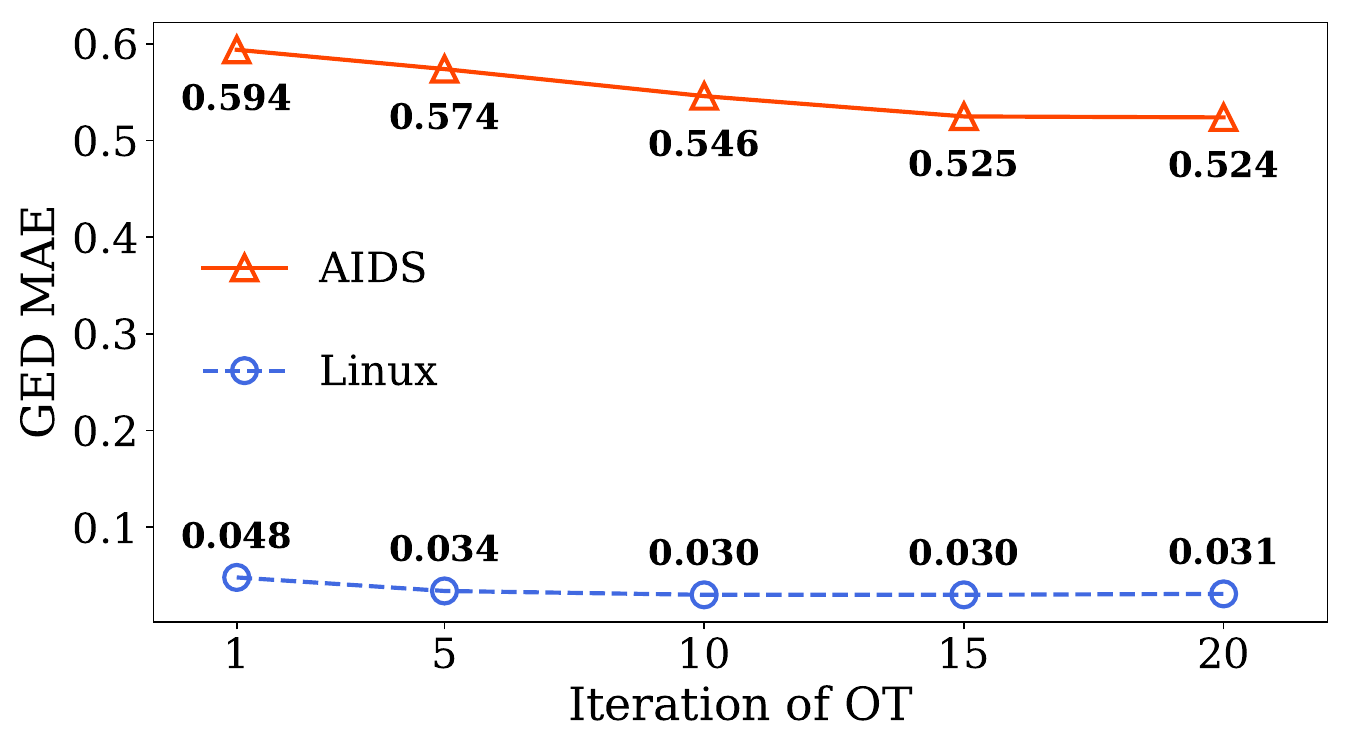}}
    \subfigure[Number of Iterations - Accuracy]{
    \includegraphics[width=0.32\linewidth]{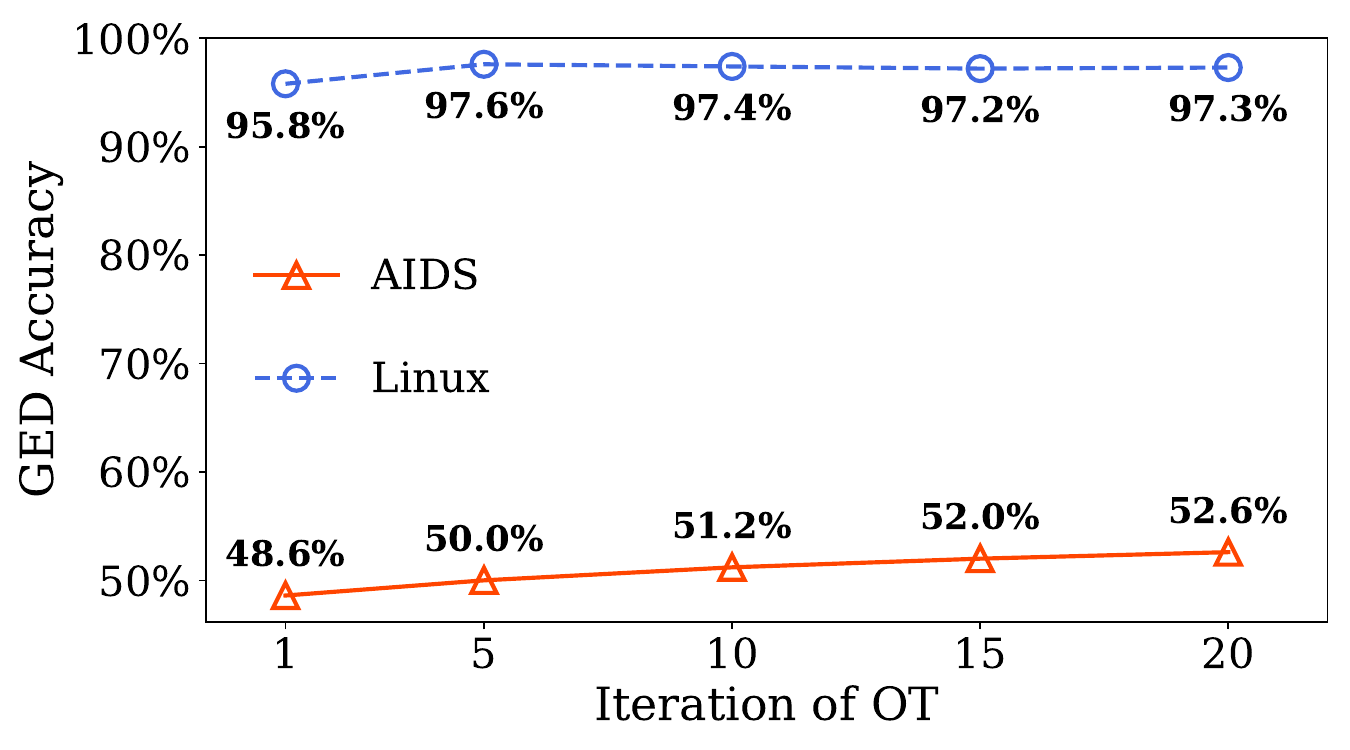}}
    \vspace{-1mm}
    \caption{Effect of Various Numbers of Iterations for the Sinkhorn Algorithm on GEDIOT}
    \label{fig:iternum}
    \vspace{-2mm}
\end{figure*}
\setlength{\textfloatsep}{5pt}

\setcounter{figure}{15}
\begin{figure}[t]
    \centering
    \subfigure[Graph Size - Relative Error]{
    \includegraphics[width=0.485\linewidth]{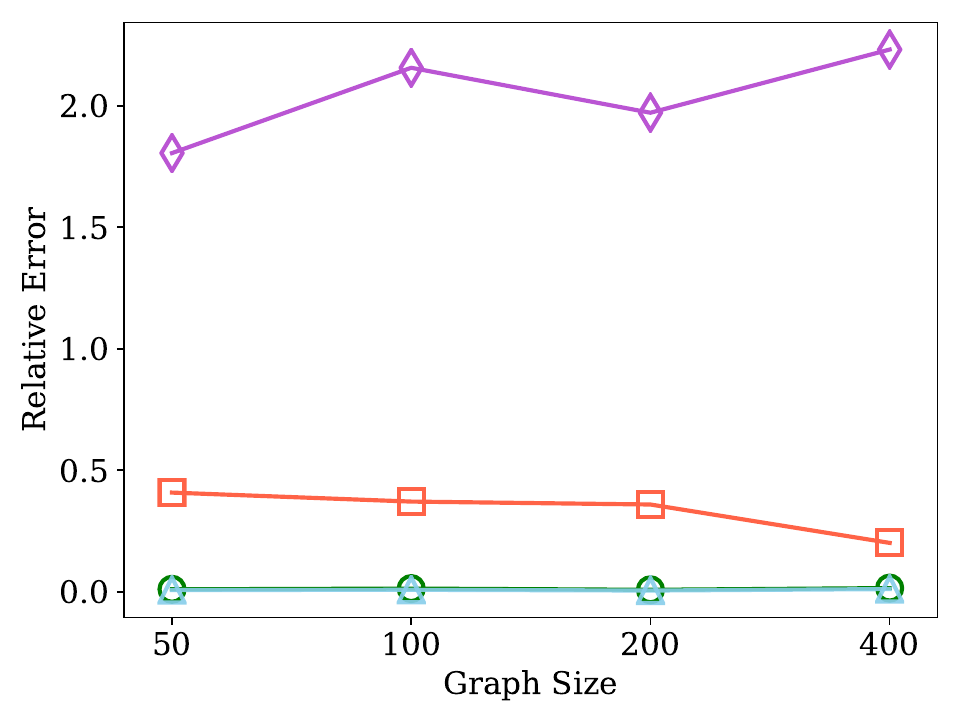}\label{fig:power-lawsubfig1}}
    \subfigure[Graph Size - Time (sec/100pair)]{
    \includegraphics[width=0.485\linewidth]{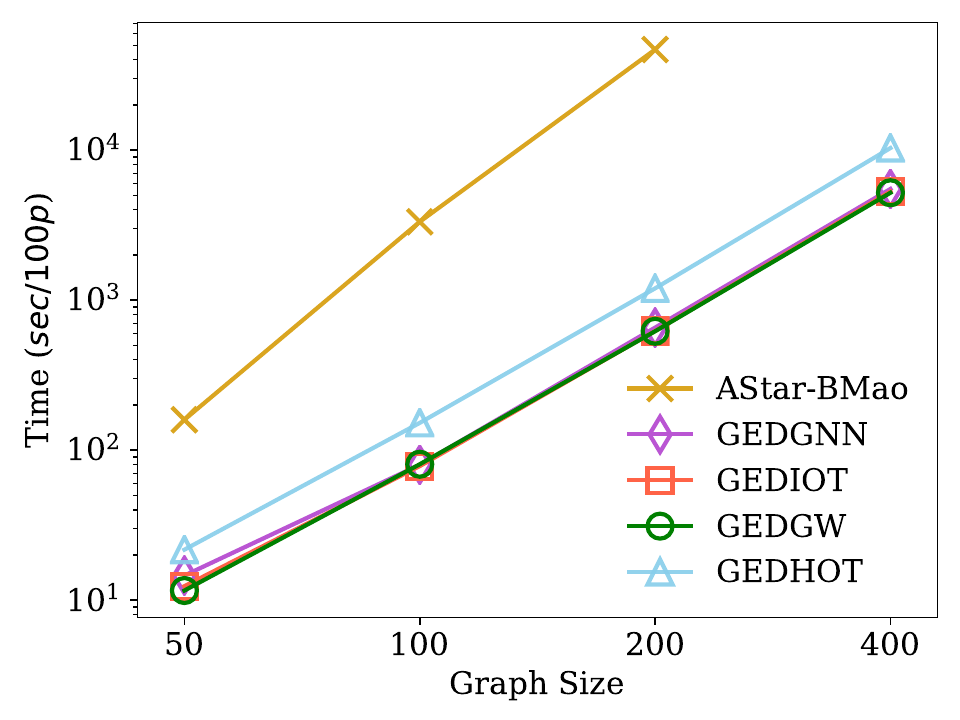}\label{fig:power-lawsubfig2}}
    \caption{Accuracy and Efficiency on Power-law Graphs}
    \label{fig:power-law}
\end{figure}
\setlength{\textfloatsep}{5pt}

\subsection{Performance Evaluation on Large Power-Law Graphs}\label{app:power-law}
{
Following the experiments on large power-law graphs in Section~5.4 of GEDGNN~\cite{piao2023gedgnn}, we also generate four groups of large synthetic power-law graphs with various graph sizes $n$. The graph sizes $n$ of the four groups are set as $50,\ 100,\ 200$, and $400$, respectively. For each $n$, we generate 500 pairs for training and testing, respectively. The results are shown in Figure~\ref{fig:power-law}. 


In Figure~\ref{fig:power-lawsubfig1}, we report the GED relative errors (i.e., $(\widehat{GED}-GED^*) / GED^*$) of the approximate methods with $k$-best matching framework: GEDGNN, GEDIOT, GEDGW, and GEDHOT. Note that the relative errors of all the methods are quite stable as the graph size $n$ varies. Specifically, the relative error of our GEDGW and GEDHOT is nearly $0$. In stark contrast, that of GEDGNN hovers around a relatively high value of almost $2$. This pronounced discrepancy showcases the superiority of our proposed methods in larger power-law graphs.

Figure~\ref{fig:power-lawsubfig2} depicts the average running time of 100 graph pairs for the exact algorithm AStar-BMao and the above approximate methods. We do not report the result of Nass as it cannot output the results on the four groups due to bus error. It shows that the average running time of approximate methods is consistently orders of magnitude faster than AStar-BMao. The result of AStar-BMao on 400-node graphs is not reported since it cannot generate results within 24 hours. The time taken by GEDIOT, GEDGW and GEDGNN is comparable, whereas the time consumed by the ensemble method GEDHOT is about the summation of the time consumed by GEDIOT and GEDGW.}

\begin{figure}[t]
    \centering
    \subfigure[$\varepsilon_0$ - MAE]{
    \includegraphics[width=0.485\linewidth]{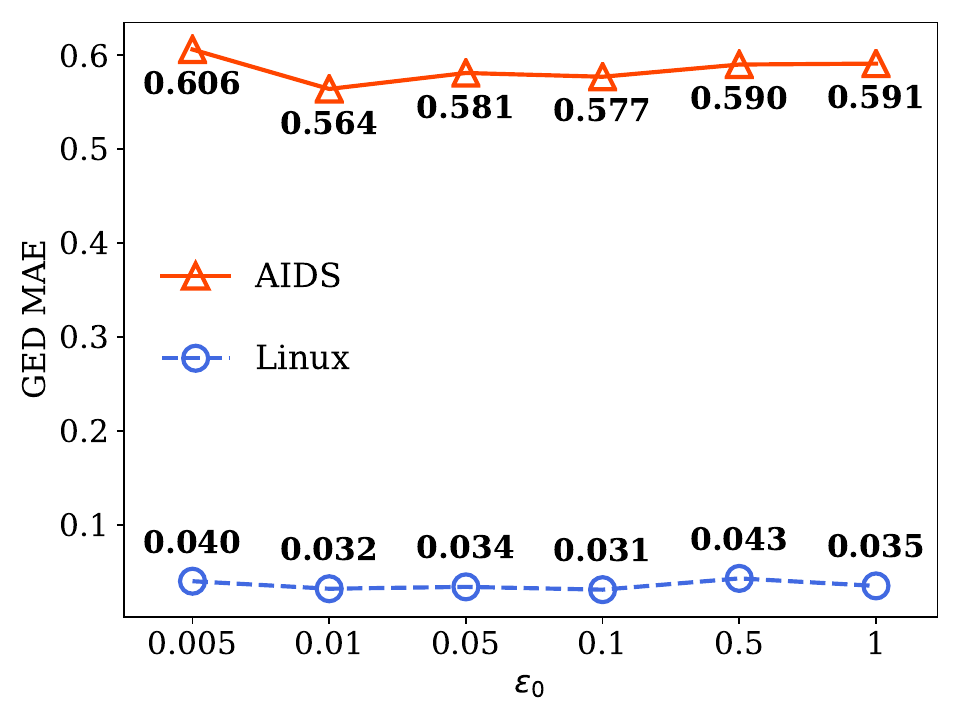}}
    \subfigure[$\varepsilon_0$ - Accuracy]{
    \includegraphics[width=0.485\linewidth]{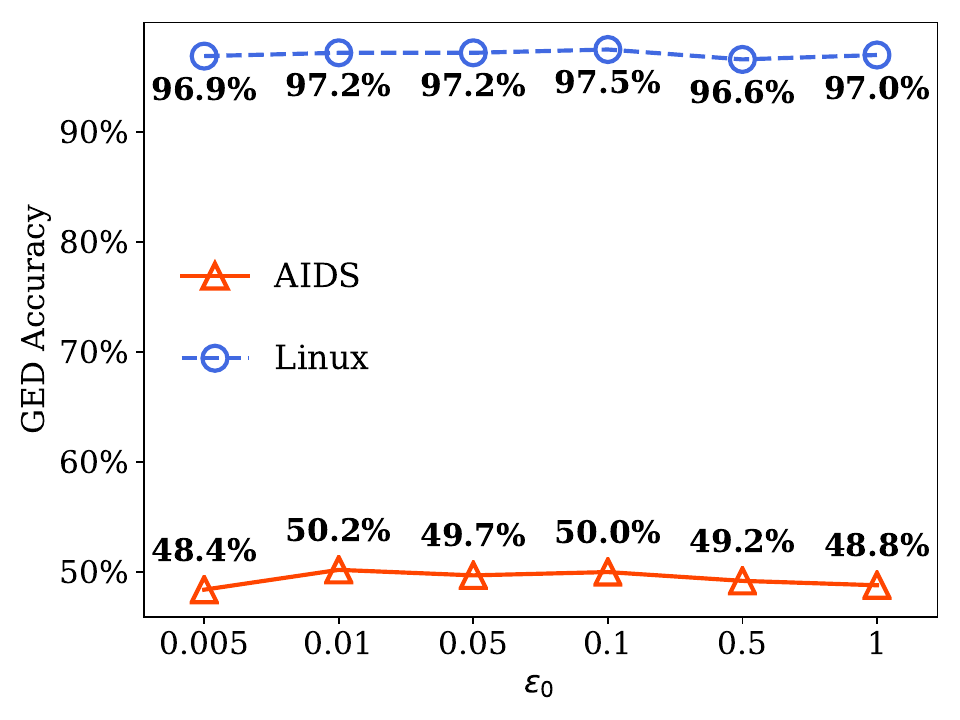}}
    \vspace{-4mm}

    \caption{Varying $\varepsilon_0$ in the Sinkhorn Algorithm}
    \vspace{-6mm}
    \label{fig:eps-vary}
\end{figure}
\setlength{\textfloatsep}{5pt}

\setcounter{figure}{18}
\begin{figure}[t]
    \centering
    \subfigure[$\lambda$ - MAE]{
    \includegraphics[width=0.485\linewidth]{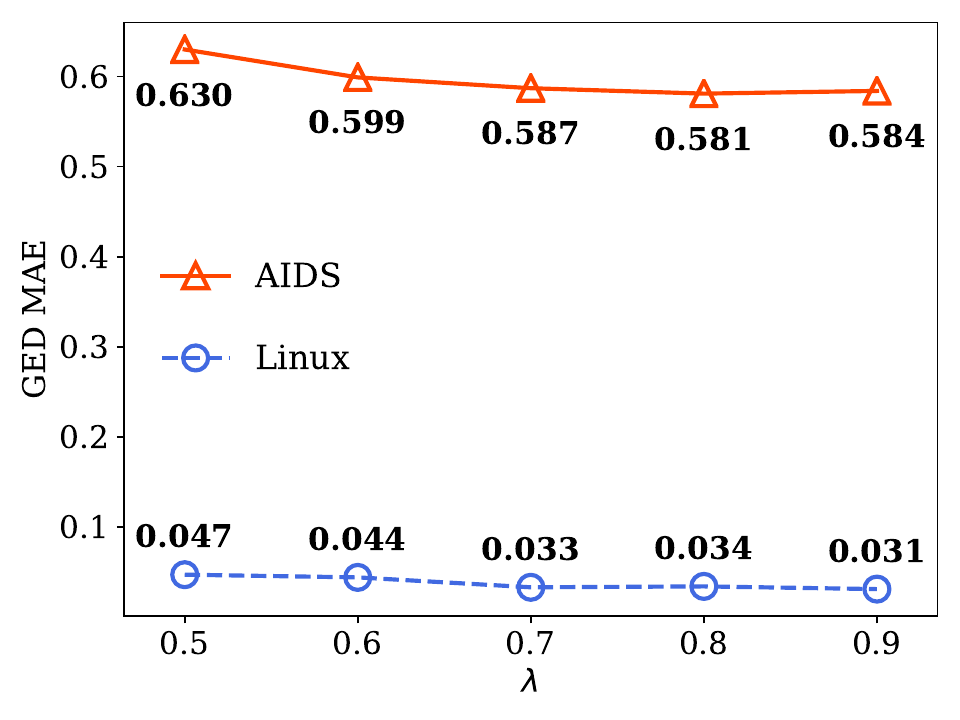}}
    \subfigure[$\lambda$ - Accuracy]{
    \includegraphics[width=0.485\linewidth]{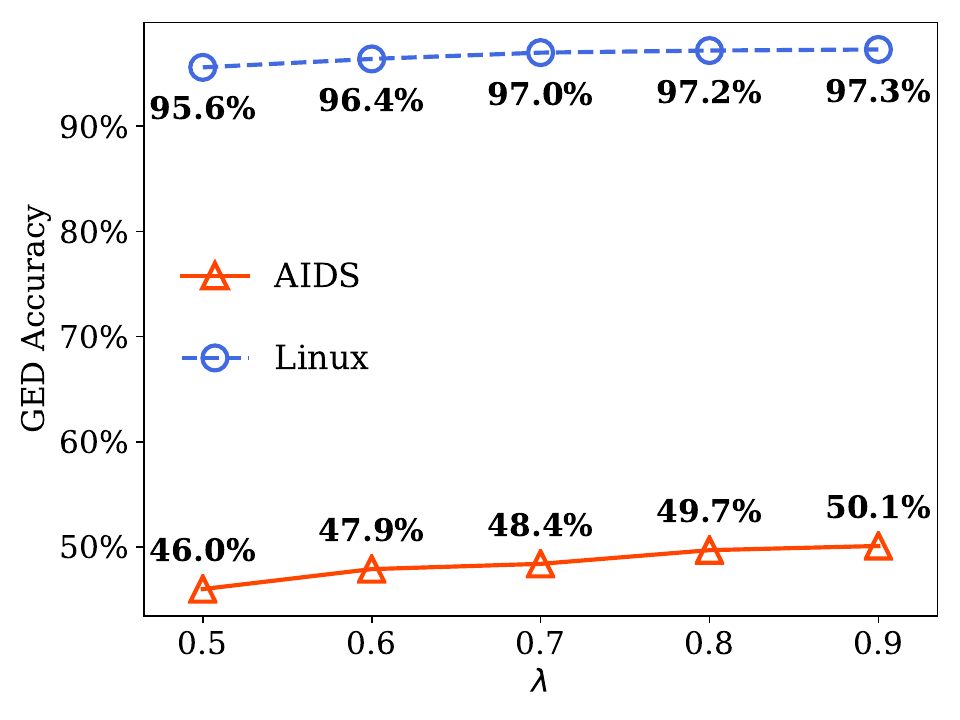}}
    \vspace{-3mm}
    \caption{Varying $\lambda$ in the Loss Function}

    \label{fig:lambda-vary}
\end{figure}
\setlength{\textfloatsep}{5pt}

\setcounter{figure}{19}

\begin{figure*}[t]
    \centering
    \subfigure[Training Size - Time]{
    \includegraphics[width=0.32\linewidth]{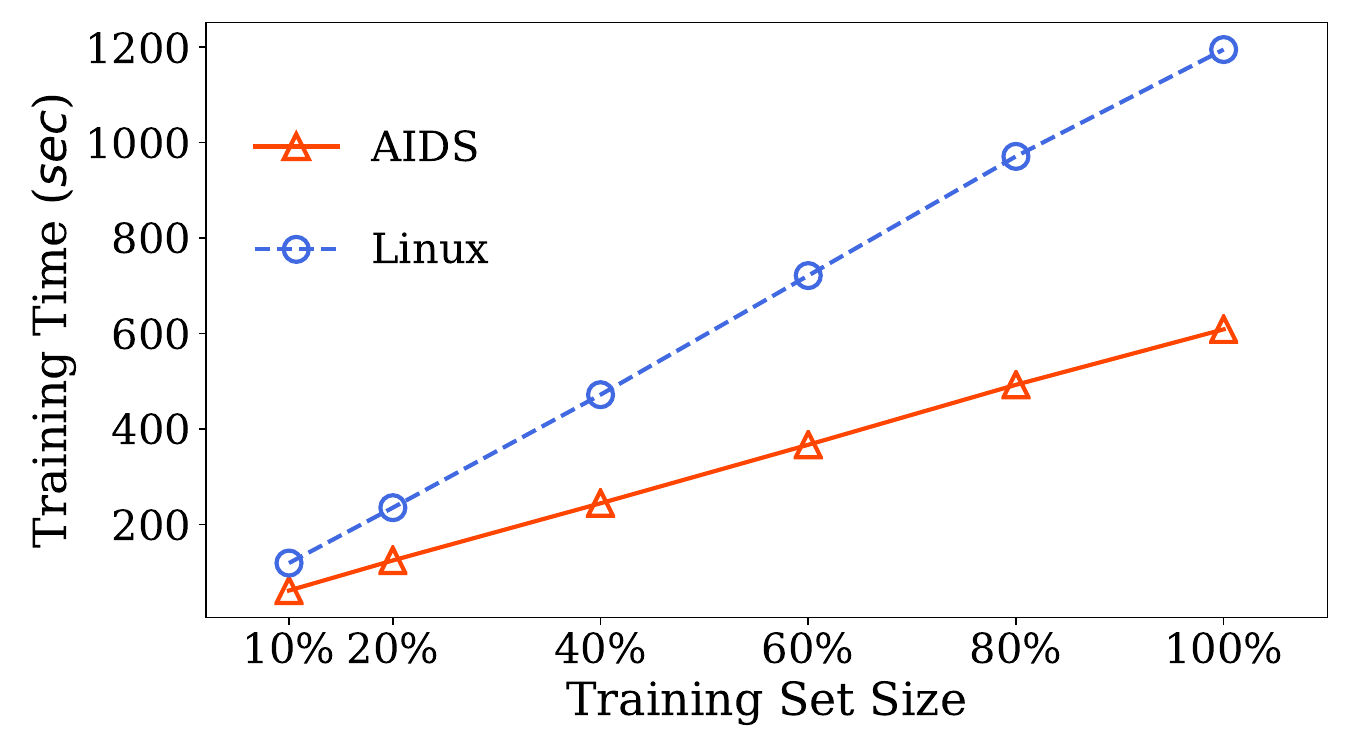}}
    \subfigure[Training Size - MAE]{
    \includegraphics[width=0.32\linewidth]{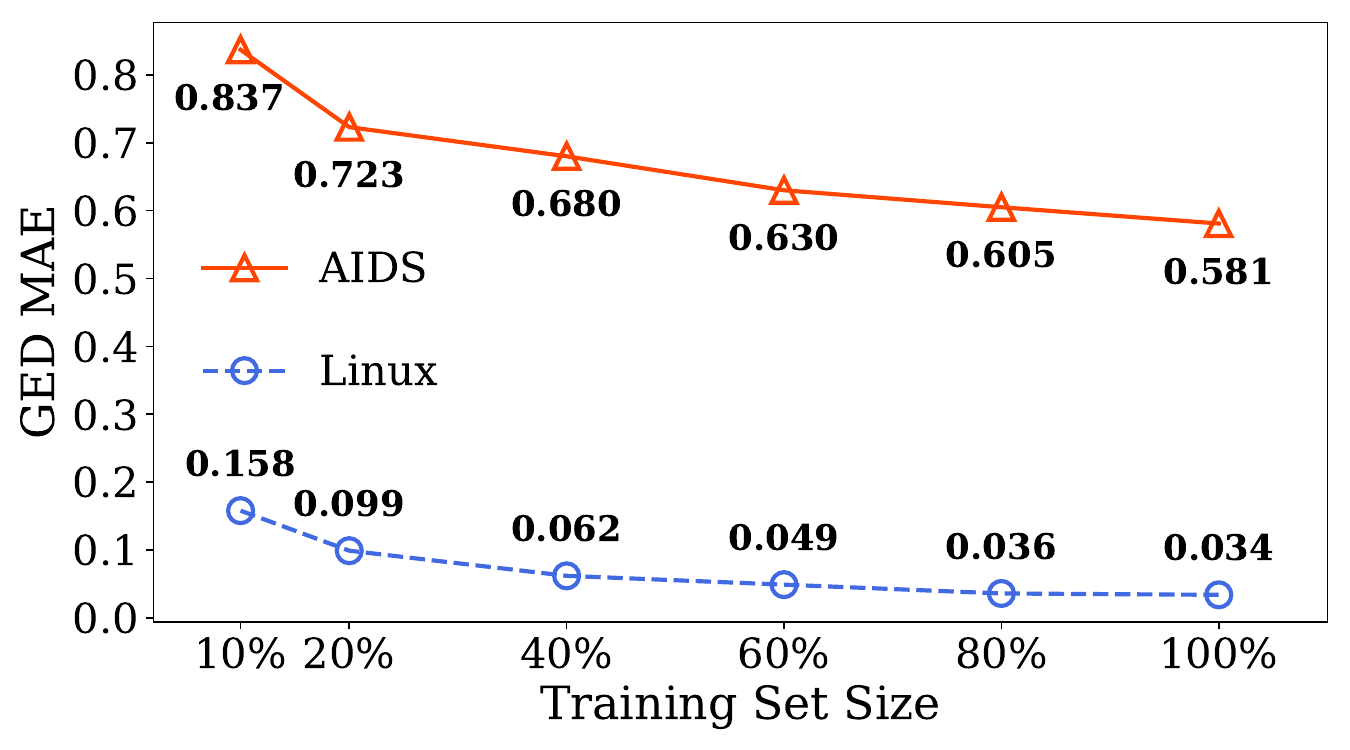}}
    \subfigure[Training Size - Accuracy]{
    \includegraphics[width=0.32\linewidth]{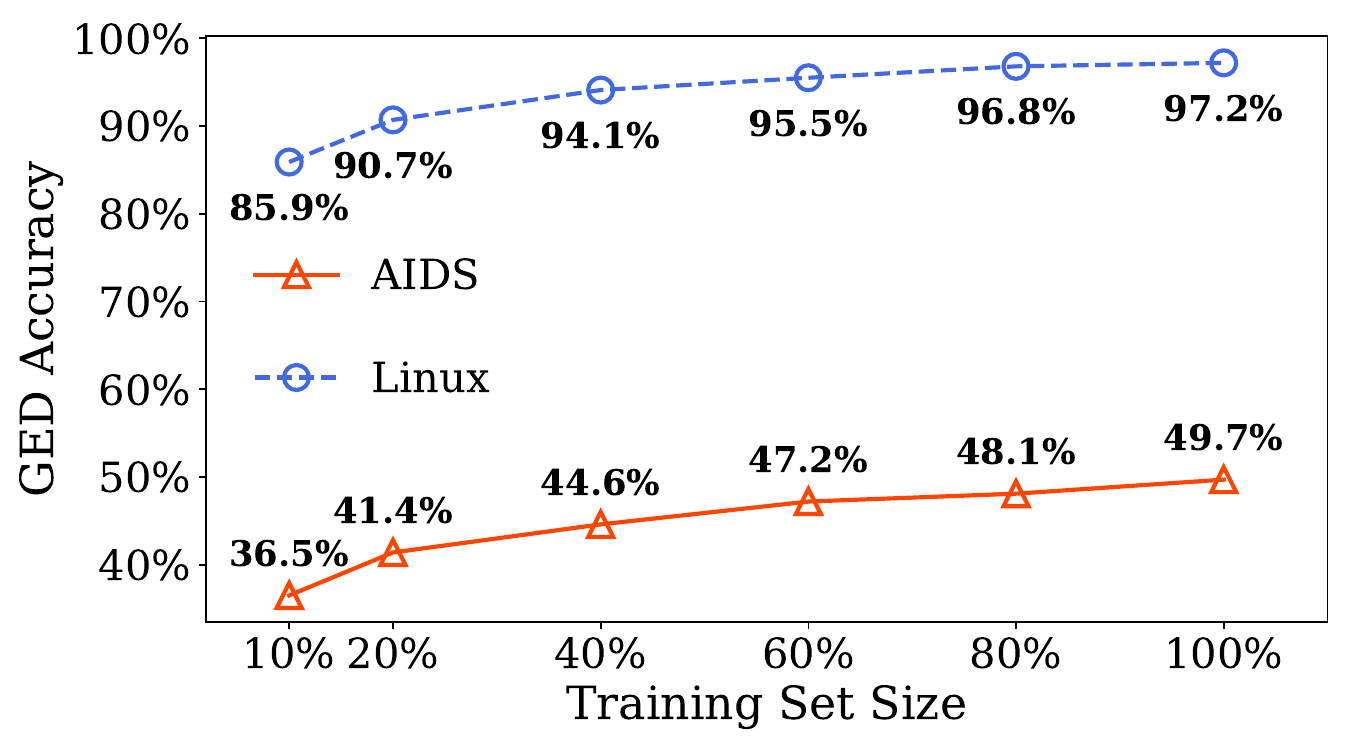}}
    \caption{Effect of Various Training Set Sizes on GEDIOT}
    \label{fig:trainsize}
\end{figure*}
\setlength{\textfloatsep}{5pt}

\begin{figure*}[t]
    \centering
    \subfigure[AIDS - Time ($sec/100p$)]{
    \includegraphics[width=0.32\linewidth]{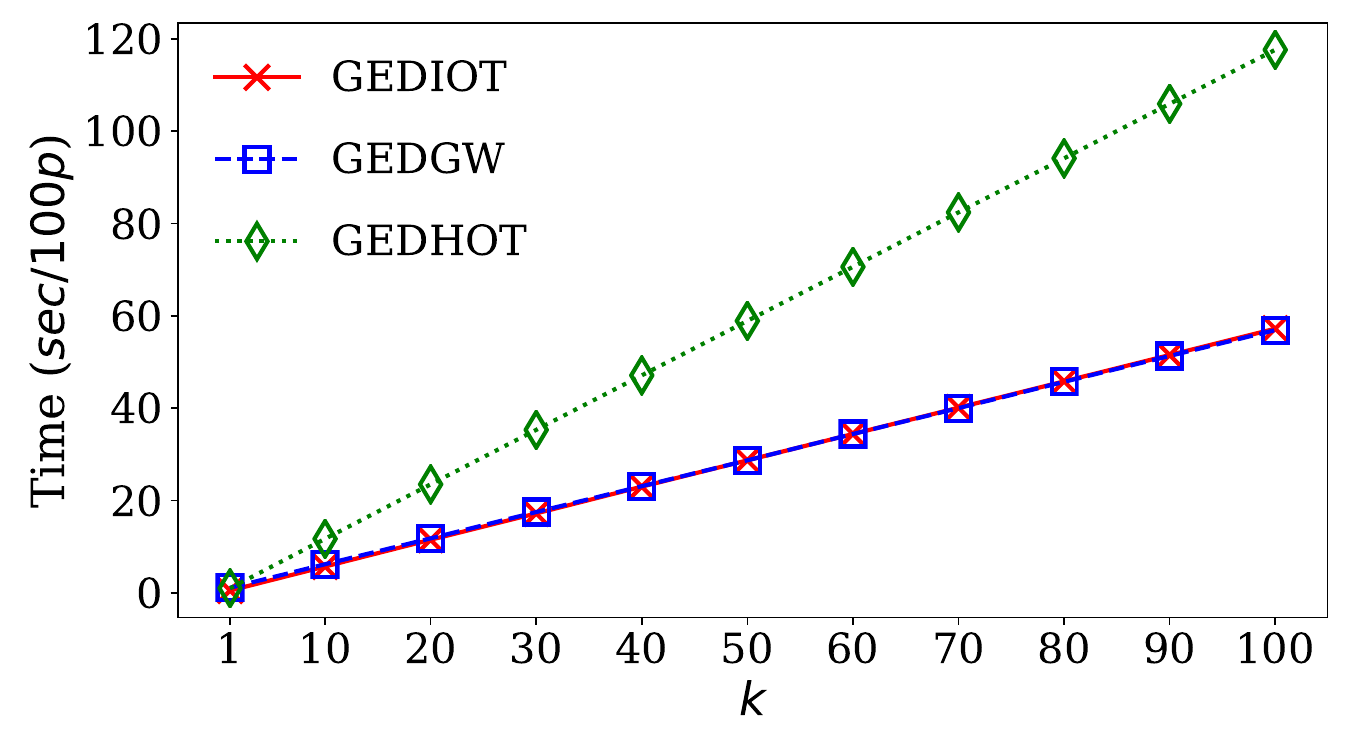}}
    \subfigure[AIDS - MAE]{
    \includegraphics[width=0.32\linewidth]{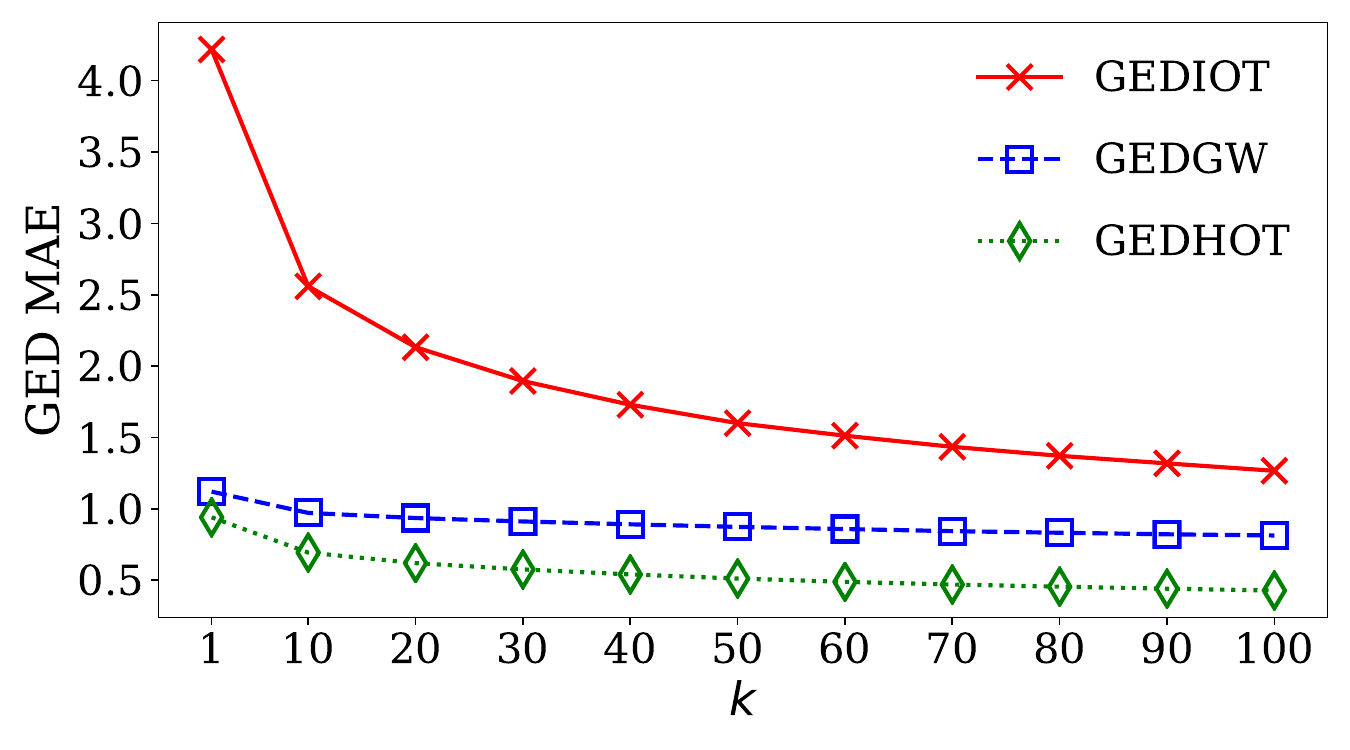}}
    \subfigure[AIDS - Accuracy]{
    \includegraphics[width=0.32\linewidth]{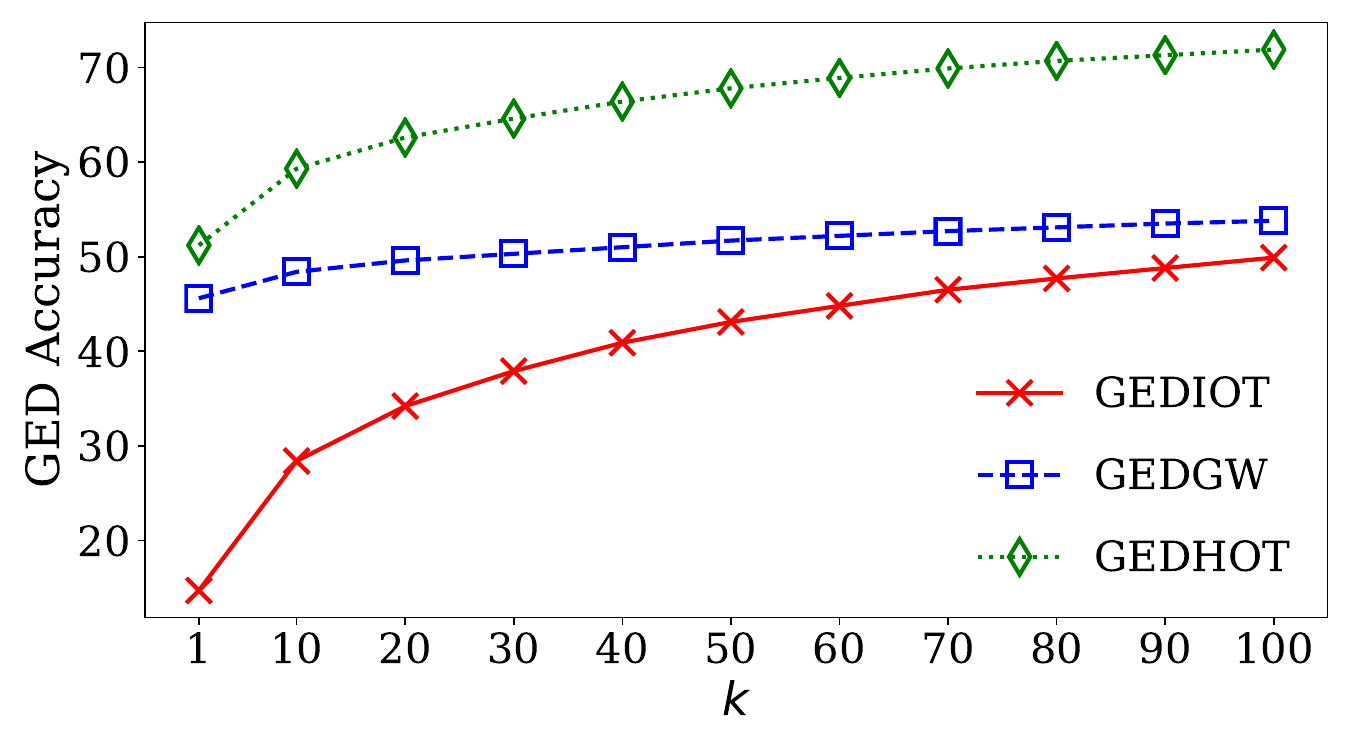}}

    \subfigure[Linux - Time ($sec/100p$)]{
    \includegraphics[width=0.32\linewidth]{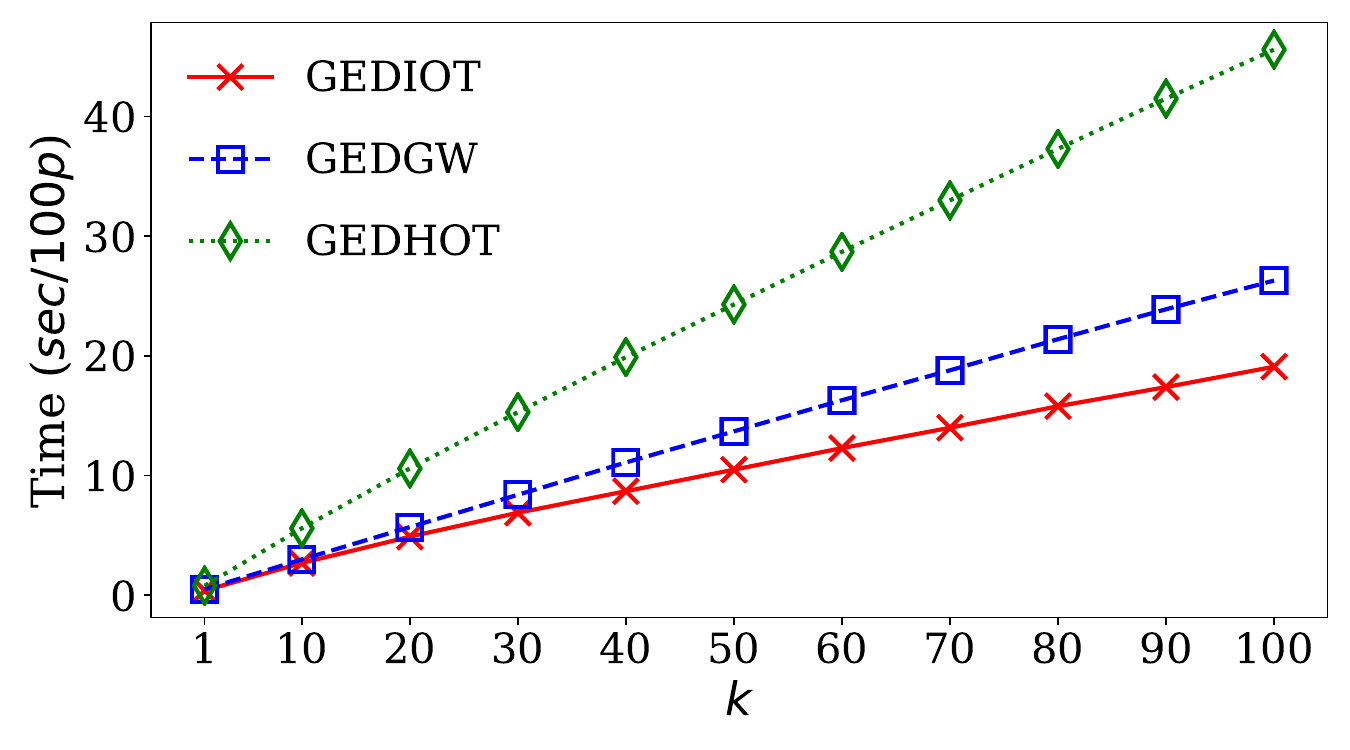}}
    \subfigure[Linux - MAE]{
    \includegraphics[width=0.32\linewidth]{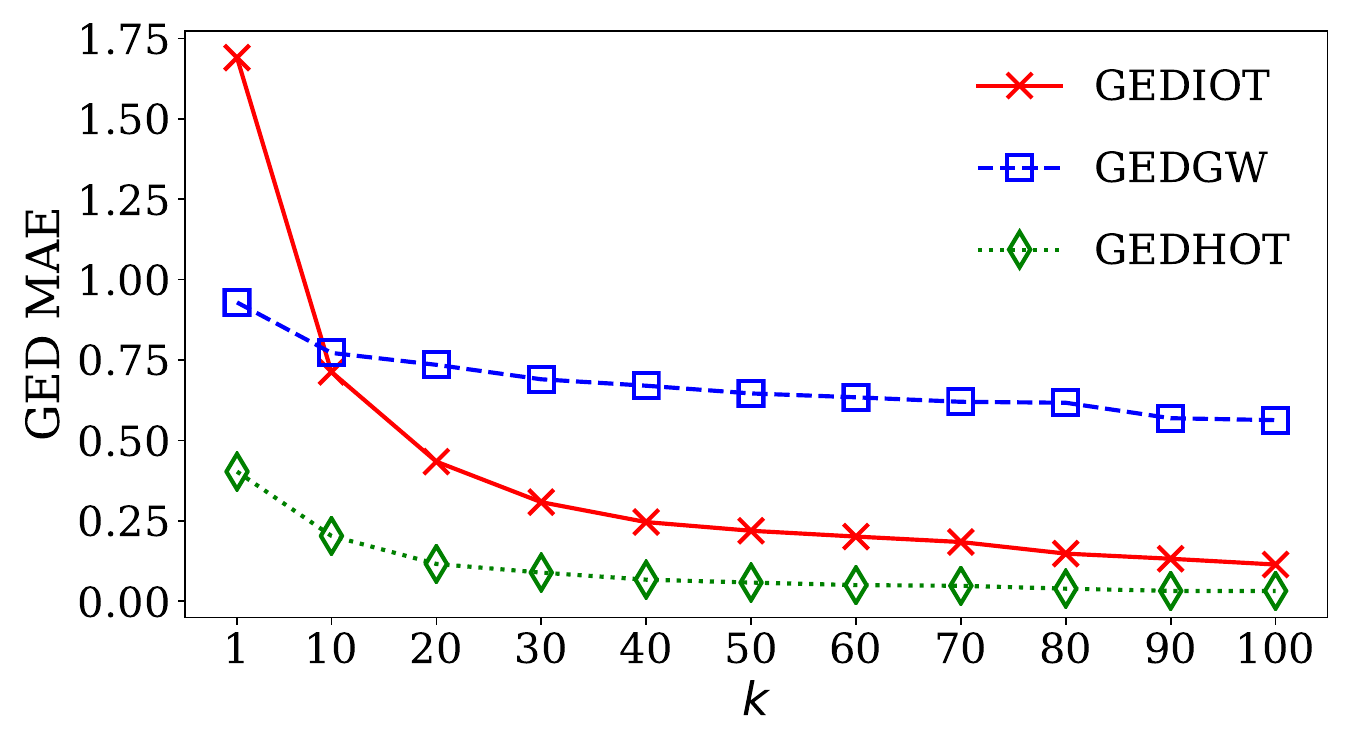}}
    \subfigure[Linux - Accuracy]{
    \includegraphics[width=0.32\linewidth]{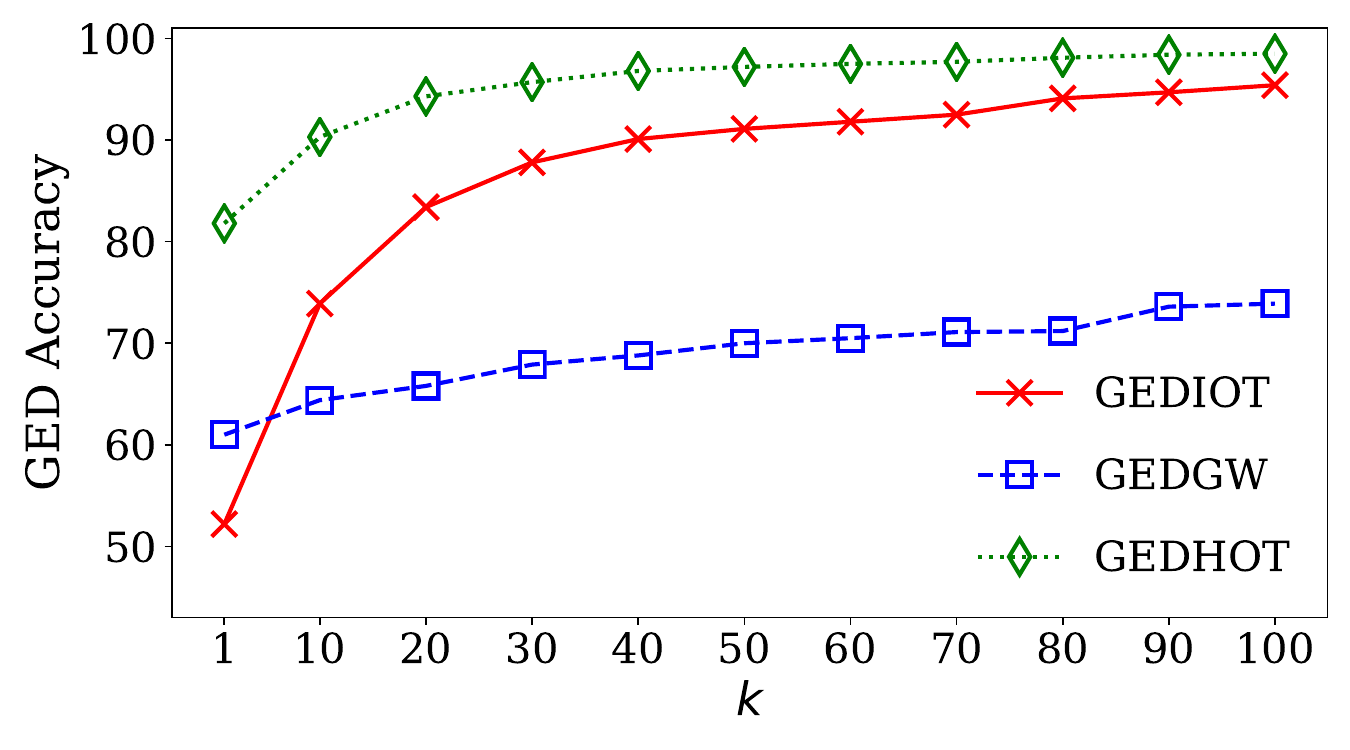}}
    
    \caption{Varying $k$ in $k$-Best Matching for GEP Generation}
    \label{fig:k-best-matching}
\end{figure*}
\setlength{\textfloatsep}{5pt}

\subsection{Ablation Study}\label{app:ablation_exp}
\vspace{1mm}
\noindent{\textbf{Varying Parameters in the Sinkhorn Algorithm.} We also study how the performance of GEDIOT is impacted as the initial regularization coefficient, denoted by $\varepsilon_0$, and the number of iterations vary in the learnable Sinkhorn layer. The results are presented in Figure~\ref{fig:eps-vary}} and Figure~\ref{fig:iternum}. 
{In Figure~\ref{fig:eps-vary}, we set $\varepsilon_0$ to $0.005$, $0.01$, $0.05$, $0.1$, $0.5$, and $1$ on AIDS and Linux. Both MAE and accuracy are stable with various $\varepsilon_0$, which shows the robustness of the learnable regularization method to $\varepsilon_0$. }
In Figure~\ref{fig:iternum}, we set the number of iterations to $1, 5, 10, 15$, and $20$ on AIDS and Linux. We can see that the MAE decreases and the accuracy increases as the number of iterations increases, but after $15$ (resp.\ $10$) iterations on AIDS (resp.\ Linux), the MAE and accuracy becomes fairly stable as the Sinkhorn algorithm converges.  
Note that the computational time also increases when conducting more iterations. Considering the time-accuracy tradeoff, we set the number of iterations to $5$ by default. 

\vspace{1mm}
\noindent{\textbf{Varying $\lambda$ in the Loss Function.} As presented in Figure~\ref{fig:lambda-vary}, we also discuss the effect of varying $\lambda$ (from $0$ to $1$) that balances the two terms $\mathcal{L}_m$ and $\mathcal{L}_v$ of the loss function in Eq.~\eqref{eq:loss}. The results show that the performance improves with the increase of $\lambda$ in $[0, 1]$ and becomes stable when $\lambda$ is around $0.8$. We set $\lambda = 0.8$ by default. }

\vspace{1mm}
\noindent\textbf{Varying the Size of Training Set.} 
In this experiment, we evaluate the effect of varying the training set size. Concretely, we randomly sample 10\%-100\% of the original training set of AIDS and Linux to retrain GEDIOT. Figure~\ref{fig:trainsize} describes its influence on training time, MAE, and accuracy of GEDIOT. It can be observed that as the training set size increases, the MAE decreases and the accuracy increases, while the training time increases linearly. Furthermore, the observed trends of MAE and accuracy with increasing training set size appear to be flattening, which shows that training set size is sufficient.

\vspace{1mm}
\noindent\textbf{$k$-Best Matching.} We further verify the effect of $k$ in $k$-best matching for GEP generation. As depicted in Figure~\ref{fig:k-best-matching}, the MAE constantly decreases and the accuracy increases as the parameter $k$ increases. Nevertheless, computational time also increases with the increase of $k$ since the search space becomes larger.

\section{More Discussion on Our Methods}

\subsection{GED Computation on Edge-labeled Graphs}
\label{app:edge-labeled}

{We here discuss how to handle the GED computation of edge-labeled graphs with GEDHOT. 
For GEDIOT, GINE~\cite{hu2019strategies} is a modified version of GIN that encodes the edge features, so we can replace GIN with GINE. For GEDGW, we can modify the 4-th order tensor $\mathcal{L}(\mathbf{A}^1,\mathbf{A}^2)$ (recall its definition from Table~\ref{Table:notations} and above Eq.~\eqref{eq:gw}), which is regarding the cost of edge edit operations. Let $\ell(u_i, u_j)$ be the label of edge $(u_i, u_j)$ and $\ell(u_i, u_j) = null$ if edge $(u_i, u_j)$ does not exist. Given $u_i, u_j$ in $G^1$ and $v_k, v_l$ in $G^2$, we set $\mathcal{L}
(\mathbf{A}^1_{i,j},\mathbf{A}^2_{k,l})_{i,j,k,l}=1$ if $\ell(u_i,u_j) \neq \ell(v_k,v_l)$, and $0$ otherwise. This modified formulation can handle edge-labeled graphs.}

\subsection{Sizes of Parameters in GEDIOT}
\label{app:parameters-sizes}

{
Like any machine learning model, we learn the model parameters during training, and these parameters are then directly used during test to provide predictions. Notably, our parameters are independent of the graph sizes $n_1$ and $n_2$, so we do not need to do any hard-coding of $n_1$ and $n_2$. Specifically,
\begin{itemize}[leftmargin=*]
    \item The first network component is graph neural network (GIN in particular), where parameters are the MLP weight matrices that only depend on the input dimension $d$ of node embeddings (see Eq.~\eqref{eq:app1} and Eq.~\eqref{eq:app2} in Section~\ref{ssec:NEC}). 
    \item The second network component is the cost matrix layer, which only has a parameter $\mathbf{W}\in\mathbb{R}^{d\times d}$ (see Eq.~\eqref{eq:app3} in Section~\ref{sec:learnot}).
    \item The third network component is the learnable Sinkhorn layer, where the only learnable parameter is the regularization parameter $\varepsilon$, which is a scalar.
    \item The last network component is the graph discrepancy component described in Section~\ref{sec:ntn}, where there is a weight matrix $\mathbf{W}_1\in\mathbb{R}^{d\times d}$ for graph pooling in Eq.~\eqref{eq:app4}, and parameters $\mathbf{W}_2^{[1:L]}\in\mathbb{R}^{L\times d\times d}$, $\mathbf{W}_3\in\mathbb{R}^{L\times 2d}$ and $\mathbf{b}\in\mathbb{R}^L$ for NTN in Eq.~\eqref{eq:app5}. Here $L$ is also a hyperparameter that is independent of graph size. 
\end{itemize}
}

\end{appendix}
\end{document}